\documentclass[nohyperref]{article}

\usepackage{microtype}
\usepackage{graphicx}
\usepackage{subcaption}
\usepackage{booktabs} %

\usepackage{hyperref}

\usepackage[accepted]{icml2022}

\usepackage{amsmath}
\usepackage{amssymb}
\usepackage{mathtools}
\usepackage{amsthm}
\usepackage{mathtools}  %
\usepackage{relsize}    %
\usepackage{microtype}
\usepackage{booktabs}
\usepackage[light,scaled=0.85]{roboto-mono}
\usepackage{enumitem} %
\setlist{nosep}       %

\usepackage{caption}  %
\usepackage[capitalize,noabbrev]{cleveref}
\Crefname{section}{\mbox{\S\hspace*{-0.25ex}}}{\mbox{\S\hspace*{-0.25ex}}}
\Crefname{equation}{Eq.}{Eqs.}
\Crefname{figure}{Fig.}{Figs.}
\Crefname{table}{Tab.}{Tabs.}
\Crefname{appendix}{\S$\!$}{\S$\!$}

\usepackage{alexdefs}

\usepackage[subtle,mathspacing=normal]{savetrees}

\theoremstyle{plain}
\newtheorem{theorem}{Theorem} %

\theoremstyle{definition}

\theoremstyle{remark}

\usepackage[textsize=tiny]{todonotes}

\newcommand{\tactis}{\textsc{tact}{\relsize{-0.5}\textls[120]{i}}\textsc{s}}
\newcommand{\tactistt}{\textsc{tact}{\relsize{-0.5}\textls[120]{i}}\textsc{s-tt}}
\newcommand{\tactisgc}{\textsc{tact}{\relsize{-0.5}\textls[120]{i}}\textsc{s-gc}}
\newcommand{\tactisic}{\textsc{tact}{\relsize{-0.5}\textls[120]{i}}\textsc{s-ic}}
\newcommand{\besthp}[1]{\textsuperscript{\textbf{#1}}}
\newcommand{\paragraphtight}[1]{\par\textbf{#1}~~~}
\newcommand{\B}[1]{\mathbf{#1}}     %

\newtheorem*{rep@theorem}{\rep@title}
\newcommand{\newreptheorem}[2]{%
\newenvironment{rep#1}[1]{%
 \def\rep@title{#2 \ref{##1}}%
 \begin{rep@theorem}}%
 {\end{rep@theorem}}}
\makeatother
\newreptheorem{proposition}{Proposition}
\newreptheorem{theorem}{Theorem}

\usepackage{setspace}

\begin{document}

\icmltitlerunning{Transformer-Attentional Copulas for Time Series}

\twocolumn[
\icmltitle{\tactis{}: Transformer-Attentional Copulas for Time Series}

\icmlsetsymbol{equal}{*}

\begin{icmlauthorlist}
\icmlauthor{Alexandre Drouin}{equal,snow}
\icmlauthor{Étienne Marcotte}{equal,snow}
\icmlauthor{Nicolas Chapados}{equal,snow}
\end{icmlauthorlist}

\icmlaffiliation{snow}{ServiceNow Research}

\icmlcorrespondingauthor{All authors}{firstname.lastname@servicenow.com}
\icmlkeywords{Time series, Deep Learning, Transformer, Machine Learning, ICML}

\vskip 0.3in
]

\printAffiliationsAndNotice{\icmlEqualContribution}

\begin{abstract} %
The estimation of time-varying quantities is a fundamental component of decision making in fields such as healthcare and finance.
However, the practical utility of such estimates is limited by how accurately they quantify predictive uncertainty.
In this work, we address the problem of estimating the joint predictive distribution of high-dimensional multivariate time series.
We propose a versatile method, based on the transformer architecture, that estimates joint distributions using an attention-based decoder that provably learns to mimic the properties of non-parametric copulas.
The resulting model has several desirable properties: it can scale to hundreds of time series, supports both forecasting and interpolation, can handle unaligned and non-uniformly sampled data, and can seamlessly adapt to missing data during training.
We demonstrate these properties empirically and show that our model produces state-of-the-art predictions on multiple real-world datasets.
\end{abstract}

\section{Introduction}\label{sec:introduction}

In numerous time series forecasting contexts, data presents itself in a raw form
that rarely matches the standard assumptions of classical forecasting methods.
For instance, in healthcare settings and economic forecasting, groups of related
time series can have different sampling frequencies, be sampled irregularly, and
exhibit missing values~\citep{shukla2021survey,sun2020review}. Covariates that are
predictive of future behavior may not be available for all historical data, or
may be available in a different form, e.g., due to changes in measurement methodology.
Moreover, optimal decision-making in downstream tasks
generally requires the full joint predictive distribution over arbitrary future
time horizons~\citep{peterson2017decision}, not just marginal
quantiles thereof at a fixed horizon. We seek to develop general forecasting
methods that are both suitable for a wide range of downstream tasks and that can
handle all stylized facts about real-world time series---as they are, not as
they ought to be---namely:

\begin{itemize}
    \item Joint characterization of a multivariate stochastic process,
    forecasting trajectories at arbitrary time horizons;
    \item Presence of non-stochastic covariates, either static or time-varying, used as conditioning variables;
    \item Variables within the process that are measured at different sampling frequencies or irregularly sampled;
    \item Missing values for arbitrary time points and variables;
    \item Variables with different domains---%
    $\reals,\reals^+,\nats,\nats^+,\ints$---with skewed and fat-tailed marginal
    behavior.
\end{itemize}

Classical times series models, such as \textsc{arima} \citep{box2015time} and
exponential smoothing methods \citep{hyndman2008forecasting}, are very restricted
in their handling of the above stylized facts. Although extensions have been
proposed that deal with individual issues, they cannot easily deal with all and
require considerable domain knowledge to be effective. Machine learning models
have recently gained in popularity~\citep{Benidis:2020tn}. Nevertheless, methods
introduced in recent years all suffer from limitations when dealing with one or
more of the above listed stylized facts, or neglect in their handling of the
full predictive distribution.

Separately, multivariate forecasting models based on copulas have been popular
in econometrics for more than a decade
\citep{Patton2012ARO,remillard2012copulasemi,Krupskii2020flexible,mayer2021estimation}.
These models enable the separate characterization of the joint behavior of a
group of random variables from their marginal behavior. This has been found to
be especially valuable in areas such as finance and insurance where marginals
are known to exhibit particular patterns of skewness and kurtosis. Recently in
machine learning, low-rank Gaussian copula processes with LSTMs
\citep{hochreiter1997long} have been proposed for high-dimensional forecasting
\citep{salinas2019high}.

Building on the recent successes of transformers as general-purpose sequence
models \citep{vaswani2017attention} and their success in time series forecasting
\citep{rasul2021multivariate,tashiro2021csdi,tang2021probabilistic}, we propose a transformer architecture that can tackle all the above
stylized facts about real-world time series. Notably, we show how to represent
an implicit copula when sampling from the transformer decoder model. Moreover,
by representing each observation as a distinct token with its own timestamp, we can naturally
handle irregularly sampled times series as well as series with missing values
and unequal sampling frequencies. 
We also show that an efficient two-dimensional attention scheme lets us scale to hundreds of time steps and series using vanilla attention \citep{bahdanau2015attention}.

\paragraphtight{Contributions:}
\begin{enumerate}
    \item We present \textit{Transformer-Attentional Copulas for Time Series} (\tactis{}), a highly flexible transformer-based model for large-scale multivariate probabilistic time series prediction (\cref{sec:tactis}).
    \item We introduce \emph{attentional copulas}, an attention-based architecture that estimates non-parametric copulas for an arbitrary number of random variables (\cref{sec:tactis-pseudocopula}).
    \item We theoretically prove the convergence of attentional copulas to valid copulas (\cref{sec:tactis-training}).
    \item We conduct an empirical study showing \tactis{}' state-of-the-art probabilistic prediction accuracy on several real-world datasets, along with notable flexibility (\cref{sec:results}).
\end{enumerate}

\section{Background}

\subsection{Problem Setting}\label{sec:problem-setting}

We are interested in the general problem of estimating the joint distribution of values missing at arbitrary time points in multivariate time series.
This general task encompasses classical problems, such as probabilistic forecasting, backcasting, and interpolation.

Formally, we consider a set of $m$ multivariate time series $\Scal \eqdef \left\{\Xb_1, \ldots, \Xb_m\right\}$, where each $\Xb \in \Scal$ is a collection of possibly-related univariate time series.
For simplicity, we elaborate the notation for a single element of $\Scal$.
Let $\Xb \eqdef \left\{\xb_i \in \reals^{l_i}\right\}_{i=1}^n$, where the $\xb_i$ are univariate time series with arbitrary lengths $l_i \in \nats^+$.
Each $\xb_i$ is associated with (\emph{i})~a Boolean mask $\mb_i \in \bools^{l_i}$, such that $m_{ij} = 1$ if $x_{ij}$ is observed and $m_{ij} = 0$ otherwise,
(\emph{ii})~a matrix of time-varying covariates $\Cb_i \eqdef \left[ \cb_{i1}, \dots, \cb_{il_i} \right] \in \reals^{p \times l_i}$, where each $\cb_{ij} \in \reals^{p}$ represents arbitrary additional information about the observations, and (\emph{iii})~a vector of time stamps $\tb_i \in \reals^{l_i}$ s.t. $t_{ij} < t_{i,j+1}$, indicating the times at which the data were measured.
Note that this setting naturally supports unaligned time series with arbitrary sampling frequencies.

Our goal is to infer the joint distribution of missing time series values given all known information:\footnote{We slightly abuse the notation and omit random variables.}
\begin{equation}\label{eq:problem-statement}
    P\Big(
        \Big\{ \xb_i^{(m)} \Big\}_{i=1}^n
        \mathrel{\Big|}
        \Big\{\xb_i^{(o)}, \Cb_i, \tb_i \Big\}_{i=1}^n
    \Big),
\end{equation}
where $\xb_i^{(o)}$ and $\xb_i^{(m)}$ are the observed ($m_{ij} = 1$) and missing ($m_{ij} = 0$) elements of $\xb_i$, respectively.

From this problem formulation, one can recover standard time series problems via specific masking patterns.
For instance, a $t$-step probabilistic forecasting task can be defined by setting the last $t$ elements of each $\mb_i$ to zero.

\subsection{Transformers}

The model that we propose builds on the \emph{transformer architecture} for
sequence-to-sequence transduction~\citep{vaswani2017attention}. Transformer
models have an encoder-decoder structure, where the encoder learns a
representation of the tokens of an input sequence, and the decoder generates the
tokens of an output sequence autoregressively, based on the input sequence. The
main feature of such models is that they can capture non-sequential dependencies
between tokens via attention
mechanisms~\citep{bahdanau2015attention}. This is in sharp contrast with
recurrent neural networks~\citep{Goodfellow-et-al-2016}, such as
LSTMs~\citep{hochreiter1997long}, which are inherently sequential. Transformers
have been widely discussed in the literature and we thus refer the reader to
the seminal work of \citet{vaswani2017attention} for additional details.
As we later describe, transformers allow \tactis{} to view time series as sets
of tokens, among which non-local dependencies can be learned regardless of
considerations like alignment and sampling frequency.

\subsection{Copulas}\label{sec:bg-copula}

Copulas are mathematical constructs that allow separating the joint dependency structure of a set of random variables from their marginal distributions~\citep{nelsen2007introduction}.
Such a separation can be used to learn reusable models that can be applied to seemingly different distributions, where variables have different marginals but share the dependency structure.

Formally, a copula $C: [0, 1]^d \rightarrow [0, 1]$ is the joint cumulative distribution function (CDF) of a $d$-dimensional random vector $[U_1, \dots, U_d]$ on the unit cube with uniform marginal distributions, i.e.,
\begin{equation}
    C(u_1, \dots, u_d) \eqdef P(U_1 \leq u_1, \dots, U_d \leq u_d),
\end{equation}
where $U_i \sim U_{[0, 1]}$.
According to \citet{sklar1959fonctions}'s theorem, the joint CDF of any random vector $[X_1, \dots, X_d]$ can be expressed as a combination of a copula $C$ and the marginal CDF of each random variable $F_i(x_i) \eqdef P(X_i \leq x_i)$,
\begin{equation}
    P(X_1 \leq x_1, \ldots, X_d \leq x_d) = C\big(F_1(x_1), \dots, F_d(x_d)\big),
    \label{eq:copula-sklar-cdf}
\end{equation}
and the corresponding probability density\footnote{Assuming that the distribution is continuous.} is given by:
\begin{equation}
\begin{multlined}
p(X_1 \!= x_1, \ldots, X_d \!= x_d) = \hspace*{3cm}\\
    c\big(F_1(x_1), \dots, F_d(x_d)\big) \times f_1(x_1) \times \cdots \times f_d(x_d),
\end{multlined}
\end{equation}
where $c: [0, 1]^d \rightarrow [0, 1]$ is the copula's density function and  $f_i$ is the marginal density function of $X_i$.
It is thus possible to estimate seemingly complex joint distributions by learning the parameters of a simple internal joint distribution (the copula) and marginal distributions, which can even be estimated empirically.
One classical approach, which has been abundantly used in time series applications, is the Gaussian copula~\citep{nelsen2007introduction}, where $C$ is constructed from a Gaussian distribution.

In this work, we avoid making such parametric assumptions and propose a new attention-based architecture trained to mimic a non-parametric copula, which we term \emph{attentional copula}. \tactis{} learns to produce the parameters of the copula, on-the-fly, based on learned variable representations. It can thus reuse a learned dependency structure across multiple sets of variables, by mapping them to similar representations.

\section{Related Work}

\paragraphtight{Neural Networks for Time Series Forecasting} Although studied in the
1990's \citep{Zhang1998forecasting}, neural networks and other machine learning (ML)
techniques had long been taken with caution in the forecasting community due to
a perceived propensity to overfit \citep{Makridakis:2018kw}. In recent years,
however, the field has seen a number of demonstrations of successful ML-based
forecasting, in particular winning the prestigious M5 competition
\citep{Makridakis2021m5uncertainty,Makridakis2022m5accuracy}. For neural
networks, work has mostly centered around so-called global models
\citep{MonteroManso2021groups}, which in contrast to classical statistical
methods such as \textsc{arima} \citep{box2015time} or exponential smoothing
\citep{hyndman2008forecasting}, learn a single set of parameters to forecast many
series. A first wave of approaches in the recent resurgence was primarily based
on recurrent or convolutional neural network encoders
\citep{shih2019temporal,Chen:2020bj}. \citet{oreshkin2020nbeats} introduce a
recursive decomposition based on a residual signal projection on a set of
learned basis functions. \citet{Guen2020structured} introduce an approach for
univariate probabilistic forecasting based on determinantal point processes to
capture structured shape and temporal diversity. Extensions of classical
state-space models have also been proposed
\citep{Yanchenko2020stanza,deBezenac2020NKF}. Comprehensive surveys of deep
learning methods for forecasting appear in \citet{Lim:2020br} and
\citet{Benidis:2020tn}, restricting coverage to regularly-sampled data. Of
relevance to the present work are studies of probabilistic multivariate methods,
transformer-based approaches, copulas, and techniques enabling irregular
sampling. 

\paragraphtight{Probabilistic Multivariate Methods}
Narrowing attention to methods carrying out probabilistic forecasting of the
joint distribution of a multivariate stochastic process, DeepAR
\citep{salinas2020deepar} computes an iterated one-step-ahead Monte Carlo
approximation of the predictive distribution by sampling from a fixed functional
form whose parameters are the result of a recurrent neural network (RNN).
Likewise, \citet{rasul2021autoregressive} propose instead to model the
predictive joint one-step-ahead distribution using a denoising diffusion process
\citep{ho2020denoising,sohl-dickstein2015deep}; however, with the diffusion
dynamics conditioned on an RNN, the model cannot easily deal with missing values
or irregularly sampled time series. \citet{rasul2021multivariate} propose to
model the predictive joint one-step-ahead distribution by a multivariate
normalizing flow \citep{Papamakarios2021normalizing}, parametrized by either a
RNN or a transformer. For general density estimation, \citet{uria2014deep}
propose an order-agnostic autoregressive density estimator based on neural
networks, with \citet{hoogeboom2022autoregressive} extending the approach to
diffusion models; both are related to the autoregressive copula decomposition
that we propose herein.

\textbf{Transformer-based approaches} have come to the fore more recently, on
the strength of their successes in other sequence modeling tasks.
\citet{lim2021temporal} introduce the temporal fusion transformer, combining
recurrent layers for local processing and self-attention layers for
characterizing long-term dependencies, evaluating performance on quantile loss
measures; notably, the architecture makes use of a gating mechanism to suppress
unnecessary covariates. \citet{li2019enhancing} introduce a transformer with
subquadratic memory complexity along with convolutional self-attention to better
handle local context. \citet{Spadon2020pay} propose a recurrent graph evolution
neural networks, which embeds transformers, to carry out point multivariate
forecasting. \citet{wu2020adversarialSparse} propose an adversarial sparse
transformer to estimate conditional quantiles of the predictive distribution.
\citet{tashiro2021csdi} use a conditional score-based diffusion model explicitly
trained for interpolation and report substantial improvements over existing
probabilistic imputation models, as well as competitive performance on some
forecasting tasks against recently-proposed deep learning models.
\citet{tang2021probabilistic} introduce a variational non-Markovian state space
model where the latent dynamics are given by an attention mechanism over all
previous latents and observed tokens. They report good multivariate
probabilistic forecasting results on five standard datasets, as well as on the
task of human motion prediction. \citet{wu2021autoformer} introduce an
auto-correlation mechanism in place of self attention and report good accuracy
on long-horizon point forecasting benchmarks. Recently, outside of time series
forecasting, \citet{muller2022transformers} show how to train a
transformer-based model to approximate the Bayesian posterior predictive
distribution in regression tasks; these results are in the spirit of the
forecasting and interpolation results that we present in this paper.

\paragraphtight{Copulas-Based Forecasting}
There exists an abundant literature on uses of copulas~\citep{Grosser2021copulae}
for economic and financial forecasting, although most published models have
focused on fixed functional forms~\citep{Patton2012ARO,remillard2012copulasemi}.
\citet{Aas2009PairCopula}~introduce vine copulas that provide increased pairwise
flexibility, and \citet{Krupskii2020flexible,mayer2021estimation} suggest more
flexible forms for forecasting applications. In ML,
\citet{LopezPaz2012SemiSupervisedDA} propose to carry out domain adaptation
using a form of non-parametric copulas based on kernel estimators. For
forecasting, \citet{salinas2019high} introduce GPVar, an LSTM-based method that
dynamically parametrizes a Gaussian copula.
To the best of our knowledge, none of the published methods show how to sample 
from a \emph{non-parametric copula} resulting from an autoregressive decomposition 
of a time-varying conditional copula distribution, as we introduce in this work.

\paragraphtight{Irregular Sampling}
Gaussian processes have been widely used to handle irregularly-sampled
series~\citep{williams2006gaussian,chapados2007functional}, but these approaches
are limited in other ways, notably in computational tractability. More recently,
\citet{shukla2021multitime} have proposed a transformer-like attention mechanism
to re-represent an irregularly sampled time series at a fixed set of reference
points, and evaluate on interpolation and classification tasks.

\section{The \tactis{} Model}\label{sec:tactis}

Our contribution is a flexible model for multivariate probabilistic time
series prediction composed of a transformer encoder~\citep{vaswani2017attention}
and an attention-based decoder trained to mimic a non-parametric copula; hence
the name \textit{Transformer-Attentional Copulas for
Time Series} (\tactis).
\cref{fig:model_overview} shows an overview of the model architecture.%
\footnote{Code available at \href{{https://github.com/servicenow/tactis}}{https://github.com/servicenow/tactis}.}

Akin to classical Transformers~\citep{vaswani2017attention}, \tactis{} views the elements of a multivariate time series ($x_{ij}$) as an arbitrary set of tokens, where some tokens are observed and some are missing (based on $m_{ij}$).
The encoder is tasked with learning a meaningful representation of each token, such as to enable the decoder to infer a multivariate joint distribution over the values of the missing tokens.
Due to the use of attention in both the encoder and the decoder, \tactis{} can adapt to an arbitrary number of tokens, without retraining.

\tactis{} learns the conditional predictive distribution of arbitrary missing values in multivariate time series. At inference time, it is the pattern of missing values themselves---namely, where they are located with respect to measured values---that determines whether it performs forecasting, interpolation, or another similar task. 
Consequently, \tactis{} naturally supports changes in its inputs, such as the unavailability of some time series, and in its outputs, such as changes in forecast horizon.
Moreover, it inherently supports misalignment and differences in sampling frequencies, since each token is encoded separately.
This is in sharp contrast with classical vector autoregressive models~\citep{wei2018multivariate} and most recurrent neural networks, which jointly consider the values of each time series at a given time step ($j$) as a vector $(x_{1j}, \ldots, x_{nj}) \in \reals^n$.
We now detail the encoder, decoder, training procedure, and provide
some theoretical guarantees.

\begin{figure}
    \centering
    \includegraphics[width=\linewidth]{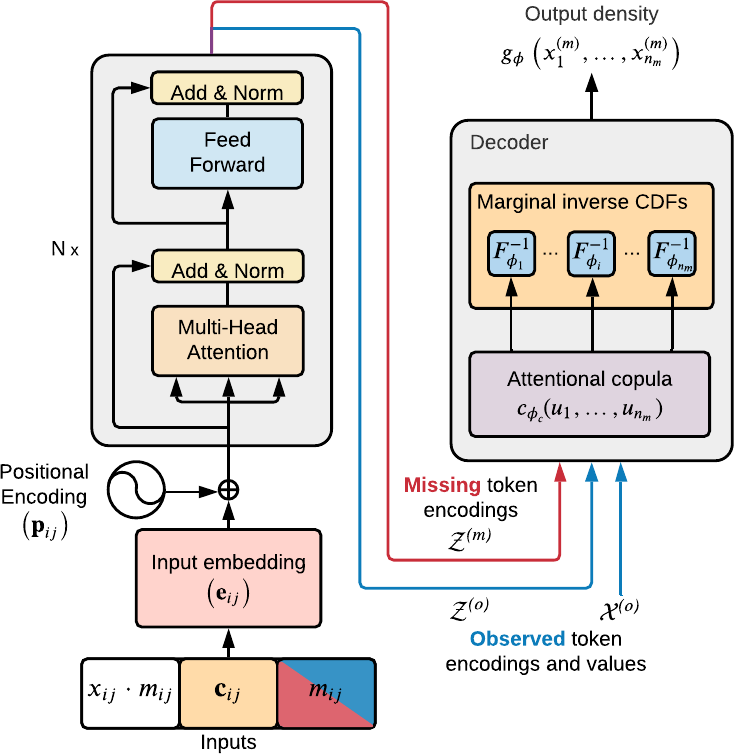}
    \vspace*{-5mm}
    \caption{Model overview. \textbf{(Left)} The \tactis{} encoder is very similar to that of standard transformers. The key difference is that both observed and missing tokens are encoded simultaneously. \textbf{(Right)} The decoder, based on an attentional copula, learns the output density given representations of the observed and missing tokens.}
    \label{fig:model_overview}
    \vspace*{-1ex}
\end{figure}

\subsection{Encoder}
\label{sec:tactis-encoder}

As shown in \cref{fig:model_overview}, the encoder used in \tactis{} is identical to that of standard Transformers~\citep{vaswani2017attention}, except that all tokens (observed and missing) are jointly encoded.

\paragraphtight{Input embedding} The encoder starts by producing a vector embedding ($\eb_{ij} \in \reals^{d_\text{emb}}$) for each element in each time series (i.e., tokens), which accounts for its value $x_{ij}$, the associated covariates\footnote{Optionally, we include a learned, per series, embedding in $\cb_{ij}$.} $\cb_{ij}$, and whether it is observed or missing $m_{ij}$.
Such embeddings are given by a neural network with parameters
$\theta_\text{emb}$, with $x_{ij} \cdot m_{ij}$ masking the values of missing tokens:
\[
    \eb_{ij} = \text{Embed}_{\theta_\text{emb}}\!\!\left(x_{ij} \cdot m_{ij}, \cb_{ij}, m_{ij}\right).
\]

\paragraphtight{Positional encoding} We add information about a token's time stamp $t_{ij}$ to the input embedding via a positional encoding $\pb_{ij} \in \reals^{d_\text{emb}}$.
For simplicity, we use the positional encodings of \citet{vaswani2017attention}, based on sine and cosine functions of various frequencies, and obtain the final embeddings as $\eb'_{ij} = \eb_{ij} \sqrt{d_\text{emb}} + \pb_{ij}$.
Note that other choices, such as positional encodings tailored to time series (e.g., accounting for holidays, day of week, etc.), would be viable, but we keep such explorations for future work.

Then, following \citet{vaswani2017attention}, the $\eb'_{ij}$ embeddings are passed through a stack of residual layers that combine multi-head self-attention and layer normalization to obtain an encoding ($\zb_{ij}$) for each token.
Such encodings contain a complex mixture of information from other tokens that were deemed relevant by the attention mechanism, such as the value of their covariates, their time stamp, and the values of observed tokens.

\paragraphtight{Scalability} One well-known limitation of transformers is the large time and memory complexity of self-attention, scaling quadratically with the number of tokens.
However, improving the efficiency of transformers is an active field of research \citep{tay2020efficient,lin2021survey} and any progress in this direction is poised to be directly applicable to \tactis{}.
In this work, when applying \tactis{} to large datasets, such as those in \cref{sec:results}, we take advantage of the fact that each token is indexed by two independent indices: $i$ for the variables, and $j$ for the time steps.
We do so by employing the temporal transformer layers of \citet{tashiro2021csdi}, which first compute self-attention between the tokens of each variable ($x_{i1}, \ldots, x_{il_i}$) and then between the tokens at a given time step ($x_{1j}, \ldots, x_{nj}$). Hence, instead of scaling in $O([n \cdot l_\text{max}]^2)$, the time and memory complexity scale in $O(n^2 \cdot l_\text{max} + n \cdot l_\text{max}^2)$, where $n$ is the number of time series and $l_\text{max} \eqdef \max_i l_i$ is the length of the longest time series.
One downside of this approach is that, for attention within a given time step ($j$) to make sense, the time series must be aligned.
When using temporal transformer layers, we will refer to our model as \tactistt{}.

\subsection{Decoder}
\label{sec:tactis-pseudocopula}

As stated in \cref{eq:problem-statement}, we aim to learn the joint distribution of the values ($x_{ij}$) of missing tokens ($m_{ij} = 0$), given the values of observed tokens ($m_{ij} = 1$), the covariates ($\cb_{ij}$), and the time stamps ($t_{ij}$). We achieve this via an attention-based decoder trained to mimic a non-parametric copula (see \cref{sec:bg-copula}), which we now describe.

Since the data consists of an arbitrary set of observed and missing tokens, we introduce the following notation: 
$\Zcal^{(o)} \eqdef \big\{\zb^{(o)}_1, \ldots, \zb^{(o)}_{n_o}\big\}$, $\Zcal^{(m)} \eqdef \big\{\zb^{(m)}_1, \ldots, \zb^{(m)}_{n_m}\big\}$, $\Xcal^{(o)} \eqdef \big\{x^{(o)}_1, \ldots, x^{(o)}_{n_o}\big\}$, 
and $\Xcal^{(m)} \eqdef \big\{x^{(m)}_1, \ldots, x^{(m)}_{n_m}\big\}$, which respectively denote
the encoded representations and values of observed and missing tokens.
We also use $\Ccal \eqdef \left\{\Cb_{1}, \ldots, \Cb_{n} \right\}$ and $\Tcal \eqdef \left\{\tb_{1}, \ldots, \tb_n \right\}$ to denote the set of all covariates and timestamps, respectively (see \cref{sec:problem-setting}).

Our goal is to accurately estimate the joint density of missing token values, using a model $g_{\phi}(x^{(m)}_1, \ldots, x^{(m)}_{n_m})$ such that:
\begin{equation*}
g_{\phi}(x^{(m)}_1, \ldots, x^{(m)}_{n_m}) \approx p(x^{(m)}_1, \ldots, x^{(m)}_{n_m} \mid \Xcal^{(o)}, \Ccal, \Tcal),
\end{equation*}
where the distributional parameters\footnote{We use subscripted versions of $\phi$ and $\theta$ to denote distributional parameters and neural network parameters, respectively.} $\phi$ are produced by a neural network parametrized by $\theta_\text{dec}$:
\begin{equation*}
\phi = \text{Decoder}_{\theta_\text{dec}}\!\big(\Xcal^{(o)}, \Zcal^{(o)}, \Zcal^{(m)}\big).
\end{equation*}
Crucially, note that using encoded token representations ($\Zcal^{(o)}, \Zcal^{(m)}$; see~\cref{sec:tactis-encoder}) as inputs into the decoder ensures that the resulting density is conditioned on observations. We consider the following copula-based structure for $g_\phi$:
\begin{equation}
\begin{multlined}
\label{eq:model_expansion}
    g_{\phi}\big(x^{(m)}_1, \ldots, x^{(m)}_{n_m}\big) \;\eqdef \hspace*{3.5cm} \\
    c_{\phi_\text{c}}\Big( 
        F_{\phi_1}\!\big(x^{(m)}_1\big), \ldots, F_{\phi_{n_m}}\!\big(x^{(m)}_{n_m}\big)
    \!\Big)  \;\times \\
    \qquad\qquad
    f_{\phi_1}\!\big(x^{(m)}_1\big) \times \cdots \times f_{\phi_{n_m}}\!\big(x^{(m)}_{n_m}\big),
\end{multlined}
\end{equation}
where $\phi \eqdef \{\phi_\text{c}, \phi_1, \ldots, \phi_{n_m}\}$, $c_{\phi_\text{c}}$ is the density of a copula, and $F_{\phi_k}$ and $f_{\phi_k}$ respectively represent the marginal (univariate) cumulative distribution and density functions of $X^{(m)}_k$.

The structure of $g_\phi$ allows for a wealth of different copula parametrizations (e.g., the Gaussian copula of \citet{salinas2019high}).
The crux of \tactis{} resides in how each of the components is parametrized.
Notably, we propose to use normalizing flows~\citep{tabak2013family} to model the marginals, and we develop a flexible non-parametric copula that can automatically adapt to a changing number of missing tokens ($n_m$).
We now outline these constructs, for which an overview appears in \cref{fig:attentional-copula}.

\begin{figure}
    \includegraphics[width=\linewidth]{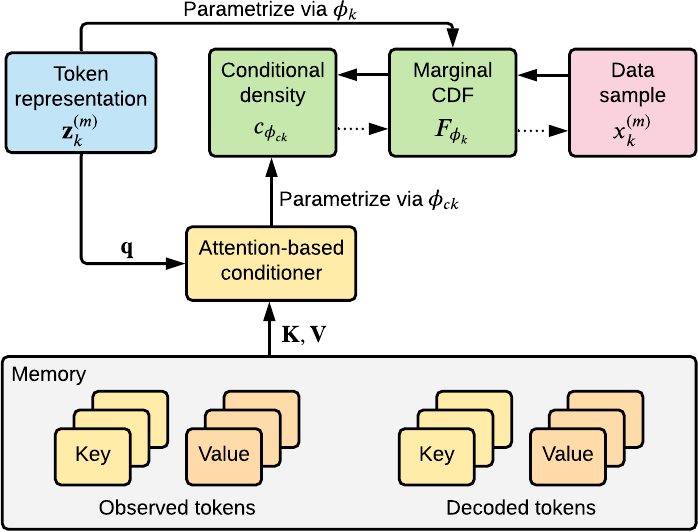}
    \caption{Overview of the \tactis{} decoder architecture. Dotted arrows indicate the flow of information during sampling.}
    \label{fig:attentional-copula}
\end{figure}

\paragraphtight{Marginal distributions}
To model the marginal CDFs ($F_{\phi_k}$), we seek a univariate function that (\emph{i})~maps values to $[0, 1]$, (\emph{ii})~is monotonically increasing, (\emph{iii})~is both continuous and differentiable.
We achieve this using a modified version of the Deep Sigmoidal Flows (DSF) of \citet{huang18-naf}, where we simply remove the logit function of the last flow layer to obtain values in $[0, 1]$.
The parameters of each flow are produced by a neural network with parameters $\theta_\text{F}$, which are shared across all $k$,
\begin{equation}\label{eq:marginal-decoder}
    \phi_{k} = \text{MarginalParams}_{\theta_\text{F}}\!\big(\zb^{(m)}_k\big).
\end{equation}
The marginal densities $f_{\phi_k}$ are obtained by differentiating $F_{\phi_k}$ w.r.t. $x^{(m)}_k$, an efficient operation for DSF.
In addition to satisfying our desiderata, DSFs have been shown to be universal density approximators, thereby not constraining the modeling ability of \tactis{}.

\paragraphtight{Copula density}
According to the definition of a copula (\cref{sec:bg-copula}), we must parametrize a distribution on the unit cube $[0, 1]^{n_m}$ with uniform marginals. We consider an autoregressive factorization of the copula density according to an arbitrary permutation $\pi = [\pi_1, \ldots, \pi_{n_m}]$ of the indices $\{1, \ldots, n_m \}$.
We denote the copula density and its parameters by $c_{\phi_{c}^{\pi}}$:
\begin{equation}
\begin{multlined}
\label{eq:copula-density}
    c_{\phi_{c}^{\pi}}\!\left(u_1, \ldots, u_{n_m} \right) = \hspace*{4cm}\\
    c_{\phi_{c1}^{\pi}}\!\!\left(u_{\pi_1}\right) \times c_{\phi_{c2}^{\pi}}\!\!\left(u_{\pi_2} \mid u_{\pi_1}\right) \times \cdots \times\\ 
    \qquad c_{\phi^{\pi}_{c{n_m}}}\!\!\left(u_{\pi_{n_m}} \mid u_{\pi_1}, \ldots, u_{\pi_{n_m - 1}} \right),
\end{multlined}
\end{equation}
where $c_{\phi_{ck}^{\pi}}$ is $k^\text{th}$ conditional density in the factorization and $u_{\pi_k} = F_{\phi_{\pi_k}}\!\big(x^{(m)}_{\pi_k}\big)$.
Importantly, we let $c_{\phi^{\pi}_{c1}}$ be the \emph{density of a uniform distribution} $U_{[0, 1]}$, and we use an \emph{attention-based conditioner} to obtain the parameters of the remaining conditional distributions $\phi^{\pi}_{ck}$, for $k > 1$.
We call the resulting construct an \emph{attentional copula}.

\paragraphtight{Attention-based conditioner}
This component of the decoder produces the parameters $\phi^{\pi}_{ck}$ for the conditional density $c_{\phi^{\pi}_{c{k}}}\!\!\left(u_{\pi_{k}} \mid u_{\pi_1}, \ldots, u_{\pi_{k-1}} \right)$ by performing attention over a memory composed of the representations of observed tokens and missing tokens that are predecessors in the permutation: $\big\{ \zb^{(o)}_1, \ldots, \zb^{(o)}_{n_o}, \zb^{(m)}_{\pi_1}, \ldots, \zb^{(m)}_{\pi_{k - 1}} \big\}$, as well as their CDF-transformed\footnote{For observed tokens, transformation via the normalizing flow does not necessarily correspond to a CDF transform, but it ensures that all values are on a similar scale.} values $\big\{ u^{(o)}_1, \ldots, u^{(o)}_{n_o}, u^{(m)}_{\pi_1}, \ldots, u^{(m)}_{\pi_{k - 1}} \big\}$. 
The conditioner is composed of several layers, which are as follows.
First, we calculate keys $\kb \in \reals^{d_\text{att}}$ and values $\vb \in \reals^{d_\text{att}}$
for each element $(\zb, u)$ in the memory using two modules
parameterized by $\theta_\text{k}$ and $\theta_\text{v}$, respectively:
\footnote{For simplicity, we use a single attention head in the presentation, but in practice, we use multiple (see \citet{vaswani2017attention}).}
\begin{align}\label{eq:copula-keysvals}
    \kb = \text{Key}_{\theta_\text{k}}\!(\zb, u) & & \vb = \text{Value}_{\theta_\text{v}}\!(\zb, u).
\end{align}
We then calculate a query $\qb \in \reals^{d_\text{att}}$ for our token of interest $\zb^{(m)}_{\pi_k}$ using a module with parameters $\theta_\text{q}$:
$$
    \qb = \text{Query}_{\theta_\text{q}}\!\big(\zb^{(m)}_{\pi_k}\big).
$$
Let $\Kb$ and $\Vb$ be the matrices of all keys and values for tokens in the memory, respectively.
Following \citet{vaswani2017attention}, we obtain an attention-based representation $\zb'' \in \reals^{d_\text{att}}$:
\begin{align*}
\zb' = \ & \text{LayerNorm}\big(\Vb^\intercal\,\text{Softmax}(\Kb \qb) + \zb^{(m)}_{\pi_k}\big), \\
\zb'' = \ & \text{LayerNorm}\big(\text{FeedForward}_{\theta_\text{FF}}\!(\zb') + \zb'\big),
\end{align*}
where $\text{FeedForward}_{\theta_\text{FF}}$ is a module with parameters $\theta_\text{FF}$.
We repeat this process from \cref{eq:copula-keysvals} for each layer, with different parameters, and replacing $\zb^{(m)}_{\pi_k}$ by the output $\zb''$ of the previous layer.
Finally, we obtain the parameters of the conditional distribution using a module parameterized by $\theta_\text{dist}$ applied to the $\zb''$ of the last layer:
$$
\phi^{\pi}_{ck} = \text{DistParams}_{\theta_\text{dist}}\!(\zb'').
$$

\paragraphtight{Choice of distribution}
Any distribution with support on $[0, 1]$ can be used to model the conditional distributions $c^{\pi}_{\phi_{ck}}$.
We choose to use a piecewise constant distribution, i.e., we divide the support into a number of bins, each parametrized by a probability density that applies to all the points it contains. Such a distribution can approximate complex multimodal distributions on $[0, 1]$ without making parametric assumptions, similarly to \citet{vandenOord2016wavenet}. The number of bins is a hyperparameter and controls approximation quality.
Other valid choices include the Beta distribution and mixtures thereof.

\paragraphtight{Sampling}
We first draw a sample from the copula, autoregressively, following an arbitrary permutation $\pi$.
As per the definition of attentional copulas, the first element of the permutation is always sampled from a $U_{[0, 1]}$.
Then, we transform each of the sampled values using their corresponding inverse marginal CDF, i.e., $F^{-1}_{\phi_k}(u^{(m)}_k)$.
This is possible since DSFs are invertible functions (see \cref{app:inverting_flow} for details).

This concludes the presentation of the decoder. However, one question remains: what guarantees that $c_{\phi^\pi_c}$ will converge to a valid copula?
The key is in the training procedure.

\subsection{Training Procedure}
\label{sec:tactis-training}

Let $g_{\phi^{\pi}}$ be the density estimator described in \cref{eq:model_expansion}, where the copula is factorized according to permutation $\pi$.
Let $\Theta \eqdef \left\{\theta_\text{enc}, \theta_\text{dec}\right\}$ be the set containing the parameters of all the components of the encoder and the decoder, respectively.
We obtain $\Theta$ by minimizing the expected negative log-likelihood of our model over permutations drawn uniformly at random%
\footnote{Note that we do not differentiate through the sampling of permutations. We simply consider a random permutation at each forward propagation.}
from the set of all permutations $\Pi$ and samples drawn from the set of all time series $\Scal$ :
\begin{equation}\label{eq:loss}
    \argmin{\Theta} \expect_{\substack{\pi \sim \Pi\mathstrut \\ \Xb \sim \Scal}} -\log g_{\phi^{\pi}}\!\!\left(x^{(m)}_1, \ldots, x^{(m)}_{n_m}\right).
\end{equation}

\begin{theorem}\label{thm:copula}
    The attentional copula $c_{\phi^{\pi}_c}$ embedded in a density estimator $g_{\phi^{\pi}}$, as shown in \cref{eq:model_expansion}, with distributional parameters $\phi^{\pi}$ given by a decoder with parameters $\theta_\text{dec} \in \Theta$, where $\Theta$ minimizes \cref{eq:loss}, is a valid copula.
\end{theorem}
\begin{proof}
    The proof, given in \cref{app:copula-proof}, relies on the fact that optimizing \cref{eq:loss} leads to permutation invariance which, by the definition of $c_{\phi^{\pi}_{c1}}$, results in $U_{[0, 1]}$ marginal distributions.
    The resulting copula $c_{\phi^{\pi}_{c}}$ is therefore a valid copula.
\end{proof}

Hence, \tactis{} provably learns to disentangle the joint dependency structure of random variables from their marginal distributions via flexible non-parametric attentional copulas.

\section{Experiments}\label{sec:results}

We start by presenting an experiment that supports the validity of attentional copulas (\cref{sec:valid-att-copula}). Then, we demonstrate the state-of-the-art performance of \tactis{} using a forecasting benchmark composed of several real-world datasets (\cref{sec:benchmark}). Finally, we present a series of experiments that emphasize the flexibility of the \tactis{} model (\cref{sec:flexibility}).

\subsection{Empirical Validation of Attentional Copulas}\label{sec:valid-att-copula}

\cref{thm:copula} guarantees that, at convergence to a minimum of \cref{eq:loss}, attentional copulas will be valid copulas. However, it does not tell us if this setting is reachable with finite amounts of data, model capacity, and training time.
Hence, we conduct a simple experiment, where we generate data from a distribution with a known copula structure and verify if the \tactis{} decoder correctly recovers the ground truth. For simplicity and ease of visualization, we use a bivariate distribution, where the underlying copula is an x-shaped mixture of two Clayton copulas.
The details are given in \cref{app:gtcopula}.

The results are shown in \cref{fig:gt-copula-learned}.
We observe that the learned copula density closely matches the ground truth (\cref{fig:gt-copula-learned}a).
Furthermore, its marginal distributions are indistinguishable from the $U_{[0, 1]}$ distribution (\cref{fig:gt-copula-learned}b), making it a valid copula.
This experiment, albeit simple, provides empirical evidence that learned attentional copulas can be valid, even in practical settings.
See \cref{app:gtcopula} for additional results and discussion.

\begin{figure}
    \centering
    \subcaptionbox{Copula density}
    {\hspace{5mm}\includegraphics[width=0.8\linewidth, trim=0 3mm 0 0, clip]{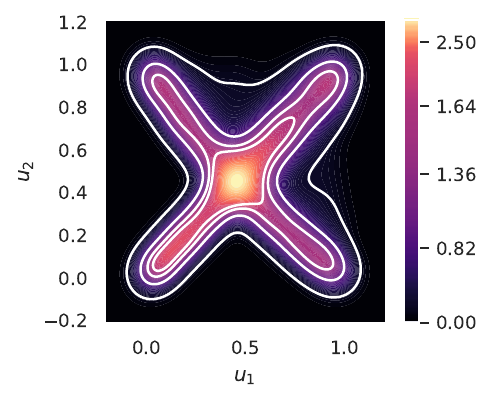}}
    \par\smallskip
    \subcaptionbox{Marginal PDFs}
    {\hspace{-6mm}\includegraphics[width=0.8\linewidth, trim=0 2mm 0 0, clip]{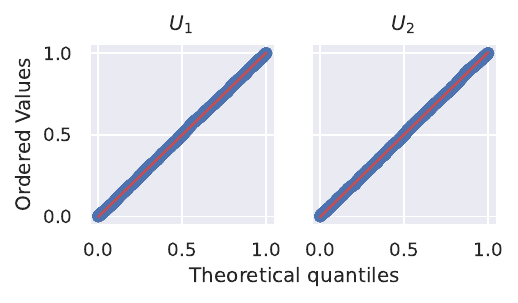}}
    \caption{
        \emph{Attentional copulas successfully learn valid copulas.}
        a) The learned copula's density (white contours) closely matches the ground truth (heatmap). Note: the support of both distributions is $[0, 1]$; any visible overflow in the figure is due to plotting artefacts. b) The marginal distributions of the learned copula are indistinguishable from $U_{[0, 1]}$, as shown by Q-Q plots against $U_{[0, 1]}$.
        }
    \label{fig:gt-copula-learned}
\end{figure}

\subsection{Forecasting: Comparison to the State of the Art}\label{sec:benchmark}

\begin{table*}[t]
    \centering
    \caption{CRPS-Sum means ($\pm$ standard errors which are autocorrelation-corrected to account for sequential backtesting using the Newey-West (\citeyear{NeweyWest1987,NeweyWest1994}) estimator) for the backtesting benchmark, and average rank of each method across datasets (lower is better for both measures). Best results are in bold.} %
    \setlength{\tabcolsep}{4.5pt}
    \label{table:results-CRPS-Sum}
    \footnotesize\renewcommand{\arraystretch}{1.15}%
    \smallskip

  \setlength{\tabcolsep}{0.45ex}
  \hspace*{-0.5ex}%
  \begin{tabular}{@{}rcccccc@{}}
   \toprule
    \multicolumn{1}{@{}c}{\textbf{Model}} & \multicolumn{1}{c}{\texttt{electricity}} & \multicolumn{1}{c}{\texttt{fred-md}} & \multicolumn{1}{c}{\texttt{kdd-cup}} & \multicolumn{1}{c}{\texttt{solar-10min}} & \multicolumn{1}{c}{\texttt{traffic}} & \multicolumn{1}{@{}c@{}}{\textbf{Avg. Rank}} \\
   \midrule
    Auto-\textsc{arima}   
                & $0.077 \pm 0.016$    & $0.043 \pm 0.005$    & $0.625 \pm 0.066$    & $0.994 \pm 0.216$    & $0.222 \pm 0.005$    & $4.7 \pm 0.3$        \\
    ETS          & $0.059 \pm 0.011$    & $\B{0.037 \pm 0.010}$& $0.408 \pm 0.030$    & $0.678 \pm 0.097$    & $0.353 \pm 0.011$    & $4.4 \pm 0.3$        \\
    TempFlow     & $0.075 \pm 0.024$    & $0.095 \pm 0.004$    & $0.250 \pm 0.010$    & $0.507 \pm 0.034$    & $0.242 \pm 0.020$    & $3.9 \pm 0.2$        \\
    TimeGrad     & $0.067 \pm 0.028$    & $0.094 \pm 0.030$    & $0.326 \pm 0.024$    & $0.540 \pm 0.044$    & $0.126 \pm 0.019$    & $3.6 \pm 0.3$        \\
    GPVar        & $0.035 \pm 0.011$    & $0.067 \pm 0.008$    & $0.290 \pm 0.005$    & $\B{0.254 \pm 0.028}$& $0.145 \pm 0.010$    & $2.7 \pm 0.2$        \\
    \tactistt{}  & $\B{0.021 \pm 0.005}$& $0.042 \pm 0.009$    & $\B{0.237 \pm 0.013}$& $0.311 \pm 0.061$    & $\B{0.071 \pm 0.008}$& $\B{1.6 \pm 0.2}$    \\
   \bottomrule
  \end{tabular}
\end{table*}

We now assess the performance of \tactis{} in comparison with state-of-the-art forecasting methods.
Of particular interest is whether the model's great flexibility is detrimental to the quality of its predictions.

\paragraphtight{Baselines}
We benchmark against multiple \emph{deep-learning-based methods} that generate multivariate probabilistic forecasts, namely: GPVar~\citep{salinas2019high}, an LSTM-based method that parametrizes a Gaussian copula; TempFlow~\citep{rasul2021multivariate}, a transformer-based method that models the predictive distribution using normalizing flows; and TimeGrad~\citep{rasul2021autoregressive}, an autoregressive method based on diffusion models.
We also compare to \emph{classical methods}: \textsc{arima}~\citep{box2015time} and ETS exponential smoothing~\citep{hyndman2008forecasting}.
The comparison is based on \tactistt{}, a variant of \tactis{} that uses temporal transformer layers in the encoder (see \cref{sec:tactis-encoder}).
In addition to these comparisons, we provide a detailed ablation study in \cref{app:ablation}.

\paragraphtight{Datasets}
We consider five real-world datasets from the \emph{Monash Time Series Forecasting Repository} \citep{godahewa2021monash}: \texttt{electricity}, \texttt{fred-md}, \texttt{kdd-cup}, \texttt{solar-10min}, and \texttt{traffic} (see \cref{app:choice-datasets} for details).
These were selected for being high-dimensional (107--862 variables), exempt of missing values,
and sampled at diverse frequencies (10 min., hourly, monthly).

\paragraphtight{Evaluation procedure}
Model accuracy is assessed via a backtesting procedure, which mimics the use of forecasting models in real-world settings. A detailed presentation can be found in \cref{app:backtesting}.
In short, we define a series of retraining timestamps for each dataset.
At each timestamp, the models are trained with all of the preceding data, and their accuracy is assessed using subsequent data.
We then report metrics aggregated over all timestamps.
The hyperparameters of each method are selected based on the protocol and grids described in \cref{app:hyper_search} and \cref{app:hp-ranges}, respectively.

\paragraphtight{Metrics}
We use the CRPS-Sum~\citep{salinas2019high}, a multivariate extension of the univariate Continuous Ranked Probability Score (CRPS) \citep{matheson1976scoring}, as our main evaluation metric (see \cref{app:metrics} for a detailed presentation).
In short, this metric corresponds to the CRPS of the univariate series obtained by summing forecasts along the variable axis.
For completeness, we also report results for two additional metrics in \cref{app:metrics}: the CRPS and the energy score~\citep{gneiting2007strictly}.
Finally, we assess how well each method does as a \emph{general forecasting algorithm}, rather than a dataset-specific one, by 
measuring the average rank of each method, w.r.t. all others, over all datasets and retraining timestamps.

\paragraphtight{Benchmark results}
The CRPS-Sum results are reported in \cref{table:results-CRPS-Sum}. From these, it is clear that \tactistt{} compares favourably to the state of the art. It achieves the lowest CRPS-Sum for $3$ out of $5$ datasets and outperforms most baselines on the remaining ones. In fact, \tactistt{} outperforms all deep-learning-based methods on \texttt{fred-md} and outperforms all but GPVar on \texttt{solar-10min}. Furthermore, it achieves the lowest average rank ($1.6$), suggesting that, if one had to choose a method to use without prior knowledge of the data, \tactistt{} would be the better option.
Hence, these results suggest that the great flexibility of \tactis{}, which we highlight in the next section, does not seem to undermine its performance.

\subsection{Model Flexibility}\label{sec:flexibility}

We now present a series of experiments that emphasize the flexibility of the \tactis{} model, namely its support for interpolation, unaligned and non-uniformly sampled data, and its ability to scale to hundreds of time series.

\paragraphtight{Prediction Beyond Forecasting}
\tactis{} relies on a Boolean-valued mask to determine which values of a multivariate time series must be predicted (see~\cref{sec:problem-setting}).
This enables it to support arbitrary prediction tasks, such as forecasting, interpolation, and even combinations thereof.
Here, we demonstrate support for interpolation by showing that \tactis{} can correctly estimate the distribution of a gap in observed values within a stochastic volatility process~\citep{kim1998stochastic}.
Specifically, we train \tactis{} to estimate the distribution of missing values centered within a univariate time series sampled from such a process.
We then compare the estimated joint conditional distribution to the ground truth posterior distribution of missing values. A typical result, where \tactis{} closely approximates the ground truth, is shown in~\cref{fig:exp-interpolation}.
Additional results, as well as a full description of the data generation and experimental protocols are available in~\cref{app:interpolation}.

\begin{figure}
    \centering
    \includegraphics[width=0.95\linewidth]{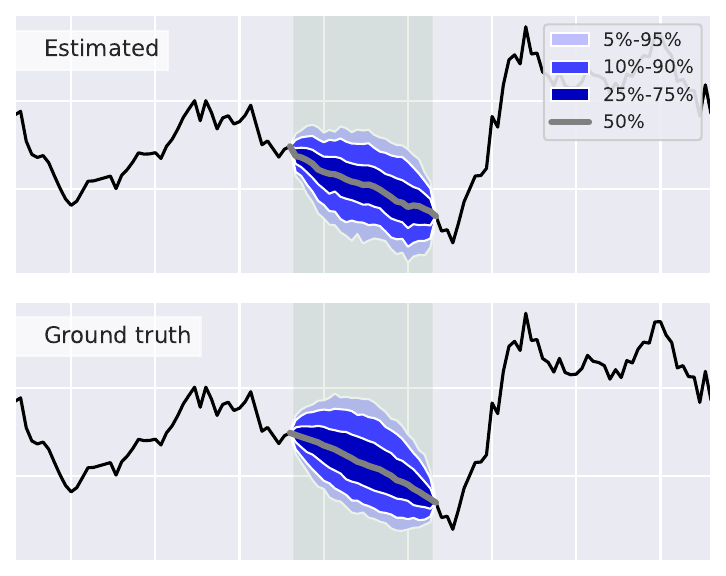}
    \caption{\tactis{} \emph{successfully interpolates} missing values (green shaded region) within a stochastic volatility process. The estimated posterior distribution of missing values (top) closely matches the ground truth (bottom).}%
    \label{fig:exp-interpolation}%
    \vspace*{-1em}
\end{figure}

\paragraphtight{Unaligned and non-uniformly sampled series}
One particularity of \tactis{} is that it considers each observed data point, in each time series, as a distinct token over which to perform self-attention.
The model operates on the set of input tokens, irrespective of their alignment and sampling frequencies, enabling  native support for unaligned and non-uniformly sampled time series.%
\footnote{The \tactistt{} variant does not support this setting (see \cref{sec:tactis-encoder}).}
Here, we conduct a simple experiment to demonstrate that \tactis{} operates well in this setting. We sample data from a bivariate noisy sine process with observations spaced randomly in each series (see \cref{app:unaligned} for details).
We then train \tactis{} to forecast the distribution of missing values at the end of each series.
As shown in \cref{fig:exp-unaligned}, \tactis{} produces accurate forecasts for this data, illustrating its support for unaligned and non-uniformly sampled time series.

\begin{figure}
    \centering
    \includegraphics[width=\linewidth, trim=7mm 0 0 0, clip]{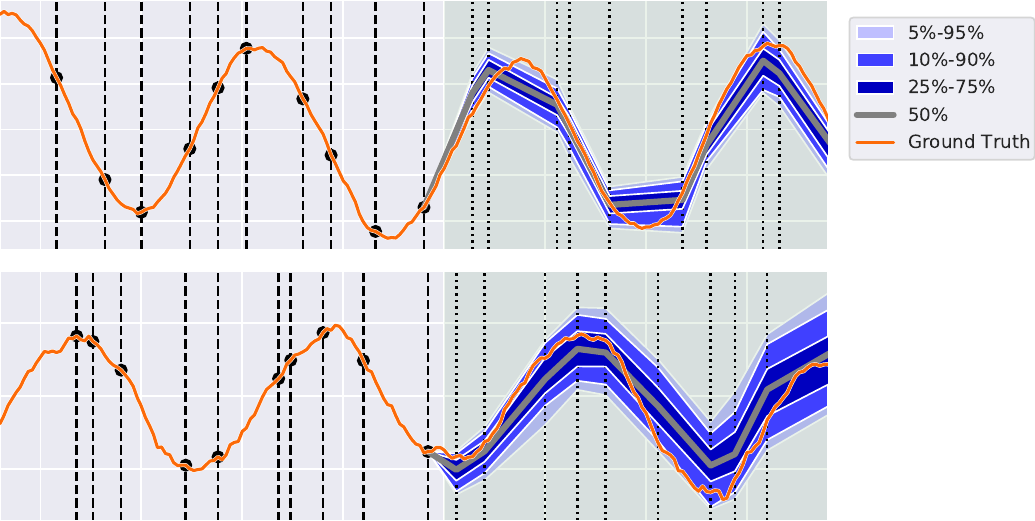}
    \caption{\tactis{} \emph{supports unaligned and non-uniformly sampled time series}, as shown in the above forecasts of a bivariate noisy sine process. Observation and prediction timestamps are marked by dashed vertical lines. The forecasted portion of the time series is shaded in green.}%
    \label{fig:exp-unaligned}%
    \vspace*{-12pt}
\end{figure}

\paragraphtight{Scaling to hundreds of time series}
One key properties of attention-based models, such as \tactis{}, is that they can seamlessly be applied to data of varying dimensionality, without retraining. We make use of this property to devise a scalable training procedure for \tactis{}, which we detail in~\cref{app:bagging}. In short, we train the model using batches composed of a random subset of $b \ll n$ time series, called a bag ($b = 20$ in the forecasting benchmark). This significantly limits the running time and memory usage of the model. Then, at inference time only, we apply the model to all series. In~\cref{table:bagging-maintext}, we explore the effect of $b \geq 2$ on the predictive performance of \tactistt{}. The results indicate that this parameter has very little impact on the accuracy of the model. However, as we show in~\cref{app:baggingexp}, small values of $b$, such as $b = 1$, can negatively affect the learned inter-series correlations. Hence, the model can be trained efficiently without incurring a significant penalty in terms of predictive performance as long as the bag size is not unreasonably small.

\begin{table}[h!]
    \centering
    \caption{\tactis{} \emph{can be trained efficiently via bagging.} This table shows the mean and standard errors of the CRPS, CRPS-Sum, and energy score metrics for various bagging sizes. The results were obtained using \tactistt{} on the \texttt{fred-md} dataset.}
    \footnotesize\renewcommand{\arraystretch}{1.15}%
    \smallskip
    \setlength{\tabcolsep}{1.45ex}
    \begin{tabular}{@{}crrr@{}}
        \toprule
        \multicolumn{1}{@{}c}{\textbf{Bagging size}} & \multicolumn{1}{c}{CRPS} & \multicolumn{1}{c}{CRPS-Sum} & \multicolumn{1}{c}{Energy} \\
                   & ---        & ---        & $\times 10^5$  \\
     \midrule
          2  & $0.046 \pm 0.009$ & $0.039 \pm 0.008$ & $7.88 \pm 1.70$ \\
          5  & $0.052 \pm 0.009$ & $0.044 \pm 0.008$ & $9.36 \pm 1.88$ \\
          10 & $0.047 \pm 0.009$ & $0.039 \pm 0.009$ & $7.90 \pm 1.71$ \\
          15 & $0.047 \pm 0.010$ & $0.040 \pm 0.009$ & $7.87 \pm 1.78$ \\
          20 & $0.049 \pm 0.009$ & $0.042 \pm 0.008$ & $8.30 \pm 1.68$ \\
          25 & $0.049 \pm 0.010$ & $0.042 \pm 0.009$ & $8.22 \pm 1.82$ \\
          30 & $0.049 \pm 0.009$ & $0.043 \pm 0.009$ & $8.34 \pm 1.74$ \\
    \bottomrule
    \end{tabular}\label{table:bagging-maintext}
\end{table}

\section{Discussion}

This work proposes \tactis{}, a method for probabilistic time series inference that combines the flexibility of attention-based models with the density estimation capabilities of a new type of non-parametric copula, termed \emph{attentional copula}.
In addition to achieving state-of-the-art performance on tasks such as probabilistic forecasting and interpolation, we showed that \tactis{} reaches an unprecedented level of flexibility:
it can infer missing values at arbitrary time points in multivariate time series (via masking) can learn from unaligned/non-uniformly sampled data, can be trained when subsets of the data are missing at random, can handle the presence of observed non-stochastic covariates, and can estimate complex distributions beyond the reach of classical copula models, such as the Gaussian copula.

That said, there are several interesting directions in which \tactis{} could be extended.
First, the model's ability to learn multivariate dependencies may benefit from using positional encodings specifically designed for temporal data, rather than those of \citet{vaswani2017attention}.
Second, the applications of recent advances in large-scale transformers (e.g., \citet{choromanski2021rethinking}) to \tactis{} could significantly reduce the amount of resources required by the model, especially in the sampling phase, where bagging is not applied.
The elaboration of more efficient sampling procedures (e.g., based on learning conditional independences) also constitutes a promising prospect.
Third, a thorough study of the training dynamics of \tactis{} may reveal architecture changes or auxiliary tasks that could significantly accelerate learning.
Fourth, \tactis{} could be extended to series measured in discrete domains by adapting the estimation of marginal distributions in the decoder.

Finally, we believe that this work could serve as a basis for models that address the cold-start problem, making sensible predictions in contexts where very few historical observations of the process are available.
In fact, \tactis{} could be trained on time series from a wealth of domains, reusing the same attentional copula, but fine-tuning its encoder to new, unforeseen domains.
Such extensions towards \emph{foundation models}~\citep{bommasani2021opportunities} for probabilistic time series constitute exciting prospects.

\section*{Acknowledgements}
The authors are grateful to G. Abuhamad, P. Beaudoin, D. Berger,
I. Laradji, C.-W. Huang, A. Lacoste, P.-A. Noël,
S. Paquet, and P. Rodriguez-Lopez for thoughtful suggestions.

\bibliography{tactis}
\bibliographystyle{icml2022}

\newpage
\appendix
\onecolumn
\begin{spacing}{0}
\section*{Appendix -- Table of Contents}

\contentsline {section}{\numberline {A}Theory: Proof of \cref {thm:copula}}{\pageref{app:copula-proof}}{appendix.A}%
\contentsline {section}{\numberline {B}Implementation Details}{\pageref{app:implementation}}{appendix.B}%
\contentsline {subsection}{\numberline {B.1}Libraries Used}{\pageref{app:libraries}}{subsection.B.1}%
\contentsline {subsection}{\numberline {B.2}Inverting the Marginal Flows}{\pageref{app:inverting_flow}}{subsection.B.2}%
\contentsline {subsection}{\numberline {B.3}Bagging: Efficient Training in High Dimensions}{\pageref{app:bagging}}{subsection.B.3}%
\contentsline {subsection}{\numberline {B.4}Data Normalization}{\pageref{app:data_norm}}{subsection.B.4}%
\contentsline {section}{\numberline {C}Forecasting Benchmark}{\pageref{app:experiment-setup}}{appendix.C}%
\contentsline {subsection}{\numberline {C.1}Datasets}{\pageref{app:choice-datasets}}{subsection.C.1}%
\contentsline {subsection}{\numberline {C.2}Training Procedure}{\pageref{app:training_procedure}}{subsection.C.2}%
\contentsline {subsection}{\numberline {C.3}Hyperparameter Search Protocol}{\pageref{app:hyper_search}}{subsection.C.3}%
\contentsline {subsection}{\numberline {C.4}Hyperparameter Ranges}{\pageref{app:hp-ranges}}{subsection.C.4}%
\contentsline {subsection}{\numberline {C.5}Backtesting Protocol}{\pageref{app:backtesting}}{subsection.C.5}%
\contentsline {subsection}{\numberline {C.6}Metrics and Additional Results}{\pageref{app:metrics}}{subsection.C.6}%
\contentsline {section}{\numberline {D}Additional Experiments}{\pageref{app:additional-experiments}}{appendix.D}%
\contentsline {subsection}{\numberline {D.1}Can Attentional Copulas Recover a Ground-Truth Copula?}{\pageref{app:gtcopula}}{subsection.D.1}%
\contentsline {subsection}{\numberline {D.2}Can \tactis{} Learn to Interpolate?}{\pageref{app:interpolation}}{subsection.D.2}%
\contentsline {subsection}{\numberline {D.3}Can \tactis{} Learn from Unaligned/Non-Uniformly Sampled Time Series?}{\pageref{app:unaligned}}{subsection.D.3}%
\contentsline {subsection}{\numberline {D.4}Ablation Study}{\pageref{app:ablation}}{subsection.D.4}%
\contentsline {section}{\numberline {E}A Deeper Dive into \tactis{} Models}{\pageref{app:tactis-deepdive}}{appendix.E}%
\contentsline {subsection}{\numberline {E.1}Some Good and Bad Forecasts}{\pageref{sec:app-deepdive-forecasts}}{subsection.E.1}%
\contentsline {subsection}{\numberline {E.2}Looking into Learned Marginal Distributions}{\pageref{sec:app-deepdive-marginals}}{subsection.E.2}%
\contentsline {subsection}{\numberline {E.3}Learning Dependencies Between Variables}{\pageref{app:correlations}}{subsection.E.3}%
\end{spacing}

\clearpage
\section{Theory: Proof of \cref{thm:copula}} \label{app:copula-proof}

\begin{reptheorem}{thm:copula}
    The copula $c_{\phi^{\pi}_c}$ embedded in a density estimator $g_{\phi^{\pi}}$, as shown in \cref{eq:model_expansion}, with distributional parameters $\phi^{\pi}$ given by a decoder with parameters $\theta_\text{dec} \in \Theta$, where $\Theta$ minimizes \cref{eq:loss}, is a valid copula.
\end{reptheorem}

\begin{proof}
    To show that $c_{\phi^\pi_c}\!\!\left(u_1^{(m)}, \ldots, u_{n_m}^{(m)}\right)$ is the density of a valid copula, we need to show two properties: 1) it is a distribution on the unit cube $\left[0, 1\right]^{n_m}$, 2) the marginal distribution of each random variable is uniform.
    
    \emph{Property (1)} is trivially satisfied since, by construction, the support of each conditional distribution is limited to the $[0, 1]$ interval (see \cref{sec:tactis-pseudocopula}, paragraph \emph{Choice of distribution}).

    \emph{Property (2)} is a consequence of minimizing the problem in \cref{eq:loss}. Recall that the parameters of the decoder, $\theta_\text{dec}$, are obtained by minimizing the following expression:
    \begin{align}
        \expect_{\substack{\pi \sim \Pi\mathstrut \\ \Xb \sim \Scal}} -\log g_{\phi^\pi}\!\!\left(x_1^{(m)}, \ldots, x_{n_m}^{(m)}\right)
        &= \expect_{\Xb \sim \Scal} -\dfrac{1}{|\Pi|} \sum_{\pi \in \Pi} \log g_{\phi^\pi}\!\!\left(x_1^{(m)}, \ldots, x_{n_m}^{(m)}\right) \nonumber\\
        &= \expect_{\Xb \sim \Scal} -\dfrac{1}{|\Pi|} \log \left[ \prod_{\pi \in \Pi} g_{\phi^\pi}\!\!\left(x_1^{(m)}, \ldots, x_{n_m}^{(m)}\right) \right] \nonumber\\
        &= \expect_{\Xb \sim \Scal} -\log \left[ \prod_{\pi \in \Pi} g_{\phi^\pi}\!\!\left(x_1^{(m)}, \ldots, x_{n_m}^{(m)}\right) \right]^{|\Pi|^{-1}}\!\!.
    \label{eq:copula-geomean}
    \end{align}

    Notice that the term inside the logarithm in \cref{eq:copula-geomean} corresponds to a geometric mean.
    This quantity is always smaller or equal to the arithmetic mean, and equality is reached i.i.f.\ all elements over which the mean is calculated are equal.
    Hence, we can rewrite the expression in \cref{eq:copula-geomean} as:
    \begin{equation*}\label{eq:copula-arimean-plus-slack}
        \expect_{\Xb \sim \Scal}
        -\log \left[ \dfrac{1}{|\Pi|} \sum_{\pi \in \Pi} g_{\phi^\pi}\!\!\left(x_1^{(m)}, \ldots, x_{n_m}^{(m)}\right) \right] + \delta,
    \end{equation*}
    where $\delta \in \reals^+$ is exactly zero i.i.f.\ the density estimated by the model is permutation invariant, i.e., $g_{\phi^\pi}\!\!\left(x_1^{(m)}, \ldots, x_{n_m}^{(m)}\right) = g_{\phi}\!\!\left(x_1^{(m)}, \ldots, x_{n_m}^{(m)}\right), \forall \pi \in \Pi$.

    Based on this expression, we conclude that the parameters $\theta_\text{dec}$ that minimize the problem in \cref{eq:loss} lead a density estimator that (i) is invariant to permutations $\pi$ and (ii) that minimizes the negative log-likelihood of the data $\expect_{\Xb \sim \Scal} -\log g_{\phi}\!\!\left(x_1^{(m)}, \ldots, x_{n_m}^{(m)}\right)$. It naturally follows from \cref{eq:model_expansion} that the embedded copula density $c_{\phi^\pi_c}\!\!\left(u_1^{(m)}, \ldots, u_{n_m}^{(m)}\right)$ is also permutation invariant.

    Now, recall that, by construction, the marginal density of the first element in a permutation $c_{\phi^\pi_{c1}}(u_1)$, is taken to be that of a $U_{[0, 1]}$ (see \cref{sec:tactis-pseudocopula}, paragraph \emph{Copula density}).
    Given that the copula density $c_{\phi^\pi_{c1}}\!\!\left(u_1^{(m)}, \ldots, u_{n_m}^{(m)}\right)$ is invariant to permutations, the marginal distribution of all variables must necessarily be $U_{[0, 1]}$.
    Thus, \emph{Property (2)} is satisfied.
    
    Since \emph{Properties (1) and (2)} are both satisfied, we conclude that the attentional copula $c_{\phi^{\pi}_c}$, with parameters obtained from a decoder with parameters $\theta_\text{enc}$ is a valid copula.

\end{proof}

\section{Implementation Details}\label{app:implementation}

\subsection{Libraries Used}\label{app:libraries}

The version of \tactis{} used in this work is implemented in PyTorch~\citep{NEURIPS2019_9015}. It relies on the PyTorchTS library~\citep{Rasul_PyTorchTS}, which allows the integration of PyTorch models with the GluonTS library \citep{gluonts_jmlr}, on which we rely heavily in our experiments for data processing, model training, and evaluation.
The implementation is available at \href{{https://github.com/servicenow/tactis}}{https://github.com/servicenow/tactis}.

\subsection{Inverting the Marginal Flows} \label{app:inverting_flow}

When sampling from the learned joint distribution, it is needed to compute the inverse of the marginal CDF for each variable, i.e., to invert the marginal flows: $F^{-1}_{\phi_k}(u^{(m)}_k)$.
Since $F_{\phi_k}(x^{(m)}_k)$ is strictly monotonic by construction, many search algorithms can be applied.
We choose to rely on binary search due to its numerical stability and ease of implementation.
While such searches are relatively slow, the overhead in compute time is negligible compared to other computations in the decoder, which are dominated by the transformer layers in the copula.

One weakness of using flows as marginal distributions is that there is very little pressure toward having well-regularized tails ($u^{(m)}_k \approx 0$ or $u^{(m)}_k \approx 1$).
When sampling from these flows, they will rarely produce values that are much smaller or larger than those in the observed data due to these tails not having the correct shape.
Given the difficulty of training the flows to avoid this issue entirely, we opted to consider only a portion of the marginals when sampling.
That is, instead of sampling from the full $u^{(m)}_k \in [0, 1]$ range in the attentional copula, we rescale sampled values to be in the $[0.05, 0.95]$ range: $x^{(m)}_k = F^{-1}_{\phi_k}(0.05 + 0.9 \times u^{(m)}_k)$.
This issue and alternative solutions have also been explored in \citet{wiese2019copula}.

\subsection{Bagging: Efficient Training in High Dimensions} \label{app:bagging}

Two of the models in our benchmarks: GPVar and \tactistt{}, can be trained with arbitrary subsets of the $n$ time series in the data without having to adjust their parameters.\footnote{Note that all variants of \tactis{} that we consider also support such bagging during training.}
Hence, the memory footprint of the model during the training phase can be significantly reduced by considering bags of randomly selected time series.
In our experiments, each training batch for these models is a bag of $20$ time series. At sampling time, the models are applied to the full set of time series.

As mentioned in \cref{app:training}, we increase the number of batches per epoch to compensate for the reduced amount of data available in each batch due to bagging. See \cref{table:bagging-maintext} and \cref{app:baggingexp} for a study of the effect of the bag size on the accuracy of \tactistt{}.

\subsection{Data Normalization} \label{app:data_norm}

It is often desirable for a model to be scale and translation invariant.
For \tactis{}, we thus transform the data according to what we call the "Standardization" procedure.
For each sample, we compute the means and variances for each series \emph{using only the values of the observed tokens} $\left\{ x_{ij}^{(o)} \right\}$.
The values of both the observed and missing tokens is then transformed using
\begin{equation}
    \tilde{x}_{ij} = \frac{x_{ij} - \text{mean}_i}{\sqrt{\text{variance}_i}}.
\end{equation}
After sampling, we can undo this transformation using
\begin{equation}
    x_{ij} = \sqrt{\text{variance}_i} \times \tilde{x}_{ij} + \text{mean}_i.
\end{equation}
Note that we consider a lower bound of $10^{-16}$ for the variance to avoid division by zero in cases where all values are (nearly) identical.
The precise value of this lower bound has no impact on the training when all values are identical. Yet, it has a massive one when sampling since the sampled values $\tilde{x}_{ij}$ are unlikely to all be zero when the marginal flows are not perfectly fitted.
Choosing a very small lower bound thus minimizes this issue.

\section{Forecasting Benchmark} \label{app:experiment-setup}

\subsection{Datasets} \label{app:choice-datasets}

\begin{table}[h!]
\centering
\caption{Detailed presentation of the datasets used in our benchmark, along with the short name by which they are referred to in the paper. Clicking on the short names links to the exact versions of the datasets that were used.}
\footnotesize\renewcommand{\arraystretch}{1.25}%
\resizebox{\textwidth}{!}{%
    \begin{tabular}{l l c c c}
     \toprule
     Short name & Monash name & Frequency & Number of series & Prediction length \\
     \midrule
     \href{https://doi.org/10.5281/zenodo.4656132}{\texttt{electricity}} & Electricity Hourly Dataset & 1 hour & 321 & 24 \\ 
     \href{https://doi.org/10.5281/zenodo.4654833}{\texttt{fred-md}} & FRED-MD Dataset & 1 month & 107 & 12 \\
     \href{https://doi.org/10.5281/zenodo.4656756}{\texttt{kdd-cup}} & KDD Cup Dataset (without Missing Values) & 1 hour & 270 & 48 \\
     \href{https://doi.org/10.5281/zenodo.4656144}{\texttt{solar-10min}} & Solar Dataset (10 Minutes Observations) & 10 minutes & 137 & 72 \\ 
     \href{https://doi.org/10.5281/zenodo.4656132}{\texttt{traffic}} & Traffic Hourly Dataset & 1 hour & 862 &  24\\ 
     \bottomrule
    \end{tabular}%
}
\label{table:datasets-def}
\end{table}

\cref{table:datasets-def} describes the five datasets that are included in our benchmark.
These datasets were selected due to being publicly available in the Monash Time Series Forecasting Repository \citep{godahewa2021monash}, not containing missing values, having a large (but reasonable) number of dimensions, and being sampled at diverse frequencies.
A modified version of the \texttt{electricity} dataset is often used for benchmarking in the literature, allowing a rough comparison of our results with those of other authors and reinforcing our belief that they are correct.
A variant of the \texttt{solar-10min} dataset, with hourly frequency, is also often used in the literature. We opted for the 10 minutes version to include such a high-frequency dataset in the benchmark.
The prediction length for this dataset was limited to 72 (12 hours) due to a prediction length of 144 (24 hours) being too taxing for the transformer-based models.
The \texttt{fred-md} dataset was selected due to the nature of its series: various economic indicators, which exhibit a wide variety of behaviours.
For the \texttt{kdd-cup} dataset, a prediction length of 48 (2 days) was used to challenge the models at producing longer forecasts than for the other hourly-frequency datasets.
Finally, the \texttt{traffic} dataset was selected for being very high-dimensional (862 series).

Each dataset was downloaded from the Monash Repository using GluonTS \citep{gluonts_jmlr} (links to the exact versions are provided in \cref{table:datasets-def}).
Apart from some models that perform preprocessing (e.g., standardization), the data were not modified or preprocessed in any way.

\subsection{Training Procedure}\label{app:training_procedure}

\subsubsection{Deep Learning Models} \label{app:training}

All deep learning models in our benchmarks: TempFlow, TimeGrad, GPVar, and \tactistt{}, are trained using the same procedure, except for a few exceptions for GPVar, which we detail below.
Each training is done in a Docker container giving access to an \textsc{nvidia} Tesla P100 GPU with 12 GB of GPU RAM, 2 CPU cores, and 32 GB of CPU RAM.

\paragraph{Batch size} %
These models are fairly memory-intensive and the amount of resources required varies considerably based on the hyperparameters and the dataset considered.
Hence, it was necessary to devise a batch-size selection procedure that would ensure that the model did not outgrow the available resources.
To select the batch sizes for a given set of hyperparameters, we ran a small number of training iterations with various batch sizes (powers of two from 1 to 256) and kept the largest one which did not result in an out-of-memory error.
While this method is crude, it allows hyperparameters that require less memory to take advantage of the available memory to train faster due to higher parallelization.

\paragraph{Training loop} In each training epoch, the models are presented with 1600 random samples from the training set.
Since the batch size is variable, we consider $\lfloor 1600 / \text{batch size} \rfloor$ batches per epoch.
As is explained in \cref{app:bagging}, the GPVar, and \tactistt{} models use bagging during training and thus only see a random subset of all series at each iteration.
To compensate for this, the number of batches per epoch for these models is increased to $\lfloor  (1600 / \text{batch size}) \times (\text{number of series} / \text{bagging size}) \rfloor$.
The batches are built by randomly sampling (uniformly) windows of length equal to the sum of the prediction and history lengths from the training dataset.
Only complete windows are considered.

After each epoch, we compute a validation score on a subset of the training set used only for validation (see \cref{fig:backtesting}).
The training ends when the first of the following condition is reached:
\begin{itemize}
    \item The model has trained for 200 epochs,
    \item The model has trained for more than 3 days, or
    \item It has been 20 epochs since the best validation score (early stopping).
\end{itemize}
The resulting model, which is used to compute the final metrics, is the one with minimal validation score. Since we compute the validation score at each epoch, the final model can be from an earlier epoch than the last one performed.

\paragraph{Exceptions for GPVar} \label{app:gpvar_training}
Following communication with GPVar authors \citep{salinas_private}, we trained that model on CPU instead of GPU, due to performance concerns.
To compensate for not having a GPU, we allocate 8 CPU cores and 64 GB of CPU RAM instead of the 2 CPU cores and 32 GB of CPU RAM allocated to models trained on GPU (see above).
Furthermore, we used the GluonTS Trainer for GPVar, which implemented its own stopping conditions for training. Following what was done for other models, we added a condition which stopped training after a maximum of 3 days.

\subsubsection{Classical Models}

Our Auto-\textsc{arima} results are obtained by running the \texttt{auto.arima} function~\citep{hyndman2008automatic} of the \texttt{forecast} package~\citep{hyndman2022forecast} in the R programming language~\citep{rcoreteam2020}. This automatically searches the model specification on a per-time series basis. Since the function supports univariate time series only, we run the function independently for each time series (i.e., we treat a $d$-dimensional multivariate problem as $d$ independent univariate problems), with the hyperparameters and parameters being specific to each time series. We restrict the search to \textsc{sarima} models with maximum order $p=3, q=3, P=2, Q=2$, and the default values for the other function parameters. We carry out an automatic Box-Cox transformation for positive-valued time series. For each time series, we limit fitting time to a maximum of 30 minutes, reverting to a simpler \textsc{arma}$(1,1)$ model with no seasonality or Box-Cox transformation if the original fitting fails. A fit is considered to have failed if the following conditions are encountered:
\begin{itemize}
\item Maximum time limit is exceeded;
\item The in-sample fitted values contain non-finite values;
\item Fewer than 20\% of the simulated predictive trajectories contain finite values (e.g. the predictive simulations diverge) or contain values that are not within a factor of 1,000 (in absolute value) of the training observations.
\end{itemize}

Our ETS (error, trend, seasonality) exponential smoothing~\citep{hyndman2008forecasting} results are obtained through the \texttt{ETSModel} implementation within Python's \texttt{statsmodels} package~\citep{seabold2010statsmodels}. This implementation replicates the R implementation within the \texttt{forecast} package discussed above. An automatic hyperparameter search is carried out to select, on a \emph{per-dataset basis}, the following hyperparameters:

\begin{itemize}
    \item Trend: either none or additive;
    \item Seasonality: either none or additive.
\end{itemize}

For robustness across a wide variety of datasets, the error term is always considered additive. As with the \textsc{arima} results, since ETS is univariate, independent fitting and predictive simulations of the model are carried out for each time series.

\subsection{Hyperparameter Search Protocol} \label{app:hyper_search}

For each model and each dataset under consideration, we ran a hyperparameter search to find the best hyperparameters, which were then used when comparing the forecasting quality of the models.
We opted to perform such a search ourselves instead of selecting previously published hyperparameter combinations for two reasons: 1) some datasets that we consider had not been considered when evaluating some methods we compare to, and 2) to maximize the fairness of the comparison by performing an equally extensive search for each model.

In each dataset, the final subset of the time steps is reserved for the backtesting procedure, which will be described in \cref{app:backtesting}.
From the remaining time steps, a window at the end of length equals to 7 times the prediction length is reserved to serve as the validation set, which is used to compute the metrics used to select the best hyperparameters.
All the remaining data, i.e., what comes before the validation set, is used as the training set.
See \cref{fig:backtesting} for a visual representation of this split.

For each model, we considered 50 hyperparameter combinations, randomly selected amongst a range of values defined for each hyperparameter of the model.
The models are trained 5 times for each combination of hyperparameters, using random initializations and data sampling orders.
For each of these training runs, we compute forecasts using 100 samples for each prediction window in the validation set and average their CRPS-Sum.

We then selected the best hyperparameter combination as being the one where none of the 5 training runs failed due to numerical or memory errors and for which the worst CRPS-Sum value amongst the 5 training runs was the lowest.
We considered the worst of the 5 runs instead of the mean or median since we observed that, for some architectures, some hyperparameters led to unstable training, which resulted in significant variability in the quality of the results.
The rationale for avoiding such hyperparameter combinations is that such instability would be undesirable in real-world applications.

\subsection{Hyperparameter Ranges}\label{app:hp-ranges}

In this section, we list the values considered for each hyperparameter of each method.
The ranges considered for the baselines are inspired by those published in their respective papers and implementations.

\cref{table:hyper-tactis} shows the hyperparameters for the \tactistt{} model.

\begin{table}[h!]
\centering
\caption{Possible hyperparameters for \tactistt{}. \besthp{e}, \besthp{f}, \besthp{k}, \besthp{s}, and \besthp{t} respectively indicate the optimal hyperparameters for \texttt{electricity}, \texttt{fred-md}, \texttt{kdd-cup},  \texttt{solar-10min}, and \texttt{traffic}. \textsuperscript{*}\! The values shown for the Decoder MLP hidden dimensions are the results of our hyperparameters search, not those used in backtesting (see \cref{app:tactis_backtest}).}
\footnotesize\renewcommand{\arraystretch}{1.25}%
\begin{tabular}{@{}r l l@{}}
 \toprule
 & Hyperparameter & Possible values \\
 \midrule
 Model    & Encoder transformer embedding size (per head) and feed forward network size & $8$\besthp{k}, $16$\besthp{f}, $24$\besthp{est} \\
          & Encoder transformer number of heads & $1$\besthp{fks}, $2$\besthp{et}, $3$ \\
          & Encoder number of transformer layers pairs & $1$\besthp{s}, $2$\besthp{ef}, $3$\besthp{kt} \\
          & Encoder time series embedding dimensions & $5$\besthp{efkst} \\
          & Decoder MLP number of layers & $1$\besthp{ekst}, $2$\besthp{f}, $4$ \\
          & Decoder MLP hidden dimensions\textsuperscript{*} & $8$\besthp{eft}, $16$, $24$\besthp{ks} \\
          & Decoder transformer number of heads & $1$, $2$, $3$\besthp{efkst} \\
          & Decoder transformer embedding size (per head) & $8$\besthp{efst}, $16$\besthp{k}, $24$ \\
          & Decoder number transformer layers & $1$\besthp{ef}, $2$\besthp{s}, $3$\besthp{kt} \\
          & Decoder number of bins in conditional distribution & $20$\besthp{efst}, $50$\besthp{k}, $100$ \\
          & Decoder DSF number of layers & $2$\besthp{fks}, $3$\besthp{et} \\
          & Decoder DSF hidden dimensions& $8$\besthp{ft}, $16$\besthp{eks} \\
          & Dropout & $0$\besthp{fkst}, $0.1$, $0.01$\besthp{e} \\
 Data     & Normalization & Standardization\besthp{efkst} \\
          & History length to prediction length ratio & $1$\besthp{ft}, $2$\besthp{ks}, $3$\besthp{e} \\
 Training & Optimizer & RMSprop\besthp{efkst} \\
          & Learning rate & $(10^{-4})$\besthp{s}, $(10^{-3})$\besthp{efkt}, $(10^{-2})$ \\
          & Weight decay & $0$\besthp{et}, $(10^{-5})$\besthp{ks}, $(10^{-4})$\besthp{f}, $(10^{-3})$ \\
          & Gradient clipping & $(10^3)$\besthp{efks}, $(10^4)$\besthp{t} \\
 \bottomrule
\end{tabular}
\label{table:hyper-tactis}
\end{table}

\cref{table:hyper-tempflow} shows the hyperparameters for the TempFlow model.
Parameters in \texttt{teletype font} refer to parameters in the authors implementation in PyTorchTS \citep{Rasul_PyTorchTS}, with all other parameters left at their default value.
The choice of possible hyperparameters for the TempFlow model has been discussed with its authors \citep{rasul_private}.

\begin{table}[h!]
\centering
\caption{Possible hyperparameters for TempFlow. \besthp{e}, \besthp{f}, \besthp{k}, \besthp{s}, and \besthp{t} respectively indicate the optimal hyperparameters for \texttt{electricity}, \texttt{fred-md}, \texttt{kdd-cup},  \texttt{solar-10min}, and \texttt{traffic}.}
\footnotesize\renewcommand{\arraystretch}{1.25}%
\begin{tabular}{@{}r l l@{}}
 \toprule
 & Hyperparameter & Possible values \\
 \midrule
 Model    & \texttt{d\_model} & $16$, $32$\besthp{k}, $64$\besthp{es}, $128$\besthp{t}, $256$\besthp{f} \\
          & \texttt{dim\_feedforward\_scale} & $1$\besthp{e}, $2$, $4$\besthp{fkst} \\
          & \texttt{num\_heads} & $1$\besthp{f}, $2$\besthp{s}, $4$\besthp{t}, $8$\besthp{ek} \\
          & \texttt{num\_encoder\_layers} & $1$\besthp{t}, $3$\besthp{k}, $5$\besthp{efs} \\
          & \texttt{num\_decoder\_layers} & $1$\besthp{et}, $3$\besthp{f}, $5$\besthp{ks} \\
          & \texttt{dropout\_rate} & $0$\besthp{est}, $0.01$\besthp{k}, $0.1$\besthp{f} \\
          & \texttt{flow\_type} & \texttt{"RealNVP"}\besthp{ekst}, \texttt{"MAF"}\besthp{f} \\
          & \texttt{n\_blocks} & $1$\besthp{f}, $3$, $5$\besthp{ekst} \\
          & \texttt{hidden\_size} & $8$\besthp{t}, $16$\besthp{f}, $32$\besthp{s}, $64$\besthp{k}, $128$\besthp{e} \\
          & \texttt{n\_hidden} & $1$, $2$\besthp{fk}, $4$\besthp{est} \\
          & \texttt{conditioning\_length} & $8$, $16$\besthp{fs}, $32$\besthp{t}, $64$\besthp{k}, $128$\besthp{e} \\
          & \texttt{dequantize} & \texttt{False}\besthp{ekst}, \texttt{True}\besthp{f} \\
 Data     & History length to prediction length ratio & $1$\besthp{efk}, $2$\besthp{s}, $3$\besthp{t} \\
 Training & Optimizer & Adam\besthp{efkst} \\
          & Learning rate & $(10^{-4})$\besthp{fs}, $(10^{-3})$\besthp{et}, $(10^{-2})$\besthp{k} \\
          & Gradient clipping & $(10^1)$\besthp{kst}, $(10^2)$\besthp{ef}, $(10^4)$ \\
 \bottomrule
\end{tabular}
\label{table:hyper-tempflow}
\end{table}

\cref{table:hyper-timegrad} shows the hyperparameters for the TimeGrad model.
Parameters in \texttt{teletype font} refer to parameters in the authors implementation in PyTorchTS \citep{Rasul_PyTorchTS}, with all other parameters left at their default value.

\begin{table}[h!]
\centering
\caption{Possible hyperparameters for TimeGrad. \besthp{e}, \besthp{f}, \besthp{k}, \besthp{s}, and \besthp{t} respectively indicate the optimal hyperparameters for \texttt{electricity}, \texttt{fred-md}, \texttt{kdd-cup},  \texttt{solar-10min}, and \texttt{traffic}.}
\footnotesize\renewcommand{\arraystretch}{1.25}%
\begin{tabular}{@{}r l l@{}}
 \toprule
 & Hyperparameter & Possible values \\
 \midrule
 Model    & \texttt{num\_layers} & $1$, $2$, $3$\besthp{efkst} \\
          & \texttt{num\_cells} & $20$\besthp{fkt}, $40$\besthp{s}, $60$\besthp{e} \\
          & \texttt{dropout\_rate} & $0$\besthp{fs}, $0.01$\besthp{e}, $0.1$\besthp{kt} \\
          & \texttt{diff\_steps} & $25$, $50$, $100$\besthp{efkst} \\
          & \texttt{beta\_schedule} & \texttt{"linear"}\besthp{fk}, \texttt{"quad"}\besthp{est} \\
          & \texttt{residual\_layers} & $4$\besthp{e}, $8$\besthp{t}, $16$\besthp{fks} \\
          & \texttt{residual\_channels} & $4$\besthp{fks}, $8$, $16$\besthp{et} \\
          & \texttt{scaling} & \texttt{False}\besthp{t}, \texttt{True}\besthp{efks} \\
 Data     & History length to prediction length ratio & $1$\besthp{ft}, $2$\besthp{ks}, $3$\besthp{e} \\
 Training & Optimizer & Adam\besthp{efkst} \\
          & Learning rate & $(10^{-4})$\besthp{ft}, $(10^{-3})$\besthp{e}, $(10^{-2})$\besthp{ks} \\
          & Gradient clipping & $(10^1)$, $(10^2)$\besthp{fks}, $(10^4)$\besthp{et} \\
 \bottomrule
\end{tabular}
\label{table:hyper-timegrad}
\end{table}

\cref{table:hyper-gpvar} shows the hyperparameters for the GPVar model.
Parameters in \texttt{teletype font} refer to parameters in the authors implementation in GluonTS \citep{gluonts_jmlr}, with all other parameters left at their default value.

\begin{table}[h!]
\centering
\caption{Possible hyperparameters for GPVar. \besthp{e}, \besthp{f}, \besthp{k}, \besthp{s}, and \besthp{t} respectively indicate the optimal hyperparameters for \texttt{electricity}, \texttt{fred-md}, \texttt{kdd-cup},  \texttt{solar-10min}, and \texttt{traffic}.}
\footnotesize\renewcommand{\arraystretch}{1.25}%
\begin{tabular}{@{}r l l@{}}
 \toprule
 & Hyperparameter & Possible values \\
 \midrule
 Model    & \texttt{num\_layers} & $1$\besthp{e}, $2$\besthp{k}, $4$\besthp{fst} \\
          & \texttt{num\_cells} & $8$\besthp{t}, $16$\besthp{fs}, $24$\besthp{k}, $32$, $40$\besthp{e} \\
          & \texttt{cell\_type} & \texttt{"lstm"}\besthp{fks}, \texttt{"gru"}\besthp{et} \\
          & \texttt{use\_marginal\_transformation} & \texttt{True}\besthp{efkst}, \texttt{False} \\
          & \texttt{rank} & $8$\besthp{f}, $16$\besthp{ek}, $24$\besthp{st} \\
          & \texttt{dropout\_rate} & $0$\besthp{ft}, $0.01$\besthp{e}, $0.1$\besthp{ks} \\
 Data     & History length to prediction length ratio & $1$, $2$\besthp{st}, $3$\besthp{efk} \\
 Training & Optimizer & Adam\besthp{efkst} \\
          & Learning rate & $(10^{-4})$\besthp{eft}, $(10^{-3})$\besthp{ks}, $(10^{-2})$ \\
          & Weight decay & $0$\besthp{k}, $(10^{-8})$\besthp{ft}, $(10^{-5})$, $(10^{-4})$, $(10^{-3})$\besthp{es} \\
          & Gradient clipping & $10$\besthp{efkst} \\
 \bottomrule
\end{tabular}
\label{table:hyper-gpvar}
\end{table}

\subsection{Backtesting Protocol} \label{app:backtesting}

We evaluate the performance of all models using a backtesting procedure that mimics how they would be applied to real-world problems.
In practice, due to the high cost of hyperparameter search, it would be done seldomly.
Model training with fixed hyperparameters is less expensive and would be done periodically.
Finally, forecasting using a pre-trained model is cheap and would be done as needed.
Therefore, in our framework, we conduct a single hyperparameter search, retrain the model multiple times (periodically), and calculate forecasts at multiple time stamps between retrainings.

\begin{figure}
    \centering
    \includegraphics[width=0.85\textwidth]{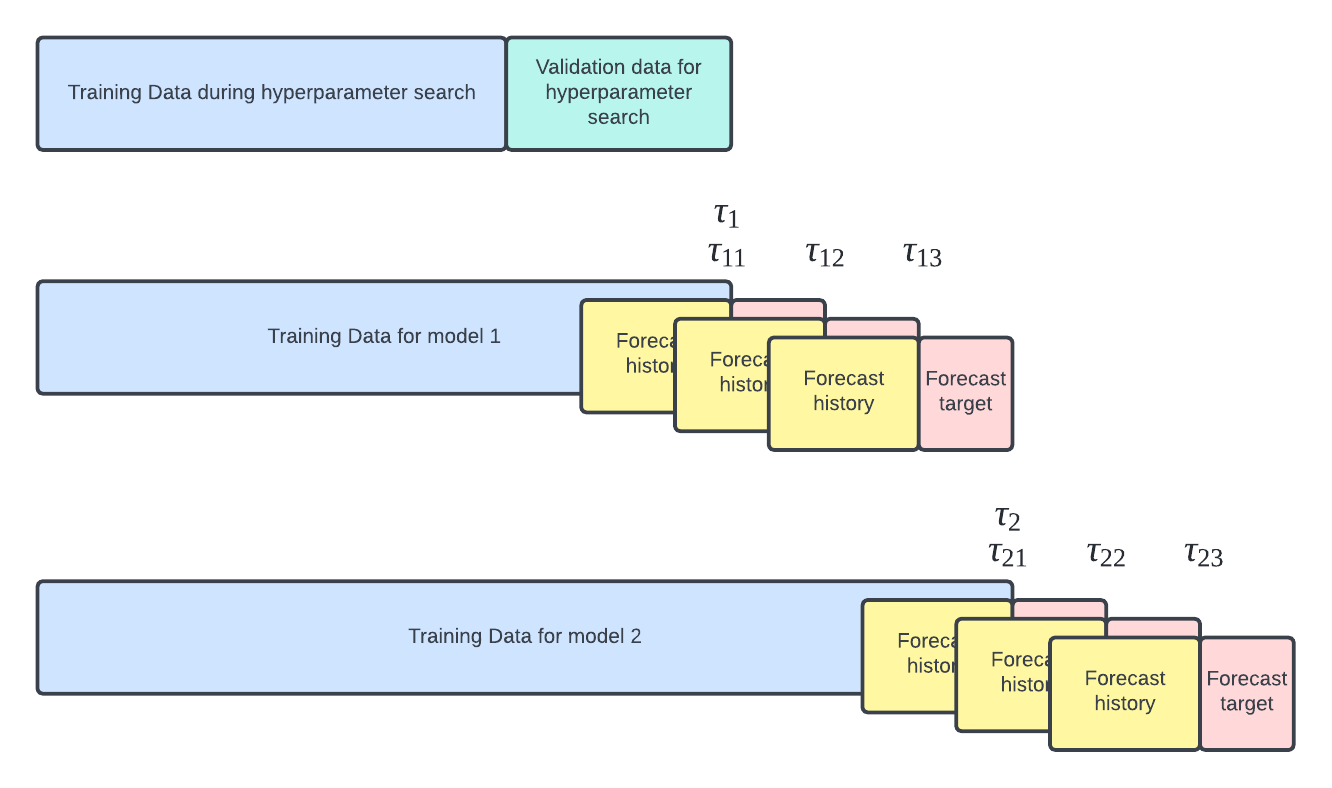}
    \caption{A visual representation of how the data is split during hyperparameter search and backtesting with $n_B = 2$ and $n_F = 3$. During backtesting, the models are trained at times $\tau_i$, and are used for forecasting at times $\tau_{ij}$ (with $\tau_i = \tau_{i1}$).}
    \label{fig:backtesting}
\end{figure}

For each dataset, we define $n_B$ backtesting timestamps $\tau_1$, $\tau_2$, \ldots, $\tau_{n_B}$.
For each backtesting timestamps $\tau_i$, we further define multiple forecasting timestamps $\tau_{i1}$, $\tau_{i2}$, \ldots, $\tau_{in_F}$.
The hyperparameter search described in \cref{app:hyper_search} is done using only data before $\tau_1$, while the $i$-th model is trained using only the data before $\tau_i$.
A visual representation of the split between the data used for model training, the historical data used for the forecast, and the target data to which forecasts are compared is shown in \cref{fig:backtesting}.

We use $n_B = 6$ for all datasets.
For \texttt{electricity}, we select Mondays at midnight as training time $\tau_{i}$, with one week between each, and estimate forecasts every 24 hours between them.
For \texttt{fred-md}, we select Januaries for the $\tau_{i}$, with one year between each, and estimate only a single forecast at the same time.
For \texttt{kdd-cup}, we also select Mondays at midnight for the $\tau_{i}$, with two weeks between each, and estimate forecasts every 48 hours between them.
For \texttt{solar-10min}, we also select Mondays at midnight for the $\tau_{i}$, with one week between each, and estimate forecasts every 12 hours between them.
For \texttt{traffic}, we also select Mondays at midnight for the $\tau_{i}$, with one week between each, and estimate forecasts every 24 hours between them.
The last backtesting timestamp $\tau_6$ for each dataset is selected as the last possible timestamp that follows the respective criterion while still having enough data after it for the full prediction range.

The training is done as explained in \cref{app:training}, except that we do not perform early stopping using a validation set.
Instead, we use the maximum number of epochs that were performed using the best-performing hyperparameters in the hyperparameter search.

\subsubsection{Exceptions for \tactistt{}} \label{app:tactis_backtest}
In experiments conducted using toy datasets, such as small subsets of our main datasets, we found that \tactistt{} had complex learning dynamics, which can be broken down into three stages.
At first, it quickly improves the quality of the marginal distributions while pushing the attentional copula towards the trivial copula (uniform distribution on the unit cube), where all values are independent.
Then, the loss tends to plateau, with no improvement in forecast quality for multiple epochs.
Finally, it moves on to learning a non-trivial attentional copula, and its loss improves again.

While these three phases are not as clear-cut in the main datasets as in the toy datasets, the plateau effect in the second learning phase could cause early stopping to be too aggressive and prevent \tactistt{} from learning a non-trivial copula.
Therefore, when training \tactistt{} in the backtesting procedure, we removed the maximum number of epochs, keeping only the 3 day maximum training duration.
In retrospect, this had minimal impact on the number of epochs performed for all datasets, except for the smallest: \texttt{fred-md}, for which training reached a number of epochs in the order of 7000 during the 3 days limit.

Another important observation that was made using toy datasets is that too small values for the \emph{MLP hidden dimensions} parameter in the decoder would greatly slow the transition from the second to third phases of training.
We thus increased this parameter to 48 for all datasets, ignoring this part of our hyperparameter search.

Given these observations, we believe there is sill work to do in understanding and improving the learning dynamics of \tactis{}. Progress in these directions may lead to significant improvement in the model's empirical performance and the quality of the learned copulas.

\subsection{Metrics and Additional Results} \label{app:metrics}

As mentioned in \cref{sec:results}, we use the CRPS-Sum as our primary metric to select the best hyperparameters and compare the forecasting accuracy of the various models.
The CRPS-Sum is based on the Continuous Ranked Probability Score (CRPS) \citep{matheson1976scoring}, which can be computed for each individual forecasted value for the $i$-th series and the $j$-th timestep as:
\begin{equation}
    \text{CRPS}(X_{ij},x_{ij}) = 
        \expect_{X_{ij}} \left[ \, |X_{ij}-x_{ij}| \, \right] - 
        \frac{1}{2} \expect_{X_{ij},X'_{ij}}\left[ \, |X_{ij}-X'_{ij}| \, \right],
\end{equation}
where $x_{ij}$ is the observed value, and $X_{ij}$ and $X'_{ij}$ are two independent variables from the forecasted sampling process.
The CRPS-Sum is obtained by replacing the individual values by the sum over all series $i$: $s_j = \sum_i x_{ij}$ and $S_j = \sum_i X_{ij}$.
This transformation of the CRPS-Sum allows it to capture the quality of some of the correlations in the forecasts: namely, the average correlation between forecasts at a single time step.
However, the CRPS-Sum metric is inherently univariate and cannot discriminate whether the forecasting method accurately predicts the correlation between different time steps or not.

\paragraph{Computing Standard Errors}

In all our results tables, we compute the standard errors (indicated by the $\pm$) of all metrics using an adjustment to the variance across individual results to account for the autocorrelation that arises from the sequential nature of our backtesting procedure. Specifically, we use a Newey-West correction for standard errors, computed using Bartlett kernel weights \citep{NeweyWest1987,NeweyWest1994}, with 3 lags (since there were only 6 backtesting folds). The automatic bandwidth selection procedure described in \citet{NeweyWest1994} is used, as implemented in the R \texttt{sandwich} package. This correction produces more conservative (i.e. wider) standard errors than those computed under an i.i.d. assumption and, as such, yields a fairer assessment of the methods under study.

\paragraph{CRPS Results} 
The benchmark results for the CRPS-Sum and CRPS are respectively shown in \cref{table:results-CRPS-Sum,table:results-CRPS}.
We used the GluonTS \citep{gluonts_jmlr} implementation of said metrics, which are thus normalized by dividing the CRPS of each individual variable by the mean absolute value of its observed values.

\begin{table}[h!]
\centering
      \caption{CRPS means ($\pm$ standard errors which are autocorrelation-corrected to account for sequential backtesting using the Newey-West (\citeyear{NeweyWest1987,NeweyWest1994}) estimator) for the backtesting benchmark, and average rank of each method across datasets (lower is better for both measures). Best results are in bold.}
\footnotesize\renewcommand{\arraystretch}{1.15}%
    \smallskip
    \setlength{\tabcolsep}{0.45ex}
    \hspace*{-0.5ex}%
    \begin{tabular}{@{}rcccccc@{}}
     \toprule
      \multicolumn{1}{@{}c}{\textbf{Model}} & \multicolumn{1}{c}{\texttt{electricity}} & \multicolumn{1}{c}{\texttt{fred-md}} & \multicolumn{1}{c}{\texttt{kdd-cup}} & \multicolumn{1}{c}{\texttt{solar-10min}} & \multicolumn{1}{c}{\texttt{traffic}} & \multicolumn{1}{@{}c@{}}{\textbf{Avg. Rank}} \\
 \midrule
  Auto-\textsc{arima}   
               & $0.129 \pm 0.015$    & $0.052 \pm 0.005$    & $0.477 \pm 0.015$    & $0.636 \pm 0.060$    & $0.310 \pm 0.004$    & $4.1 \pm 0.3$        \\
  ETS          & $0.094 \pm 0.014$    & $0.050 \pm 0.011$    & $0.560 \pm 0.028$    & $0.844 \pm 0.119$    & $0.437 \pm 0.012$    & $4.9 \pm 0.2$        \\
  TempFlow     & $0.109 \pm 0.024$    & $0.110 \pm 0.003$    & $0.451 \pm 0.005$    & $0.547 \pm 0.036$    & $0.320 \pm 0.015$    & $4.1 \pm 0.3$        \\
  TimeGrad     & $0.101 \pm 0.027$    & $0.142 \pm 0.058$    & $0.495 \pm 0.023$    & $0.560 \pm 0.047$    & $0.217 \pm 0.015$    & $3.7 \pm 0.3$        \\
  GPVar        & $0.067 \pm 0.010$    & $0.086 \pm 0.009$    & $0.459 \pm 0.009$    & $\B{0.298 \pm 0.034}$& $0.213 \pm 0.009$    & $2.9 \pm 0.1$        \\
  \tactistt{}  & $\B{0.052 \pm 0.006}$& $\B{0.048 \pm 0.010}$& $\B{0.420 \pm 0.007}$& $0.326 \pm 0.049$    & $\B{0.161 \pm 0.009}$& $\B{1.4 \pm 0.1}$    \\
 \bottomrule
\end{tabular}
\label{table:results-CRPS}
\end{table}

\paragraph{Energy Score} 
An alternative metric is the energy score \citep{gneiting2007strictly}.
Like the CRPS and CRPS-Sum metrics, it is a proper scoring method; thus, a perfect forecast would minimize it.
Unlike the CRPS and CRPS-Sum metrics, it is sensitive to the correlations between all forecasted variables, thus allowing a good multivariate forecasting technique to shine compared to one that forecasts each series independently.
The energy score is defined as:
\begin{equation}\label{eq:energy-score}
    \text{energy}(\left\{X,x\right\}) = \expect_{X} \| X - x \|_F^\beta - \frac{1}{2} \expect_{X,X'} \| X - X' \|_F^\beta,
\end{equation}
where $x$ is the observed value in matrix form, $X$ and $X'$ are two independent variables from the forecasted sampling process in matrix forms, $\beta$ is a parameter set to 1, and $\|\cdot\|_F$ is the Frobenius matrix norm.

The benchmark results for the energy score are shown in \cref{table:results-energy}.
Unlike the CRPS, which is normalized, the energy scores we report are not, which explains the orders of magnitude differences between the datasets.
However, it should be noted that these results were obtained from models with hyperparameters selected to minimize the CRPS-Sum and not the energy score.
It is thus likely that some hyperparameters are not optimal in terms of energy score, so these results may not be the best that the models can do.

\begin{table}[h!]
\centering
      \caption{Energy score means ($\pm$ standard errors which are autocorrelation-corrected to account for sequential backtesting using the Newey-West (\citeyear{NeweyWest1987,NeweyWest1994}) estimator) for the backtesting benchmark, and average rank of each method across datasets (lower is better for both measures). Best results are in bold.}
\footnotesize\renewcommand{\arraystretch}{1.15}%
    \smallskip
    \setlength{\tabcolsep}{1.45ex}
    \begin{tabular}{@{}rrrrrrr@{}}
     \toprule
      \multicolumn{1}{@{}c}{\textbf{Model}} & \multicolumn{1}{c}{\texttt{electricity}} & \multicolumn{1}{c}{\texttt{fred-md}} & \multicolumn{1}{c}{\texttt{kdd-cup}} & \multicolumn{1}{c}{\texttt{solar-10min}} & \multicolumn{1}{c}{\texttt{traffic}} & \multicolumn{1}{@{}c@{}}{\textbf{Avg. Rank}} \\
               & $\times 10^4$        & $\times 10^5$        & $\times 10^3$        & $\times 10^2$        & $\times 10^0$           & --- \\
 \midrule
  Auto-\textsc{arima}   
               & $44.59 \pm 8.56$     & $8.72 \pm 0.81$      & $18.76 \pm 3.31$     & $19.42 \pm 3.37$     & $4.10 \pm 0.05$      & $4.9 \pm 0.2$        \\
  ETS          & $7.94 \pm 0.93$      & $\B{7.90 \pm 1.88}$  & $3.60 \pm 0.24$      & $4.74 \pm 0.17$      & $4.98 \pm 0.07$      & $4.1 \pm 0.2$        \\
  TempFlow     & $10.25 \pm 2.03$     & $20.16 \pm 0.74$     & $3.30 \pm 0.28$      & $4.25 \pm 0.16$      & $4.59 \pm 0.25$      & $4.3 \pm 0.4$        \\
  TimeGrad     & $9.69 \pm 2.62$      & $19.87 \pm 7.23$     & $3.30 \pm 0.19$      & $4.31 \pm 0.23$      & $3.38 \pm 0.11$      & $3.2 \pm 0.4$        \\
  GPVar        & $6.80 \pm 0.62$      & $11.43 \pm 1.60$     & $3.18 \pm 0.20$      & $\B{2.60 \pm 0.10}$  & $3.57 \pm 0.10$      & $2.6 \pm 0.2$        \\
  \tactistt{}  & $\B{5.42 \pm 0.57}$  & $8.18 \pm 1.83$      & $\B{2.93 \pm 0.22}$  & $2.88 \pm 0.23$      & $\B{3.10 \pm 0.13}$  & $\B{1.8 \pm 0.4}$    \\
 \bottomrule
\end{tabular}\label{table:results-energy}
\end{table}

\section{Additional Experiments}\label{app:additional-experiments}

\subsection{Can Attentional Copulas Recover a Ground-Truth Copula?}\label{app:gtcopula}

In this experiment, we evaluate the density estimation abilities of the \tactis{} decoder.
That is, we evaluate its ability to correctly recover the copula and the marginal distributions that underlie a joint distribution.
Doesn't the result in \cref{thm:copula} already guarantee this?
No, \cref{thm:copula} guarantees that, at convergence to a minimum of \cref{eq:loss}, the resulting attentional copula ($c_{\phi^{\pi}_c}$) will be a valid copula, but it does not tell us if this setting is reachable given finite amounts of data, capacity, and training time.
Below, we show that a valid copula can be learned in practical conditions.

\paragraphtight{Experimental setup} For simplicity and ease of visualization, we focus on a simple bivariate joint density estimation problem. We define a bivariate data distribution with known marginal distributions and copula structure. We then train the \tactis{} decoder using samples from this distribution and evaluate its ability to recover the ground truth marginals and copula.

\paragraphtight{Data generation}
We define the joint dependency structure of the variables using a uniformly weighted mixture of two Clayton copulas with parameter $\theta$ equals to $14.75$ and $-0.85$, respectively. This leads to a complex x-shaped copula density (\cref{fig:app-gtcopula-data}b).
As for the marginal distributions, we use Chi-squared distributions and let $X_1 \sim \chi^2\!(5)$ and $X_2 \sim \chi^2\!(10)$, respectively (\cref{fig:app-gtcopula-data}cd).

\begin{figure}
    \centering
    \subcaptionbox{}{\includegraphics[trim=0 0 425 0, clip, scale=0.68]{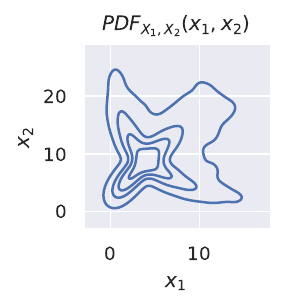}}
    \subcaptionbox{}{\includegraphics[trim=145 0 283 0, clip, scale=0.68]{figures/ground_truth_copula.pdf}}
    \subcaptionbox{}{\includegraphics[trim=285 0 145 0, clip, scale=0.68]{figures/ground_truth_copula.pdf}}
    \subcaptionbox{}{\includegraphics[trim=425 0 0 0, clip, scale=0.68]{figures/ground_truth_copula.pdf}}
    \caption{Overview of the simulated dataset: a) contour plot of the joint PDF of $X_1$ and $X_2$, b) contour plot of the ground truth copula's PDF, c-d) marginal PDF of $X_1$ and $X_2$, respectively. All plots are based on $5000$ samples.}
    \label{fig:app-gtcopula-data}
\end{figure}

\paragraphtight{Model} We isolate the density estimation components of \tactis{}, i.e., those tasked with learning the marginal distributions and the copula. As such, we replace the encoder with a simple set of embeddings (one for each variable in the distribution) and use these embeddings as the $\Zcal^{(m)}$ input to the \tactis{} decoder (see \cref{sec:tactis-pseudocopula}). Since we focus on an unconditional density estimation task, we let the $\Zcal^{(o)}$ input to the decoder be the empty set.
We then train the model by maximum likelihood, using the permutation-based loss in \cref{eq:loss}.
The full configuration of the model is given in \cref{table:hyper-gtcopula}.

\begin{table}[h!]
\centering
\caption{Hyperparameters for \tactis{} for the experiments in \cref{app:gtcopula}.}
\footnotesize\renewcommand{\arraystretch}{1.25}%
\begin{tabular}{@{}r l l@{}}
 \toprule
 & Hyperparameter & Selected values \\
 \midrule
 Model    & Encoder variable embedding dimensions & $3$ \\
          & Decoder MLP number of layers & $2$ \\
          & Decoder MLP hidden dimensions & $30$ \\
          & Decoder transformer number of heads & $1$ \\
          & Decoder transformer embedding size (per head) & $8$ \\
          & Decoder number transformer layers & $2$ \\
          & Decoder number of bins in conditional distribution & $30$ \\
          & Decoder DSF number of layers & $2$ \\
          & Decoder DSF hidden dimensions& $8$ \\
          & Dropout & $0$ \\
 Training & Optimizer & RMSprop \\
          & Learning rate & $(10^{-3})$ \\
          & Batch size & $128$ \\
          & Weight decay & $0$ \\
          & Gradient clipping & none \\
 \bottomrule
\end{tabular}
\label{table:hyper-gtcopula}
\end{table}

\paragraphtight{Results} Our desiderata is that 1) the learned copula and marginal densities closely match those shown in \cref{fig:app-gtcopula-data}b and \cref{fig:app-gtcopula-data}cd, respectively and 2) that the learned copula be valid (see \cref{sec:bg-copula}), i.e., a distribution on the unit cube with $U_{[0, 1]}$ marginals.
As shown in \cref{fig:app-gt-copula-learned}, the learned copula density (a) and the marginal distributions (b) closely approximate the ground truth.
Further, as shown in \cref{fig:app-gtcopula-uniform-marginals}, the marginal distributions of the copula are indistinguishable from the uniform distribution.
Hence, we conclude that 1) the \tactis{} decoder successfully recovered the components of the data distribution and 2) that the attentional copula successfully converged to a valid copula, and this, even in a setting where data, capacity, and training time were limited.

\begin{figure}
    \centering
    \subcaptionbox{Copula density}
    {\includegraphics[scale=0.92]{figures/gt_copula_learned_density_permTrue.pdf}}
    \subcaptionbox{Marginal PDFs}
    {\includegraphics[scale=0.65]{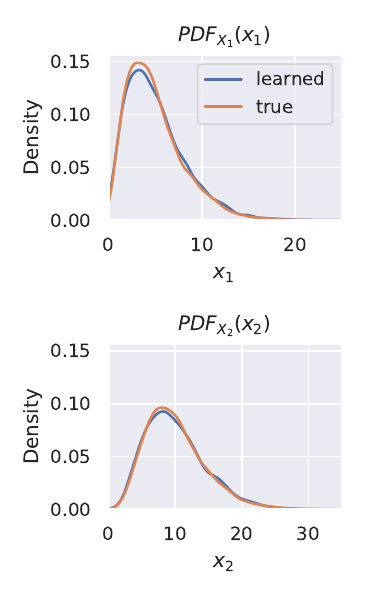}}
    \caption{The \tactis{} decoder successfully recovers the ground truth components of the distribution: a) the learned copula density (white contours) closely matches the ground truth density (heatmap); b) the learned marginals (blue) closely match the ground truth (orange). All plots are based on $5000$ samples.}
    \label{fig:app-gt-copula-learned}
\end{figure}

\begin{figure}
    \centering
    \includegraphics[width=0.5\textwidth]{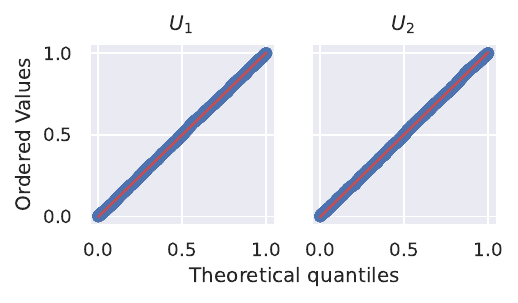}
    \caption{The empirical marginal densities of the learned copula are indistinguishable from that of the $U_{[0, 1]}$ distribution, as shown by these Q-Q plots w.r.t. $U_{[0, 1]}$. All plots are based on $5000$ samples.}
    \label{fig:app-gtcopula-uniform-marginals}
\end{figure}

\paragraphtight{Additional remarks} While conducting these experiments, which allow for easy visualization, we realized that initialization seems to play an important role in the model. In fact, at initialization, the bins in the categorical distribution used in the copula have uniform weights, resulting in a trivial, but valid, copula with $U_{[0, 1]}$ marginals. While this does not correctly capture the dependency between the variables, it enables the model to concentrate on fitting the marginals very precisely. As such, the model sometimes gets stuck in this regime and fails to learn a non-trivial copula.
In these experiments, we circumvented this issue by conducting each experiment with various initializations (via the random seed).
As for our main results, the forecasting benchmark experiments, we ensured that this issue did not arise by verifying that a non-zero correlation existed between some variables in the learned distribution (see \cref{app:ablation}).
Nonetheless, a thorough study of the learning dynamics of \tactis{} is an interesting future direction, which ultimately might help to prevent this pitfall.

\subsection{Can \tactis{} Learn to Interpolate?}
\label{app:interpolation}

In this section, we evaluate the ability of \tactis{} to perform interpolation, i.e., estimating the distribution of missing values for a gap \emph{within} a time series.

\paragraphtight{Data generation} In this experiment, we consider a univariate\footnote{We could also have considered a multivariate setting, but we limit the experiment to a univariate one for simplicity.} time series generated according to the following stochastic volatility process~\citep{kim1998stochastic}:
\begin{align}
    y_t \mid h_t & ~~\sim~~ \mathcal{N}\left(0, \exp~h_t \right)\\
    h_t \mid h_{t-1}, \mu, \phi, \sigma & ~~\sim~~ \mathcal{N}\left( \mu + \phi(h_{t-1} - \mu), \sigma^2 \right)\\
    h_0 \mid \mu, \phi, \sigma & ~~\sim~~ \mathcal{N}\left(\mu, \sigma^2 / (1 - \phi^2)\right),
\end{align}
with a log-variance of level $\mu=-9$, persistence $\phi=0.99$, and volatility $\sigma=0.04$.
The $h_t$ variables are unobserved and correspond to a latent time-varying volatility process.
Our quantity of interest is $x_t \eqdef x_{t-1} + y_t$, where $y_t$ is a stochastic increment incurred at each time step.

\emph{Training data.} We use this process to generate a univariate time series $\xb^{\text{train}}$ of length $10,000$, with $x^{\text{train}}_{1} = 1$.
This time series is used to train \tactis{}.

\emph{Evaluation data.} We then generate 1000 more univariate time series, denoted $\xb^\text{test}_1, \ldots \xb^\text{test}_{1000}$, which correspond to possible trajectories of length $500$ starting at the end of $\xb^\text{train}$.
For each $\xb^\text{test}_i$, we create a gap of missing values of length $25$ at an arbitrary time point.
We then use the Markov Chain Monte-Carlo sampling capabilities of the Stan probabilistic programming language~\citep{Stan} to sample possible interpolation trajectories from the posterior distribution of missing values.
Such trajectories form a ground truth for the interpolation task associated with $\xb^\text{test}_i$.

\paragraphtight{Training protocol} We train \tactis{} using batches of $64$ windows of length $125$ drawn randomly from the training sequence $\xb^\text{train}$.
For each window, we mask the values of the center $25$ time points and let the other time points ($50$ on each side) be observed. The model then learns to estimate the distribution of the $25$ masked values given the $100$ observed ones. The hyperparameter values used in this experiment are reported in \cref{table:tactis-interp-parameters}.

\paragraphtight{Evaluation protocol} For each of the evaluation time series $\xb^\text{test}_i$, we extract a region of length $125$ centered on the gap of missing values.\footnote{We keep only the testing tasks where such sampling was within the bounds of the time series, resulting in $803$ valid tasks $\xb^\text{test}_i$.} We then use \tactis{} to draw samples from the estimated distribution of missing values.
Then, the energy score (see \cref{eq:energy-score}) is used to compare the sampled trajectories to \emph{each} of the ground truth trajectories for $\xb^\text{test}_i$.
This results in a distribution of energy scores.

For comparison, we repeat the same process and obtain a distribution of energy scores for the following baselines:
\begin{itemize}
    \item Oracle: an oracle model that produces samples by randomly drawing from the ground truth trajectories for $\xb^\text{test}_i$.
    \item Dummy: a simple baseline that performs a linear interpolation between the points directly before and after the gap of missing values. This baseline is deterministic, i.e., all its samples are identical.
\end{itemize}
We then use the Wasserstein distance to compare the distribution of scores obtained for \tactis{} and for the Dummy baseline to that of the Oracle, for each testing series $\xb^\text{test}_i$.

\paragraphtight{Results and conclusions}

\cref{fig:interpolation-wd-to-oracle} summarizes the Wasserstein distances obtained for each of the testing series.
Clearly, the distributions estimated by \tactis{} are closer to the ground truth distributions than those estimated by the Dummy baseline.
This is reflected in their respective Wasserstein distance distributions, which have a mean of $0.0391$ and $0.1459$ for \tactis{} and the Dummy baseline, respectively.
We further assess the quality of the distributions estimated by \tactis{} visually.
\cref{fig:interpolation-good-samples} compares the distribution estimated by \tactis{} for the task ($\xb^\text{test}_i$) where the energy score distribution was closest to that of the Oracle ($\textrm{WD}=0.0009$).
We observe that the shape of the estimated distribution is very close to that of the ground truth. Furthermore, its median is coherent with the values observed directly before and after the gap.
We repeat this analysis for the task where the WD was closest to the mean of the \tactis{} WD distribution ($\textrm{WD}=0.0388$) and show the results in \cref{fig:interpolation-average-samples}.
The same conclusions apply, supporting the fact that \tactis{} generally performs well at this interpolation task.
Finally, we repeat the analysis for the task where the WD was maximal ($\textrm{WD}=1.4578$) and show the results in \cref{fig:interpolation-bad-samples}.
In this case, \tactis{} clearly fails to estimate an upwards trend in the median of the ground truth distribution, leading to a median that is not coherent with the bounds of the gap.
Fortunately, tasks where \tactis{} performs at this level are rare, as supported by the distribution of Wasserstein distances.

In conclusion, the aforementioned results clearly show the ability of \tactis{} to perform accurate probabilistic interpolation. It is thus safe to conclude that \tactis{} is sufficiently flexible to address both forecasting and interpolation tasks.

\begin{figure}
    \centering
    \includegraphics[width=0.5\textwidth]{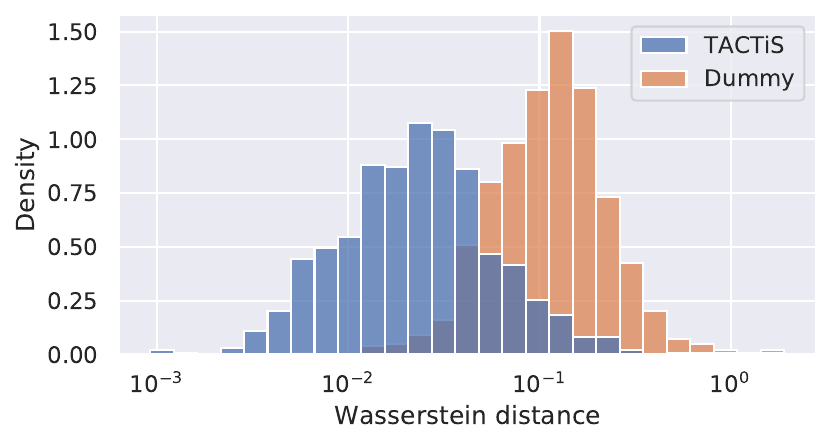}
    \caption{\tactis{} outperforms the Dummy baseline on the considered interpolation tasks, as shown by the distribution of Wasserstein distances between the energy score distributions of each baseline and the Oracle (estimated for our 803 valid testing tasks). The \tactis{} and Dummy distributions have means at $0.0391$ and $0.1459$, respectively. The distances are estimated using 50 samples from each energy score distribution.}\label{fig:interpolation-wd-to-oracle}
\end{figure}

\begin{figure}
    \centering
    \includegraphics[width=0.9\textwidth]{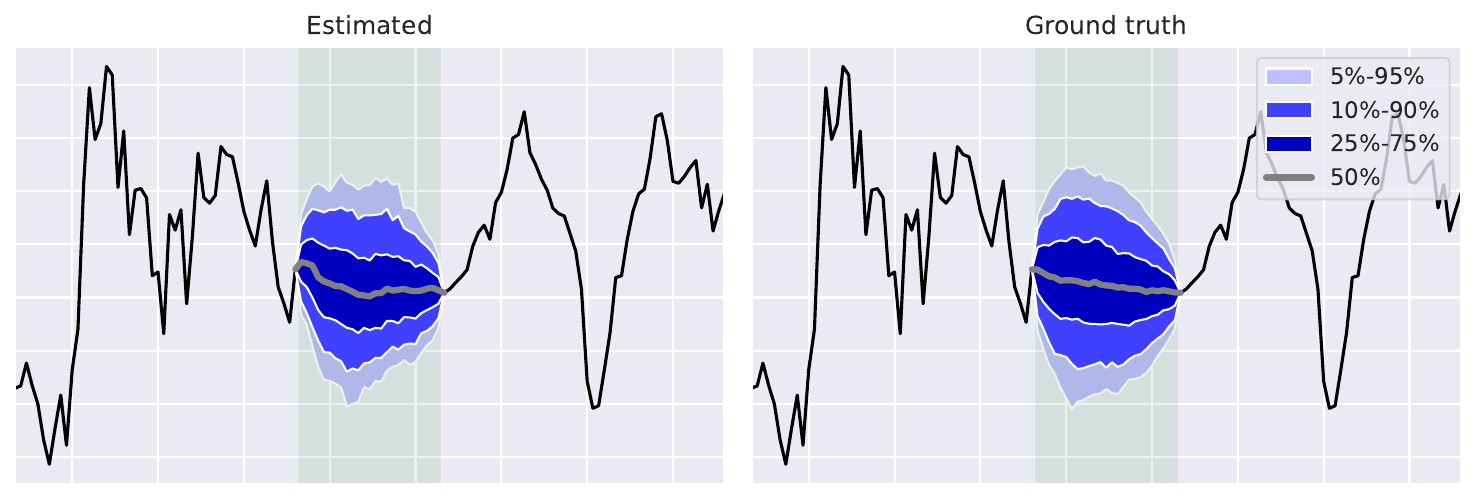}
    \caption{Interpolation - Good performance. Comparison of the interpolation distribution estimated by \tactis{} (left) and the ground truth distribution (right) for the task where \tactis{} most faithfully approximates the Oracle, as measured by the Wasserstein distance between energy score distributions (WD$=0.0009$). The colors correspond to confidence intervals of the distribution. All distributions are estimated using 1000 samples.}
    \label{fig:interpolation-good-samples}
\end{figure}

\begin{figure}
    \centering
    \includegraphics[width=0.9\textwidth]{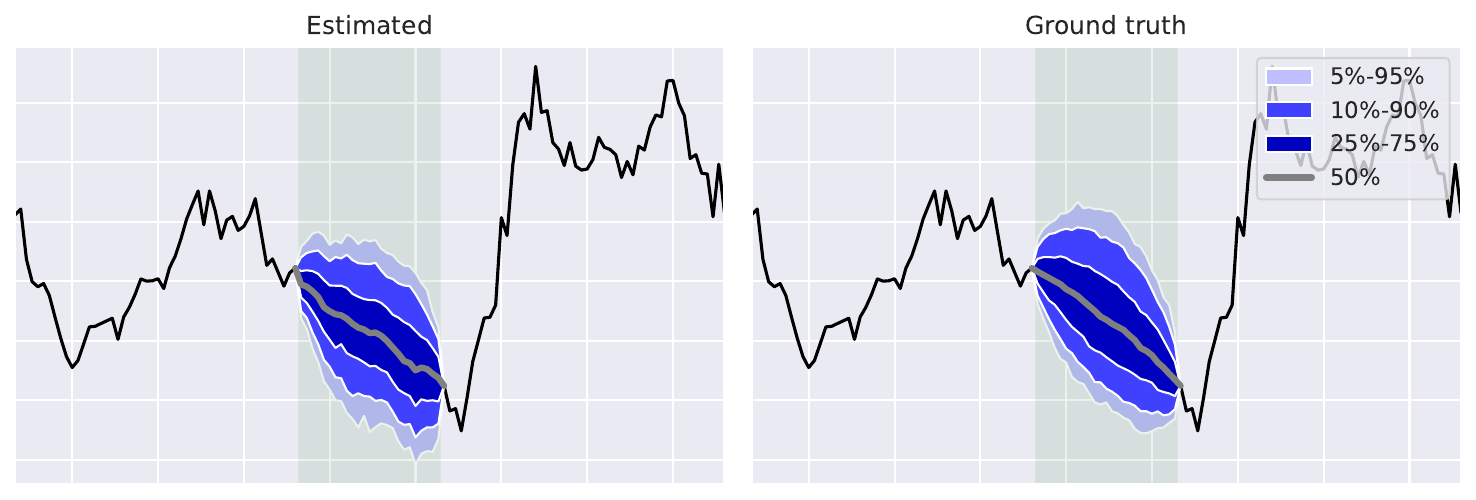}
    \caption{Interpolation - Average performance. Comparison of the interpolation distribution estimated by \tactis{} (left) and the ground truth distribution (right) for a task that is representative of the average performance of \tactis{}, as measured by the Wasserstein distance between energy score distributions (WD$=0.0388$). The colors correspond to confidence intervals of the distribution. All distributions are estimated using 1000 samples.}
    \label{fig:interpolation-average-samples}
\end{figure}

\begin{figure}
    \centering
    \includegraphics[width=0.9\textwidth]{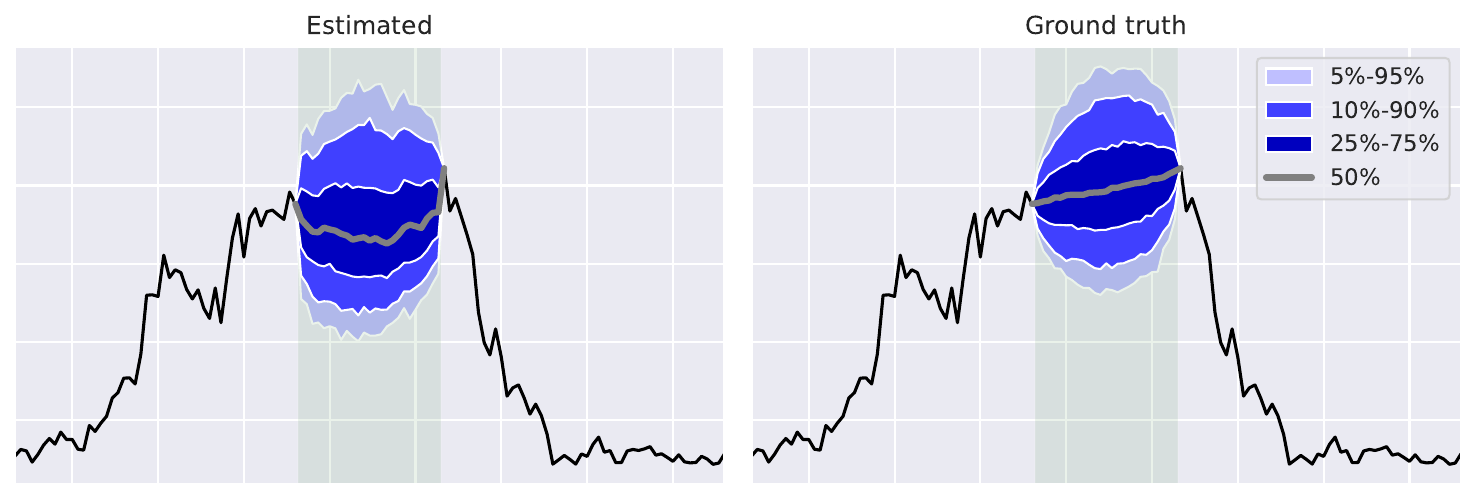}
    \caption{Interpolation - Poor performance. Comparison of the interpolation distribution estimated by \tactis{} (left) and the ground truth distribution (right) for the task where \tactis{} least faithfully approximates the Oracle, as measured by the Wasserstein distance between energy score distributions (WD$=1.4578$). The colors correspond to confidence intervals of the distribution. All distributions are estimated using 1000 samples.}
    \label{fig:interpolation-bad-samples}
\end{figure}

\begin{table}[h!]
    \centering
    \caption{Hyperparameters for \tactis{} for the interpolation experiment.}
    \footnotesize\renewcommand{\arraystretch}{1.25}%
    \begin{tabular}{r l l}
     \toprule
     & Hyperparameter & Selected value \\
     \midrule
     Model    & Encoder transformer embedding size (per head) and feed forward network size & $8$ \\
              & Encoder transformer number of heads & $2$ \\
              & Encoder number of transformer layers & $2$ \\
              & Encoder time series embedding dimensions & $5$ \\
              & Decoder MLP number of layers & $1$ \\
              & Decoder MLP hidden dimensions & $8$ \\
              & Decoder transformer number of heads & $2$ \\
              & Decoder transformer embedding size (per head) & $8$ \\
              & Decoder number transformer layers & $2$ \\
              & Decoder number of bins in conditional distribution & $20$ \\
              & Decoder DSF number of layers & $2$ \\
              & Decoder DSF hidden dimensions & $8$ \\
              & Dropout & $0$ \\
     Data     & Normalization & Standardization \\
              & History length to prediction length ratio (before interpolation range) & $2$ \\
              & History length to prediction length ratio (after interpolation range) & $2$ \\
     Training & Optimizer & RMSprop \\
              & Learning rate & $10^{-3}$ \\
              & Batch size & $64$ \\
              & Weight decay & $10^{-5}$ \\
              & Gradient clipping & $10^3$ \\
     \bottomrule
    \end{tabular}
    \label{table:tactis-interp-parameters}
\end{table}

\subsection{Can \tactis{} Learn from Unaligned/Non-Uniformly Sampled Time Series?}\label{app:unaligned}

The main experiments, presented in \cref{sec:benchmark}, are based on datasets containing only aligned and uniformly-sampled time series.
As such, they cannot determine whether \tactis{} can take advantage of its transformer architecture to support unaligned or non-uniformly sampled time series.
We use this section to bring evidence that \tactis{} is indeed compatible with such time series, using an experiment with a small simulated dataset.

\paragraph{Dataset}
We consider two independent univariate time series, each generated by adding a random walk to a sinusoidal signal, with the second series having a twice-larger random walk amplitude than the first.
From this ground truth, we select a sample of non-uniformly spaced points by selecting one point in each window of 10 time steps in the ground truth, uniformly at random.
These points are revealed to the model and the others are hidden.
To produce unaligned time series, this procedure is performed independently for each series.

\paragraph{Implementation details} To make our implementation of \tactis{} compatible with this setting, we had to use the index of each time point in the ground truth time series to generate the positional encodings. The development of positional encodings better suited for temporal data, e.g., that are generated using a real-valued time stamp, goes beyond the scope of this work. However, any developments could be directly integrated in \tactis{} in a plug-and-play fashion.

\paragraph{Results and conclusions}
\tactis{} was trained to forecast windows of $10$ time points, given a history of the same length.
The hyperparameters used for the model are given in \cref{table:tactis-unaligned-experiment}.
In \cref{fig:unaligned-sine-both}, we show an example forecast produced by \tactis{} on this data.
While this is by no means a thorough study of the performance of \tactis{} in this challenging setting, it does confirm that the model can indeed work with multivariate time series that are unaligned and sampled at various frequencies.

\begin{figure}
    \centering
    \includegraphics[width=0.8\textwidth]{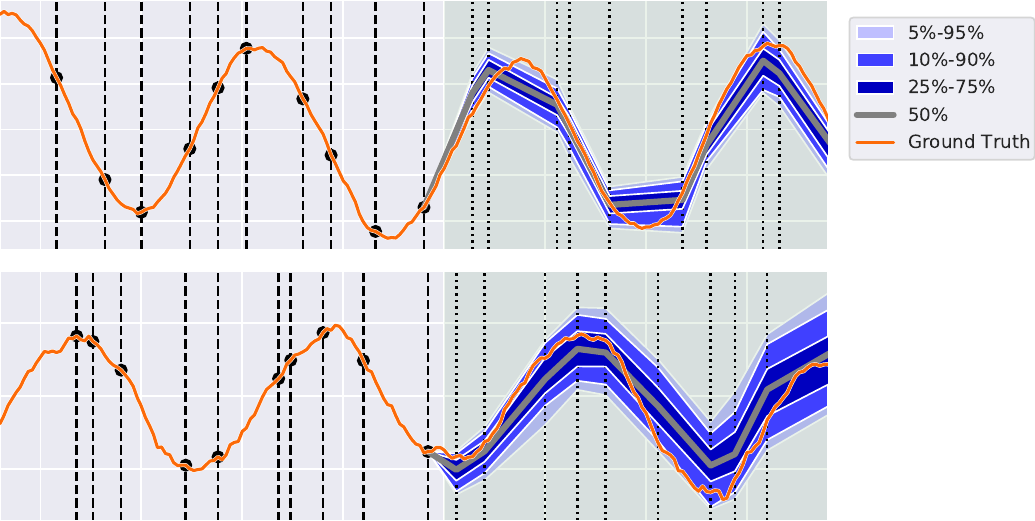}
    \caption{Unaligned and non-uniformly sampled time series. Demonstration of the forecasting ability of \tactistt{} on two independent time series sampled at unaligned and non-uniform time points. The orange line represents the ground truth of the process to be forecasted. The dashed vertical lines and black dots represent the data points given to the model (history), while the dotted lines represent the time points at which the model is asked to produce forecasts.}
    \label{fig:unaligned-sine-both}
\end{figure}

\begin{table}[h!]
    \centering
    \caption{Hyperparameters for \tactis{} for the unaligned/non-uniformly sampled experiment.}
    \footnotesize\renewcommand{\arraystretch}{1.25}%
    \begin{tabular}{r l l}
     \toprule
     & Hyperparameter & Selected value \\
     \midrule
     Model    & Encoder transformer embedding size (per head) and feed forward network size & $8$ \\
              & Encoder transformer number of heads & $2$ \\
              & Encoder number of transformer layers & $2$ \\
              & Encoder time series embedding dimensions & $5$ \\
              & Decoder MLP number of layers & $1$ \\
              & Decoder MLP hidden dimensions & $8$ \\
              & Decoder transformer number of heads & $2$ \\
              & Decoder transformer embedding size (per head) & $8$ \\
              & Decoder number transformer layers & $2$ \\
              & Decoder number of bins in conditional distribution & $20$ \\
              & Decoder DSF number of layers & $2$ \\
              & Decoder DSF hidden dimensions & $8$ \\
              & Dropout & $0$ \\
     Data     & Normalization & Standardization \\
              & History length to prediction length ratio & $1$ \\
     Training & Optimizer & RMSprop \\
              & Learning rate & $10^{-3}$ \\
              & Batch size & $256$ \\
              & Weight decay & $10^{-5}$ \\
              & Gradient clipping & $10^3$ \\
     \bottomrule
    \end{tabular}
    \label{table:tactis-unaligned-experiment}
\end{table}

\subsection{Ablation Study}\label{app:ablation}

In this section, we attempt to measure the relative impact of the following key components of \tactistt{} on the quality of its predictions:
\begin{enumerate}
    \item The self-attention layers in the encoder,
    \item The use of Deep Sigmoidal Flows to estimate marginal distributions in the decoder, and
    \item The use of attentional copulas to model correlations in the decoder.
\end{enumerate}

To achieve this, we compare \tactistt{} to two ablations:
\begin{itemize}
    \item \tactisic{}: replacing the attentional copula with a trivial one, which measures the importance of (3), and
    \item \tactisgc{}: replacing the entire decoder with the Gaussian-copula-like decoder of GPVar~\citep{salinas2019high}, which measures the importance of (2-3).
\end{itemize}
Comparing these two ablations with \tactistt{} also allows a measure of the relative importance of (1). For example, if both ablations perform just as well as \tactistt{}, we may conclude that (1) is the main driver of the model's performance.

Note that it would be challenging to ablate the marginal flows (2) directly without affecting the quality of the learned copula.
For example, suppose that we replaced the marginal flows with an empirical CDF (ECDF), as used in \citet{salinas2019high}. Training via maximum likelihood would push the attentional copula to encode how the univariate forecasts deviate from the ECDF, which would prevent it from converging to a valid copula. Hence, we abstain from conducting such an experiment.

\subsubsection{\tactisgc{}: Using ECDFs and a Gaussian Copula in the Decoder} \label{app:tactis-gc}

Here, we evaluate the \tactisgc{} ablation, where we replace the marginal flows with ECDFs and the attentional copula with a low-rank Gaussian copula, as used in \citet{salinas2019high}.
As is done in \citet{salinas2019high}, the ECDFs are computed, per variable, using the history of each sample and are composed with the inverse CDF of normal distribution of zero mean and unit variance.
Since this ablation adds new components to the model, which come with their own hyperparameters, we re-run the hyperparameter search using the ranges listed in \cref{table:hyper-tactis-gc}.
Of note, we include the use of the ECDF as one of many possible data normalization methods in the hyperparameter search (denoted CDF in the table). It can thus be bypassed if this leads to more accurate predictions, in which case the learned joint distribution is simply a low-rank Gaussian.

\begin{table}[h!]
\centering
\caption{Possible hyperparameters for \tactisgc{}. \besthp{e}, \besthp{f}, \besthp{k}, \besthp{s}, and \besthp{t} respectively indicate the optimal hyperparameters for \texttt{electricity}, \texttt{fred-md}, \texttt{kdd-cup},  \texttt{solar-10min}, and \texttt{traffic}.}
\footnotesize\renewcommand{\arraystretch}{1.25}%
\begin{tabular}{@{}r l l@{}}
 \toprule
 & Hyperparameter & Possible values \\
 \midrule
 Model    & Encoder transformer embedding size (per head) and feed forward network size & $8$\besthp{f}, $16$\besthp{et}, $24$\besthp{ks} \\
          & Encoder transformer number of heads & $1$\besthp{t}, $2$\besthp{efks} \\
          & Encoder number of transformer layers pairs & $1$\besthp{ef}, $2$, $4$\besthp{kst} \\
          & Encoder time series embedding dimensions & $5$\besthp{efkst} \\
          & Decoder MLP number of layers & $1$\besthp{s}, $2$\besthp{kt}, $4$\besthp{ef} \\
          & Decoder MLP hidden dimensions and rank of the low-rank approximation & $8$\besthp{e}, $16$\besthp{kt}, $24$\besthp{fs} \\
          & Dropout & $0$\besthp{est}, $0.1$\besthp{f}, $0.01$\besthp{k} \\
 Data     & Normalization & None\besthp{st}, Standardization\besthp{e}, CDF\besthp{fk} \\
          & History length to prediction length ratio & $1$\besthp{f}, $2$\besthp{et}, $3$\besthp{ks} \\
 Training & Optimizer & RMSprop\besthp{efkst} \\
          & Learning rate & $(10^{-4})$\besthp{s}, $(10^{-3})$\besthp{efkt}, $(10^{-2})$ \\
          & Weight decay & $0$, $(10^{-5})$\besthp{ef}, $(10^{-4})$\besthp{k}, $(10^{-3})$\besthp{st} \\
          & Gradient clipping & $(10^3)$\besthp{efkst}, $(10^4)$ \\
 \bottomrule
\end{tabular}
\label{table:hyper-tactis-gc}
\end{table}

\cref{table:ablation-crps-sum,table:ablation-energy} show the performance of this ablation in comparison with \tactistt{}.
We observe that it performs worse on two datasets: \texttt{electricity} and \texttt{traffic}, slightly worse on one: \texttt{fred-md}, and slightly better on two: \texttt{kdd-cup} and \texttt{solar-10min}.
It should not be surprising that, in some contexts, \tactisgc{} performs well.
The low-rank Gaussian copula limits the learned correlations to stem from a small number of latent variables, which may be particularly well suited in some situations. For example, in \texttt{solar-10min}, it may help forecast the amount of solar energy produced by nearby stations which operate in similar weather conditions.
Nonetheless, since the differences by which \tactisgc{} beats \tactistt{} are much smaller than those by which it loses, we conclude that \tactistt{} is generally superior.

\begin{table}[h!]
\centering
\caption{Means and standard errors of the CPRS-Sum metrics for the ablation experiments.}
\footnotesize\renewcommand{\arraystretch}{1.15}%
    \smallskip
    \setlength{\tabcolsep}{1.45ex}
    \begin{tabular}{@{}rrrrrr@{}}
     \toprule
      \multicolumn{1}{@{}c}{\textbf{Model}} & \multicolumn{1}{c}{\texttt{electricity}} & \multicolumn{1}{c}{\texttt{fred-md}} & \multicolumn{1}{c}{\texttt{kdd-cup}} & \multicolumn{1}{c}{\texttt{solar-10min}} & \multicolumn{1}{c}{\texttt{traffic}} \\
 \midrule
      \tactistt{}  & $0.021 \pm 0.005$ & $0.042 \pm 0.009$ & $0.237 \pm 0.013$ & $0.311 \pm 0.061$ & $0.071 \pm 0.008$ \\
      \tactisgc{}  & $0.054 \pm 0.011$ & $0.050 \pm 0.011$ & $0.214 \pm 0.008$ & $0.283 \pm 0.023$ & $0.125 \pm 0.010$ \\
      \tactisic{}  & $0.022 \pm 0.005$ & $0.042 \pm 0.009$ & $0.303 \pm 0.009$ & $0.381 \pm 0.068$ & $0.080 \pm 0.011$ \\
\bottomrule
\end{tabular}\label{table:ablation-crps-sum}
\end{table}

\begin{table}[h!]
\centering
\caption{Means and standard errors of the energy score metrics for the ablation experiments.}
\footnotesize\renewcommand{\arraystretch}{1.15}%
    \smallskip
    \setlength{\tabcolsep}{1.45ex}
    \begin{tabular}{@{}rrrrrr@{}}
     \toprule
      \multicolumn{1}{@{}c}{\textbf{Model}} & \multicolumn{1}{c}{\texttt{electricity}} & \multicolumn{1}{c}{\texttt{fred-md}} & \multicolumn{1}{c}{\texttt{kdd-cup}} & \multicolumn{1}{c}{\texttt{solar-10min}} & \multicolumn{1}{c}{\texttt{traffic}} \\
               & $\times 10^4$        & $\times 10^5$        & $\times 10^3$        & $\times 10^2$        & $\times 10^0$           \\
 \midrule
      \tactistt{}  & $5.42 \pm 0.57$ & $8.18 \pm 1.83$ & $2.93 \pm 0.22$ & $2.88 \pm 0.23$ & $3.10 \pm 0.13$ \\
      \tactisgc{}  & $8.70 \pm 0.89$ & $9.55 \pm 1.95$ & $2.65 \pm 0.17$ & $2.94 \pm 0.05$ & $4.99 \pm 0.02$ \\
      \tactisic{}  & $5.32 \pm 0.54$ & $8.40 \pm 1.69$ & $5.24 \pm 0.42$ & $9.14 \pm 1.25$ & $10.69 \pm 1.04$ \\
\bottomrule
\end{tabular}\label{table:ablation-energy}
\end{table}

\subsubsection{\tactisic{}: Using a Trivial Copula Instead of an Attentional Copula} \label{app:tactis-ic}

Here, we evaluate the \tactisic{} ablation in which we replace the attentional copula with a trivial copula, where all variables are independent and their distribution is $U_{[0, 1]}$.
The goals of this ablation are twofold: 1) measure the contribution of the marginal flows vs. the dependencies learned by the attentional copulas in the quality of predictions, and 2) make sure that the learned attentional copulas are non-trivial and that they contribute to the forecasts.
Since this ablation does not add new hyperparameters or model parameters, we reuse the trained \tactistt{} models to conduct this experiment.

\cref{table:ablation-crps-sum,table:ablation-energy} show the performance of this ablation in comparison to \tactistt{}.
Interestingly, the CRPS-Sum and energy score metrics paint two different pictures: for CRPS-Sum, the results are quite similar with only two datasets showing a clear advantage for \tactistt{}: \texttt{kdd-cup} and \texttt{solar-10min}. However, for the energy score, there are three datasets for which \tactistt{} dominates \tactisic{}: \texttt{kdd-cup}, \texttt{solar-10min}, and \texttt{traffic}.
This is consistent with CRPS-Sum having a lower discrimination ability than the energy score, as discussed in \citet{koochali2022random}.

The lack of performance degradation on the CRPS-Sum metric shows that learning very good marginal distributions without good correlations is sufficient to achieve good CRPS-Sum values.
In contrast, the energy score is much better at revealing the lack of learned correlations.
Nonetheless, there are two datasets: \texttt{electricity} and \texttt{fred-md}, where we do not see a contrast in energy score, but for which we know that the attentional copula successfully learned non-trivial correlations (see \cref{app:correlations}).
This suggests that either the energy score is also dominated, to some extent, by fitting the marginal distributions very well or that the attentional copula's contribution to forecast quality is less than that of properly fitting the marginals for these datasets.

\subsubsection{Relative Importance of \tactistt{} Key Components}

When comparing the ablations results in \cref{table:ablation-crps-sum,table:ablation-energy} with our benchmarks in \cref{table:results-CRPS-Sum,table:results-energy}, we see that both \tactistt{} and its ablations, \tactisgc{} and \tactisic{}, are quite competitive with the state of the art.
This suggests that the shared portion of their architecture, i.e., the self-attention-based encoder (1), is the primary driver of their good performance.
Now, what can we say about the importance of the marginal flows (2) and the attentional copula (3)?

Our results indicate that these components are important drivers of the good performance but that their individual importance depends on the dataset under consideration.
The importance of the marginal flows (2) is revealed by the \texttt{electricity} and \texttt{fred-md} datasets, where \tactisic{} outperforms \tactisgc{}.
The importance of modelling dependencies between random variables, which is the role of the attentional copula (3), instead of only learning their marginals, is revealed by the \texttt{kdd-cup} and \texttt{solar-10min} datasets, where the \tactisic{} fails in comparison to \tactistt{} and \tactisgc{}.
Finally, the joint importance of the marginal flows (2) and the attentional copula (3), is revealed by the \texttt{traffic} dataset, where \tactistt{} outperforms \tactisic{} and \tactisgc{}.

In summary, we conclude that the main driver of the performance of \tactistt{}'s good performance is its transformer-based architecture, but that each of the components of its decoder, namely the marginal flows and the attentional copula, play a key role in achieving state-of-the-art results on some datasets.

\section{A Deeper Dive into \tactis{} Models}\label{app:tactis-deepdive}

\subsection{Some Good and Bad Forecasts}\label{sec:app-deepdive-forecasts}

In this section, we present a few examples of the probabilistic forecasts generated by \tactistt{} in our forecasting experiments.
Due to the dimensionality of the datasets, it would be infeasible to present all forecasts, so we selected a few examples to demonstrate both the strength and weakness of \tactistt{}.
For each dataset, we thus hand-picked four forecasts that were \emph{particularly good} from the backtesting experiment with the lowest CRPS-Sum and four \emph{particularly bad} forecasts from the backtesting experiment with the highest CRPS-Sum.

\paragraph{The \texttt{electricity} dataset}

\cref{fig:best-electricity} shows some particularly good forecasts for the \texttt{electricity} dataset.
The top- and bottom-left forecasts show that the model could predict that the time series would likely see a sudden increase after a zero-valued gap.
The top- and bottom-right forecasts show the model's ability to follow the daily periodicity of the data.
The top-right forecast is particularly impressive since the model was able to foresee the higher demand that would occur on Monday (2014-12-15), following lower values for Saturday (2014-12-13) and Sunday (2014-12-14), and this, even though our positional encodings do not include day-of-week information.

\begin{figure}
    \centering
    \includegraphics[width=0.95\linewidth]{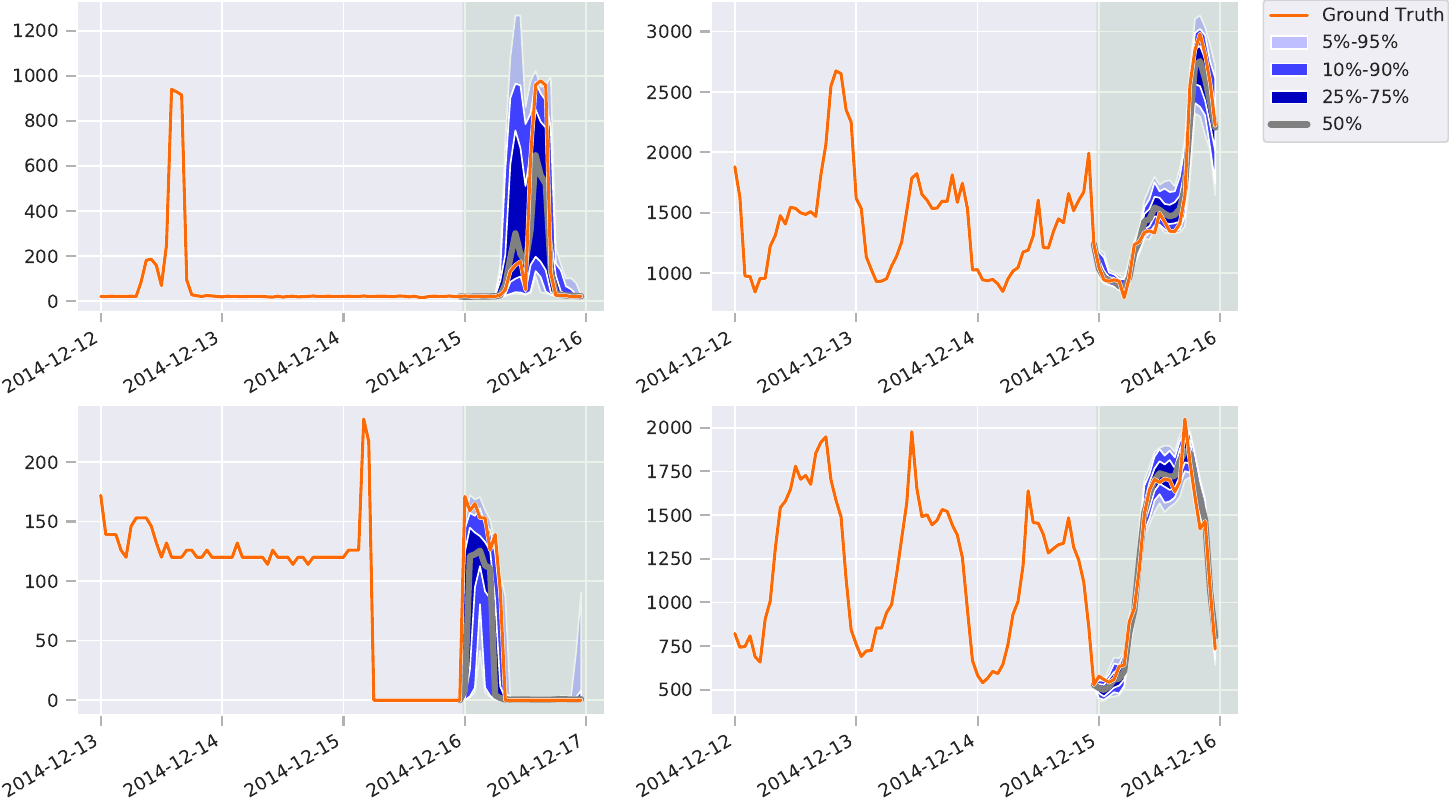}
    \caption{
        Hand-picked example forecasts for the \texttt{electricity} dataset. These examples have been selected to show particularly \textbf{good} forecasts.
        The historical ground truth shown is the one that was made available to the model.
        The data is from 2014, so December 15th is a Monday.
    }\label{fig:best-electricity}
\end{figure}

\cref{fig:worst-electricity} shows some particularly bad forecasts for the \texttt{electricity} dataset.
In the top-left forecast, the distribution of forecasted values is near-uniform, even though the data in this series only take on three possible values.
This illustrates a limitation of the marginal flows, which may poorly fit discrete distributions, especially when only a small number of series in the dataset contain discrete patterns.
In the top-right forecast, the model is extremely confident that the series would obey the same pattern as in the previous days, missing the change in dynamics.
The bottom-left one is similar; the model believes that the forecasted day would be similar to the previous ones, failing to capture the increase that occurs from day to day.
The bottom-right forecast also shows an example where the model assumes that the series has daily periodicity, even though the history reveals this is not the case.
Note that using positional encodings especially designed for temporal data may help circumvent these issues.

\begin{figure}
    \centering
    \includegraphics[width=0.95\linewidth]{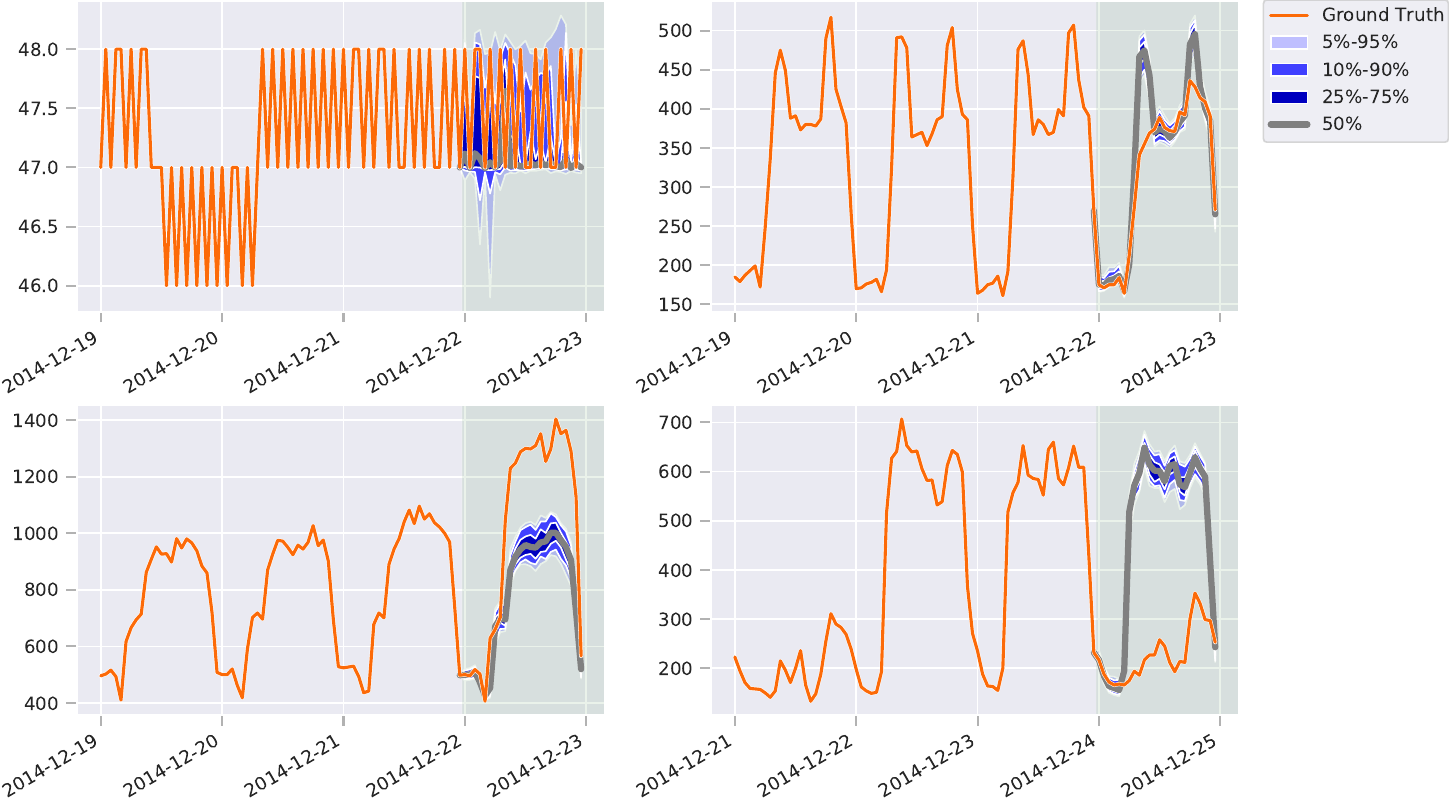}
    \caption{
        Hand-picked example forecasts for the \texttt{electricity} dataset. These examples have been selected to show particularly \textbf{bad} forecasts.
        The historical ground truth shown is the one that was made available to the model.
    }\label{fig:worst-electricity}
\end{figure}

\paragraph{The \texttt{solar-10min} dataset}

\cref{fig:best-solar} shows some particularly good forecasts for the \texttt{solar-10min} dataset.
This dataset has the particularity of having sudden variations (increases or decreases) between consecutive time points.
The forecasts we present in this figure all seem to account for this characteristic by attributing significant likelihood to a wide range of values in their forecasts.
The bottom-left and -right forecasts are particularly interesting since their distributions are oddly shaped: uncertainty is maximal at the beginning of the forecast and decreases afterwards.
Again, this seems to account for the stochasticity that exists within the process while also accounting for the fact that periods of non-zero value always revert to zero periodically.

\begin{figure}
    \centering
    \includegraphics[width=0.95\linewidth]{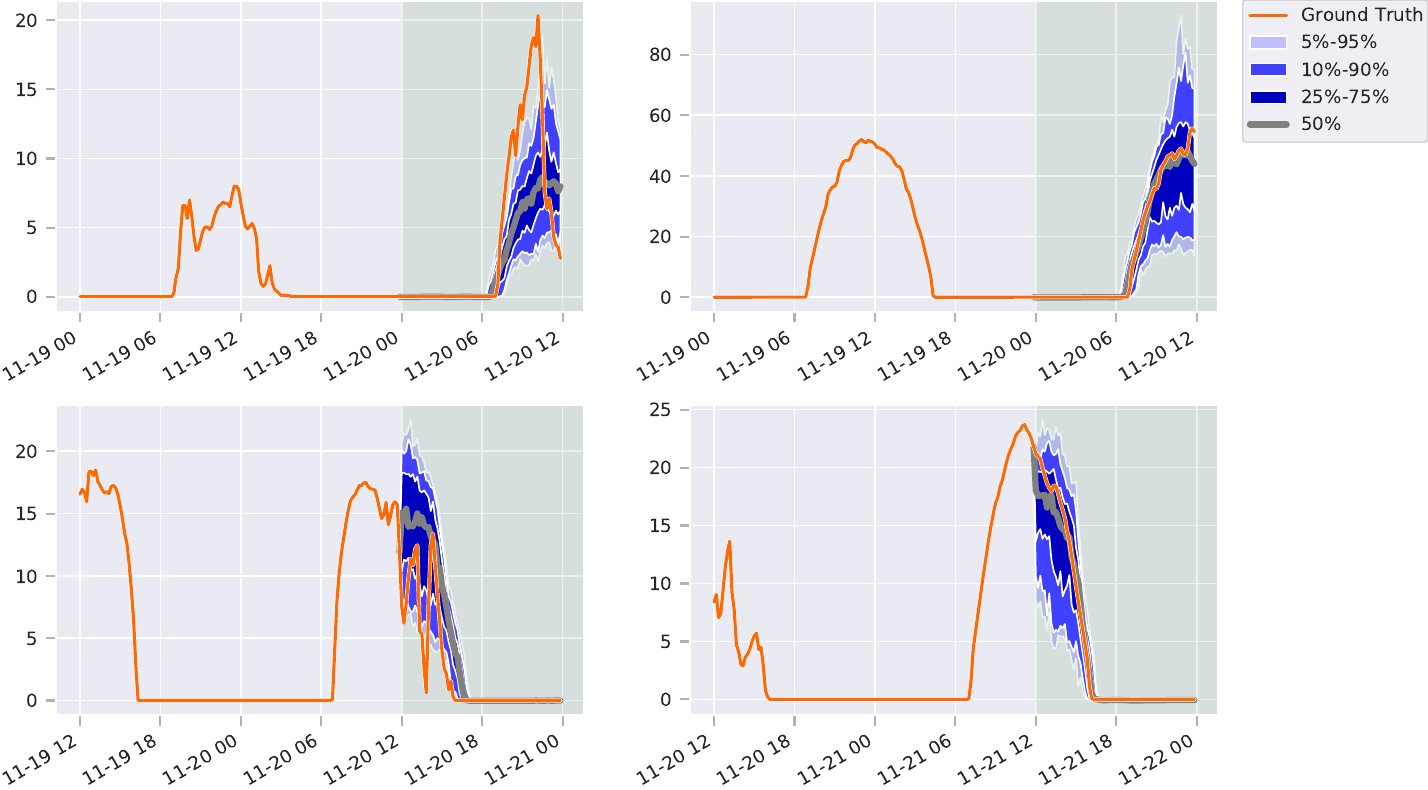}
    \caption{
        Hand-picked example forecasts for the \texttt{solar-10min} dataset. These examples have been selected to show particularly \textbf{good} forecasts.
        The historical ground truth shown is the one that was made available to the model.
    }\label{fig:best-solar}
\end{figure}

\cref{fig:worst-solar} shows some particularly bad forecasts for the \texttt{solar-10min} dataset.
All of the forecasts have the same issue: the model often predicts a peak in values around mid-day and, when this happens, fails to attribute high likelihood to values as low as those observed in the history.

\begin{figure}
    \centering
    \includegraphics[width=0.95\linewidth]{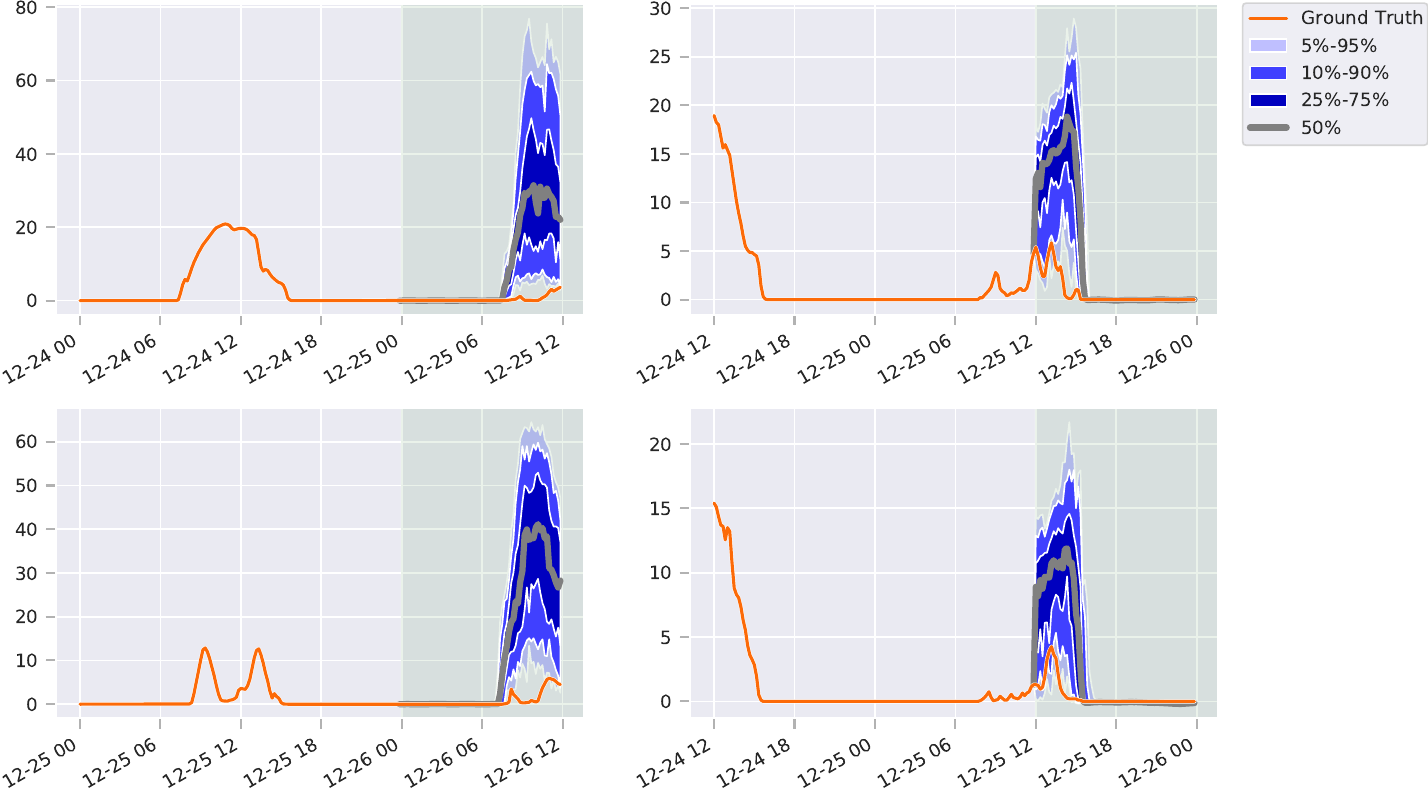}
    \caption{
        Hand-picked example forecasts for the \texttt{solar-10min} dataset. These examples have been selected to show particularly \textbf{bad} forecasts.
        The historical ground truth shown is the one that was made available to the model.
    }\label{fig:worst-solar}
\end{figure}

\paragraph{The \texttt{fred-md} dataset}

\cref{fig:best-fred} shows some particularly good forecasts for the \texttt{fred-md} dataset.
These show that the model can be applied to data with various trends.
Notice how the forecasts' variance gradually increases as we move away from the history, indicating that the model properly accounts for the increasing uncertainty arising from the compounding effects of more and more random events.
Furthermore, the ground-truth values stay within the forecasted range while sometimes deviating from the median, suggesting that the variance is not overestimated.

\begin{figure}
    \centering
    \includegraphics[width=0.95\linewidth]{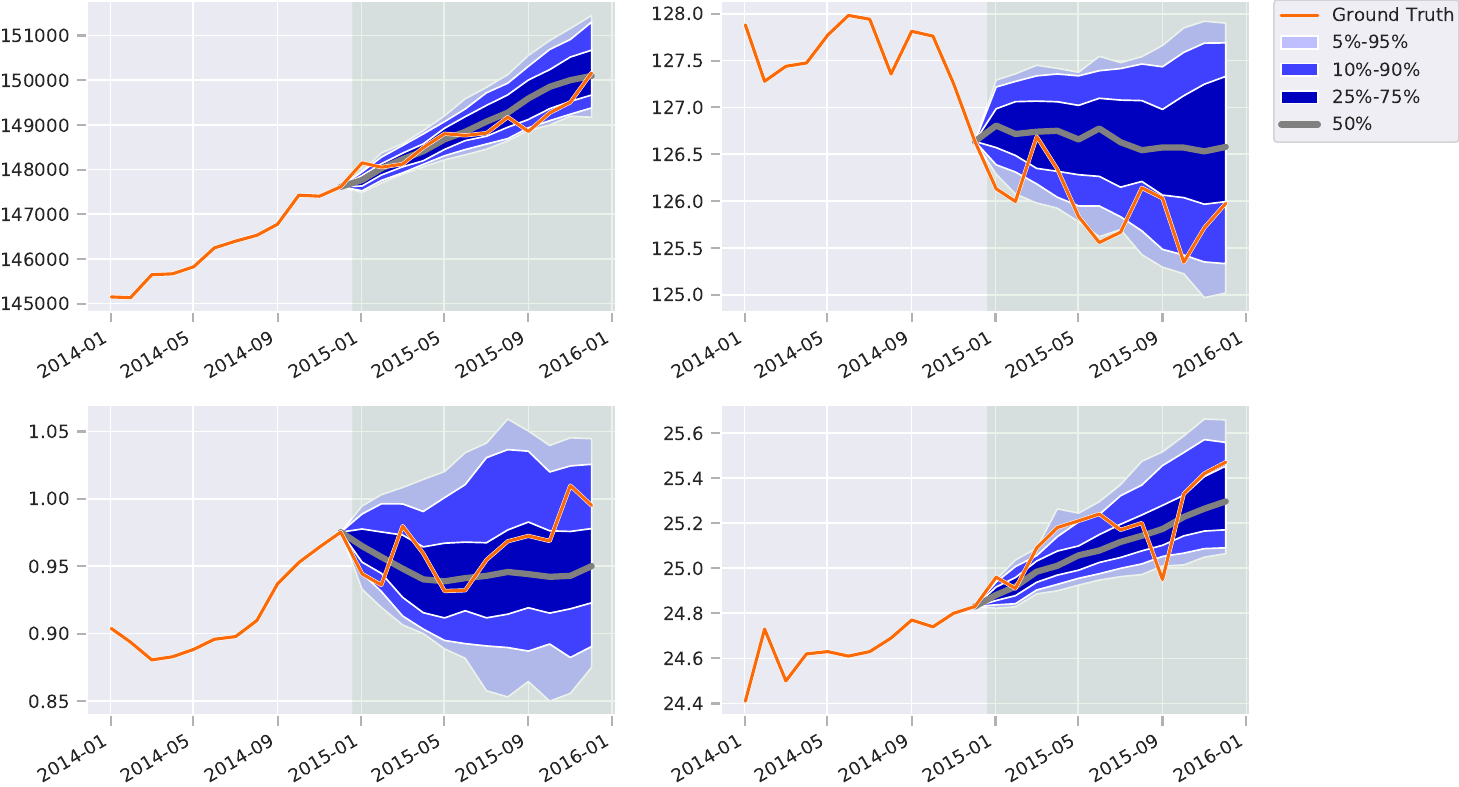}
    \caption{
        Hand-picked example forecasts for the \texttt{fred-md} dataset. These examples have been selected to show particularly \textbf{good} forecasts.
        The historical ground truth shown is the one that was made available to the model.
    }\label{fig:best-fred}
\end{figure}

\cref{fig:worst-fred} shows some particularly bad forecasts for the \texttt{fred-md} dataset.
The top-left and -right forecasts show that the model relied too much on the historical trend and either overreacted to a blip in the data (from 2012-10 to 2012-12 in the top-left) or missed a change in the trend (top-right).
Further, in the bottom-left and -right forecasts, the model completely fails to account for radical changes in the behaviour of the series.

\begin{figure}
    \centering
    \includegraphics[width=0.95\linewidth]{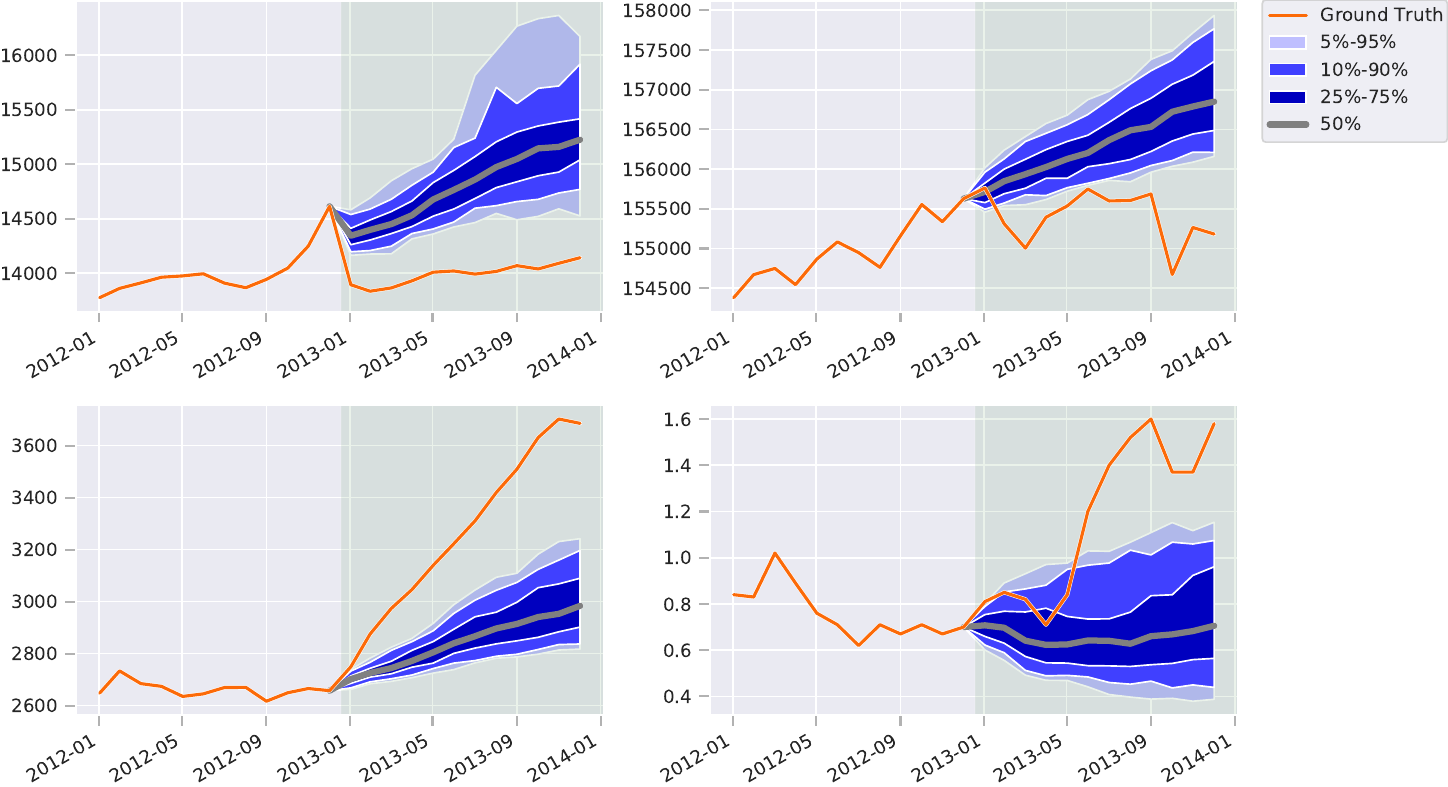}
    \caption{
        Hand-picked example forecasts for the \texttt{fred-md} dataset. These examples have been selected to show particularly \textbf{bad} forecasts.
        The historical ground truth shown is the one that was made available to the model.
    }\label{fig:worst-fred}
\end{figure}

\paragraph{The \texttt{kdd-cup} dataset}

\cref{fig:best-kdd} shows some particularly good forecasts for the \texttt{kdd-cup} dataset.
All four are quite similar in that the model predicts a relatively uniform distribution on a range of possible values, which seems to be accurate based on the oscillations in the ground truth.
Furthermore, we see a gradual increase of variance at time points subsequent to the history, indicating that the model learned that, in this dataset, changes occur gradually rather than suddenly (e.g., in the \texttt{solar-10min} dataset).

\begin{figure}
    \centering
    \includegraphics[width=0.95\linewidth]{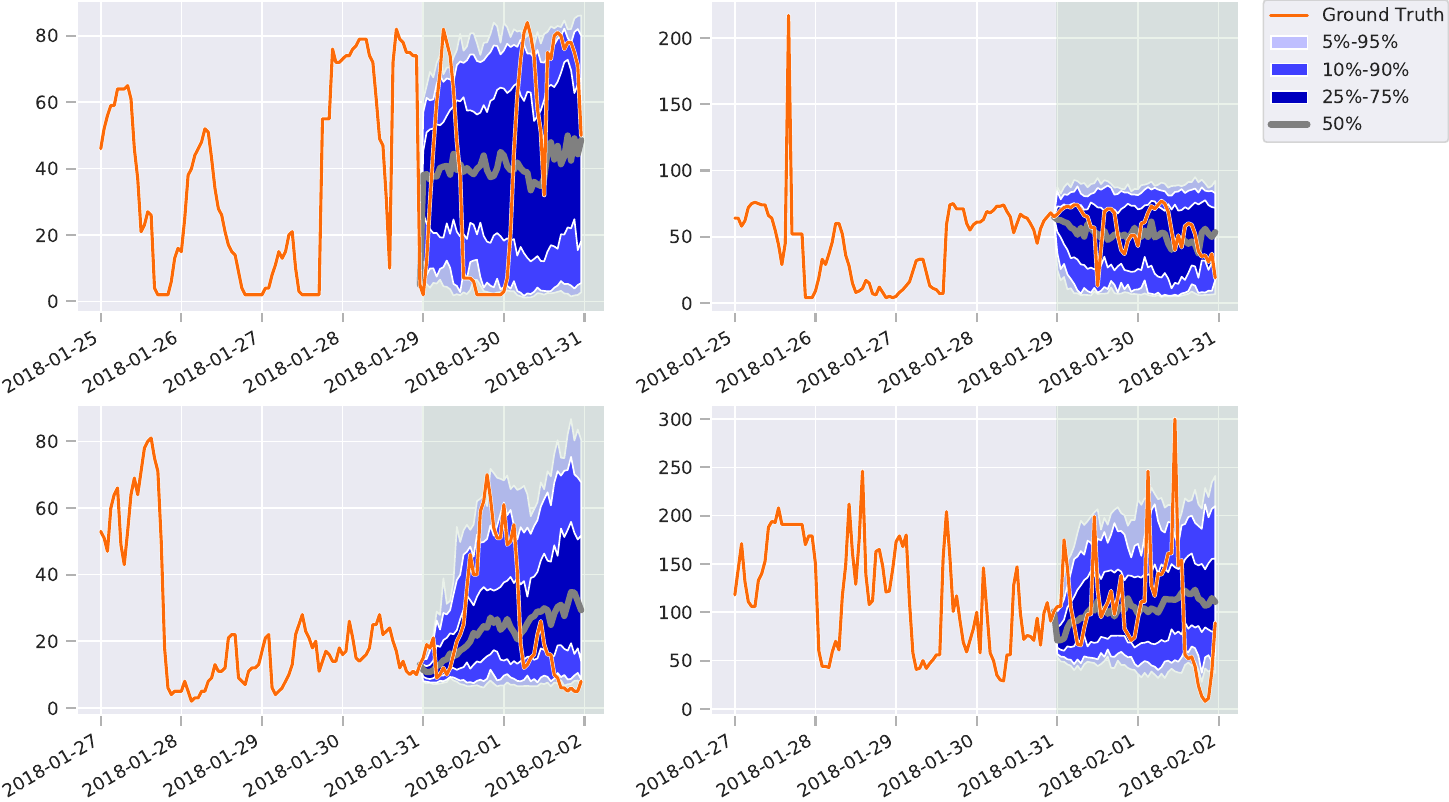}
    \caption{
        Hand-picked example forecasts for the \texttt{kdd-cup} dataset. These examples have been selected to show particularly \textbf{good} forecasts.
        The historical ground truth shown is the one that was made available to the model.
    }\label{fig:best-kdd}
\end{figure}

\cref{fig:worst-kdd} shows some particularly bad forecasts for the \texttt{kdd-cup} dataset.
The top-left forecast shows that the model underestimated the maximum value of the series, not having seen such values in the history.
The bottom-left shows the opposite: the model underestimated the odds of the series remaining at relatively constant and low values.
The top-right forecast assumed a slow return to the recent normal, disregarding the possibility of a further decrease.
The bottom-right forecast failed to predict that the series could increase as much as it did, even though the recent history showed such increases.

\begin{figure}
    \centering
    \includegraphics[width=0.95\linewidth]{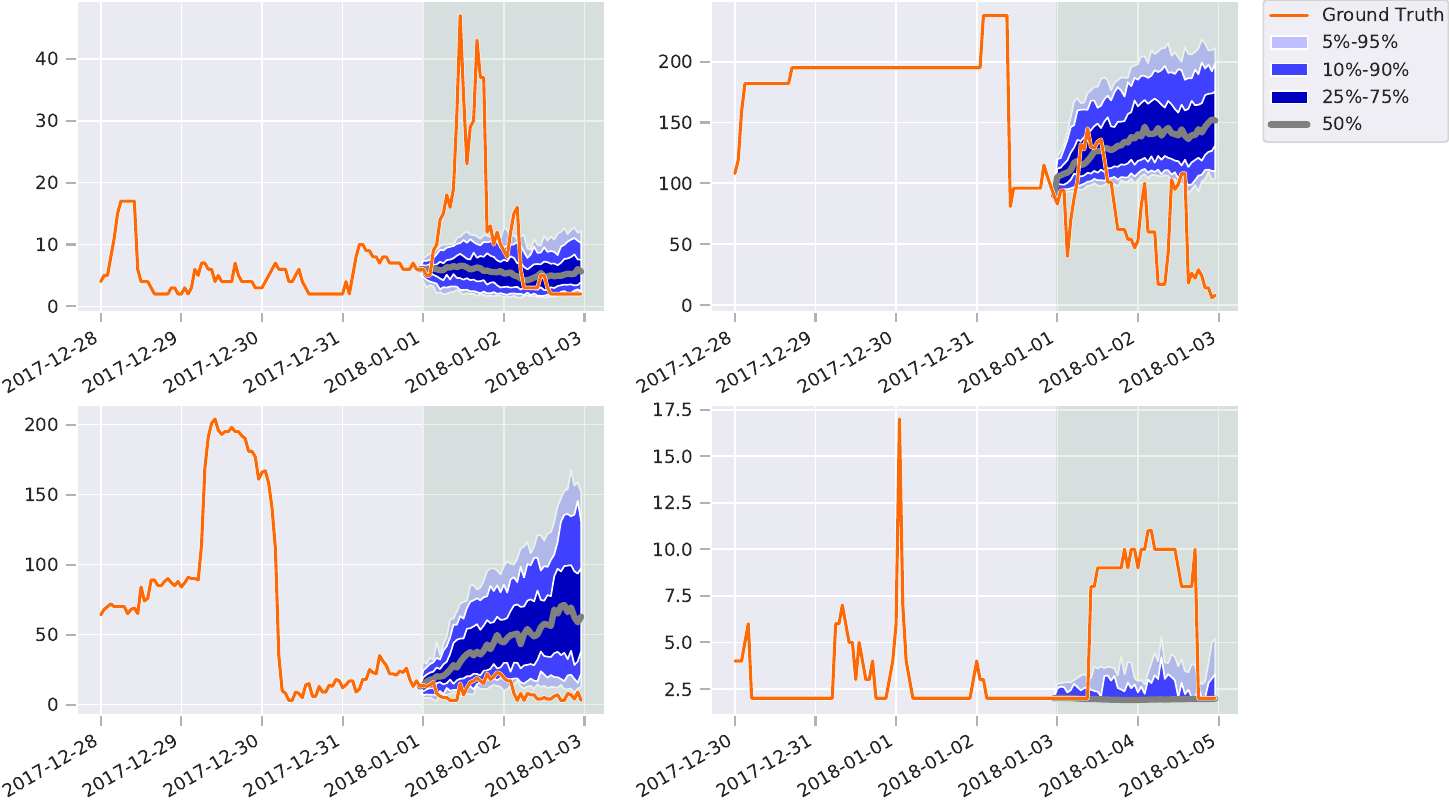}
    \caption{
        Hand-picked example forecasts for the \texttt{kdd-cup} dataset. These examples have been selected to show particularly \textbf{bad} forecasts.
        The historical ground truth shown is the one that was made available to the model.
    }\label{fig:worst-kdd}
\end{figure}

\paragraph{The \texttt{traffic} dataset}

\cref{fig:best-traffic} shows some particularly good forecasts for the \texttt{traffic} dataset.
The forecast on the top left shows that \tactistt{} predicted the morning and evening traffic spikes, even though only the morning spike of the previous day was visible in the history.
The bottom-left forecast shows the model's uncertainty about whether traffic would return to normal after a seemingly anomalous history with most of the data at zero.
The two forecasts on the right are examples where \tactistt{} was able to recognize a pattern in the data and forecast it with very high confidence. 
Note that this pattern abounds in the \texttt{traffic} dataset for many variables and time stamps, which explains why the model could learn to estimate it so accurately.

\begin{figure}
    \centering
    \includegraphics[width=0.95\linewidth]{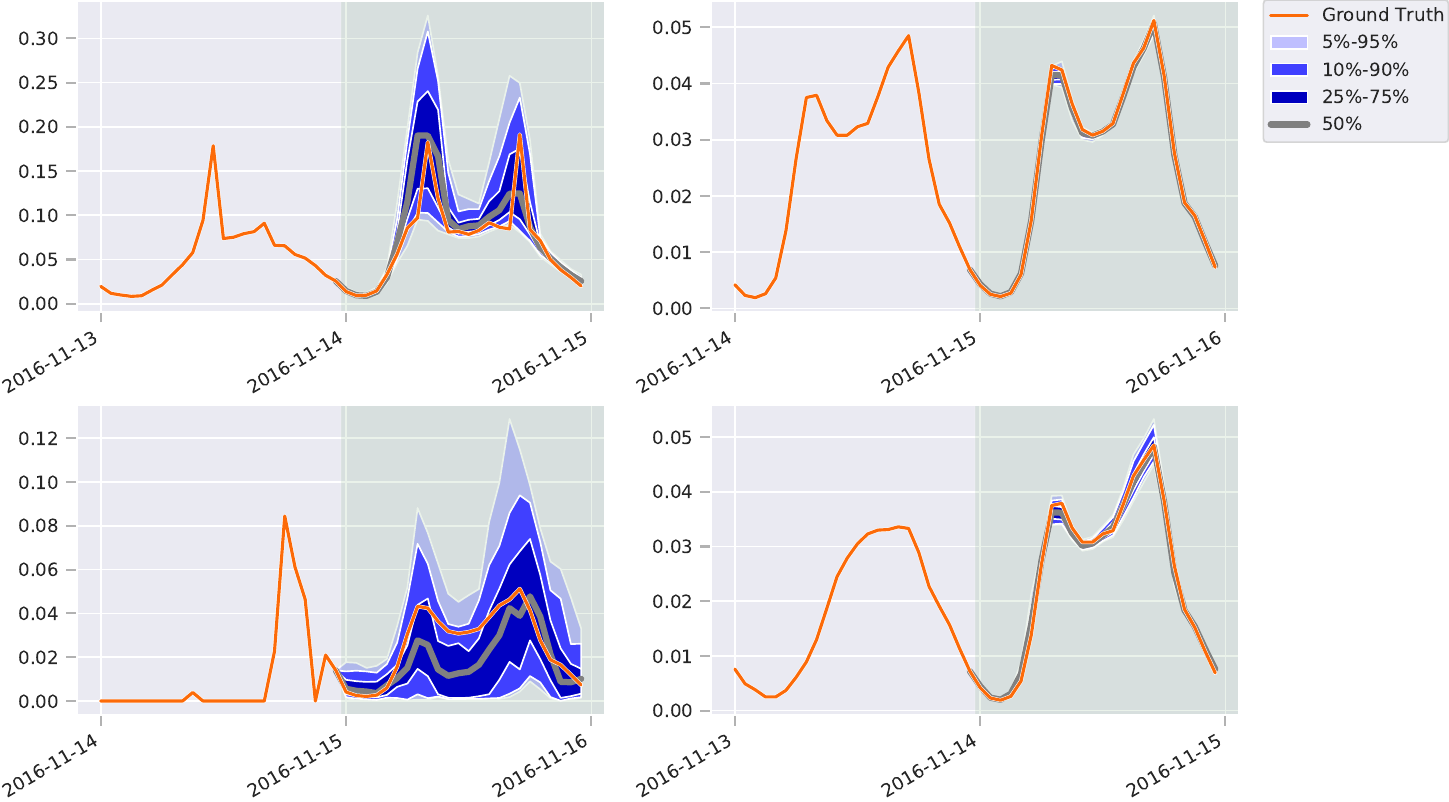}
    \caption{
        Hand-picked example forecasts for the \texttt{traffic} dataset. These examples have been selected to show particularly \textbf{good} forecasts.
        The historical ground truth shown is the one that was made available to the model.
    }\label{fig:best-traffic}
\end{figure}

\cref{fig:worst-traffic} shows some particularly bad forecasts for the \texttt{traffic} dataset.
In the top-left forecast, the model assumed that the forecasted day would match the previous one since similar historical patterns are typically followed by a repetition of the same pattern, thus missing a significant disruption.
For the top-right forecast, the model assumed a quick increase in traffic, ignoring the possibility of a delay before the series deviated from zero.
In the bottom-left forecast, the model dismissed again the possibility of an anomaly, which would cause traffic to go to zero.
Finally, in the bottom-right forecast, the model did not foresee that the series could reach values considerably greater than those in the history.

\begin{figure}
    \centering
    \includegraphics[width=0.95\linewidth]{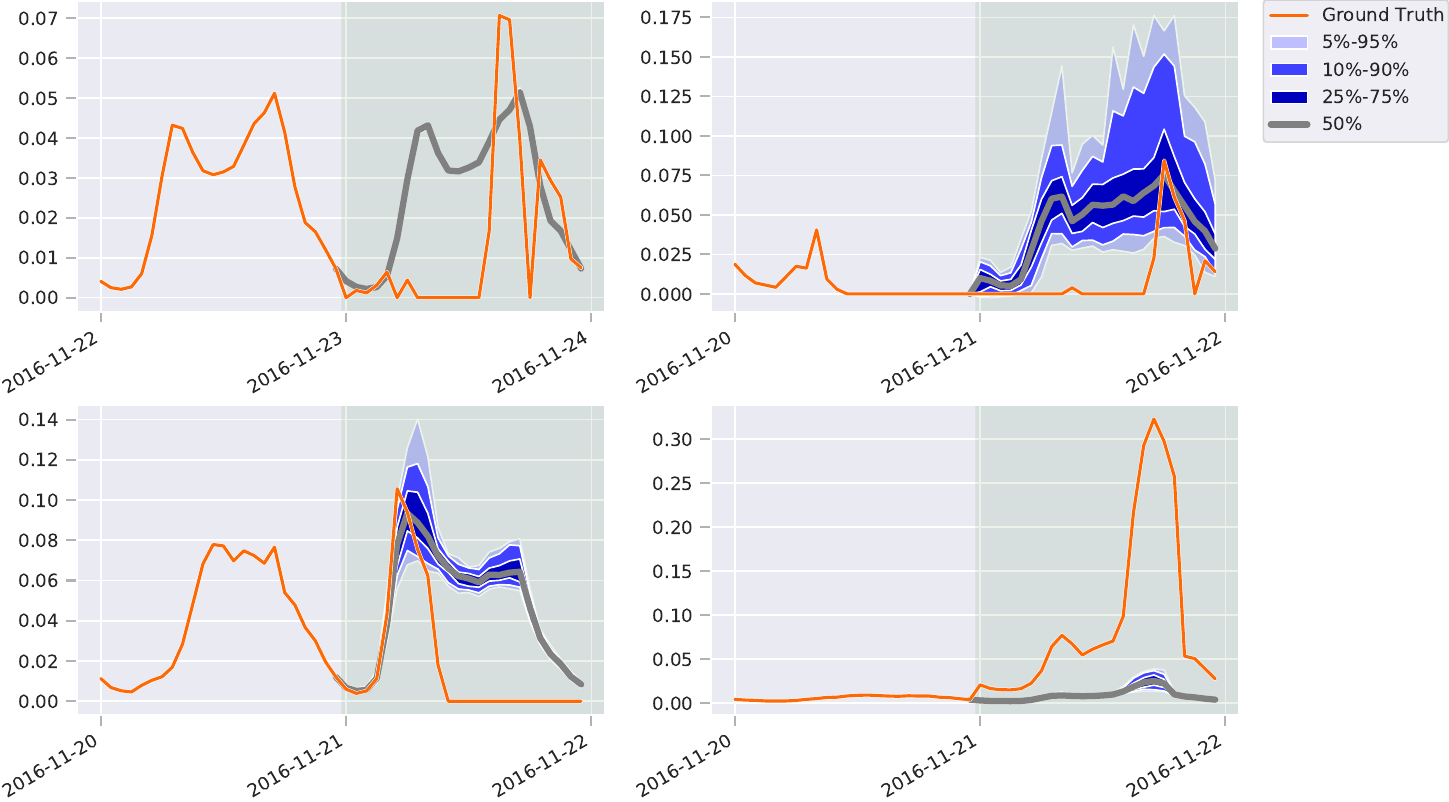}
    \caption{
        Hand-picked example forecasts for the \texttt{traffic} dataset. These examples have been selected to show particularly \textbf{bad} forecasts.
        The historical ground truth shown is the one that was made available to the model.
    }\label{fig:worst-traffic}
\end{figure}

\subsection{Looking into Learned Marginal Distributions}\label{sec:app-deepdive-marginals}

\tactis{} forecasts result from a combination of an attentional copula and the marginal distributions learned for each variable.
Since the marginal distributions are significant contributors to the quality of forecasts (as shown in \cref{app:tactis-ic}), we use this section to look into some of the learned marginal distributions.

\cref{fig:marginal-examples1,fig:marginal-examples2} show a few hand-picked example marginals generated from the learned models.
The marginal cumulative distribution function (CDF) of a variable is directly obtained from $F_{\phi_k}(x^{(m)}_k)$, while its probability density function (PDF) is obtained by differentiating the CDF w.r.t. $x^{(m)}_k$.
Note that these marginals are normalized using the standardization procedure (see \cref{app:data_norm}), so a distribution with a mean of 0 and a variance of 1 would be expected if the model simply took the historical values as the forecast.

\paragraph{Unimodal marginals} As can be seen in \cref{fig:marginal-examples1}, the marginals for \texttt{electricity}, \texttt{fred-md}, and \texttt{traffic} are all unimodal.
The main difference is that the \texttt{fred-md} marginal is quite broad, showing that the variable has an extensive range of possible values that are outside the range of historical values. 
In contrast, the \texttt{traffic} marginal is very narrow, demonstrating that the model is quite confident in the value of the variable.

\paragraph{Multimodal marginals} As can be seen in \cref{fig:marginal-examples2}, the marginals for \texttt{kdd-cup} and \texttt{solar-10min} are both multimodal, something that would not have been possible using many common parametric distributions, such as the normal or gamma distributions.
Furthermore, both marginals have a significant spike at their minimal value, showing that the model has learned that some variables have a non-zero probability of being exactly at their lower bound in these datasets.
For example, in the \texttt{solar-10min} dataset, this may be caused by zero solar energy production on a cloudy day.

\begin{figure}
    \centering
    \includegraphics[width=0.95\linewidth]{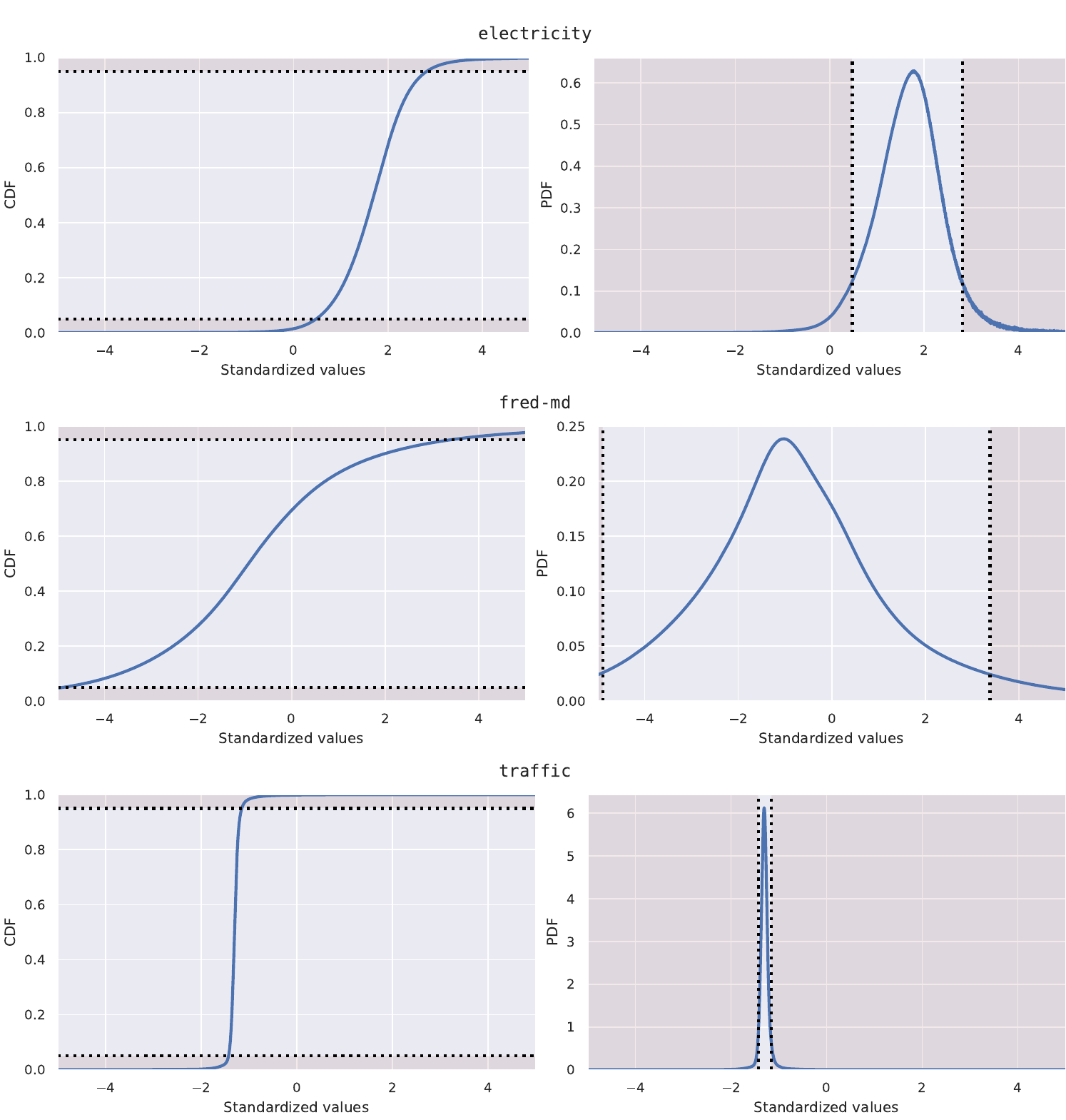}
    \caption{
        Cumulative distribution functions (CDF) and probability density functions (PDF) of hand-picked example marginals generated by \tactistt{} for \texttt{electricity}, \texttt{fred-md}, and \texttt{traffic}.
        These examples have been selected to show various behaviours of unimodal marginals.
        The red areas represent the part of the distribution which has been removed from the sampling, as mentioned in \cref{app:inverting_flow}.
    }\label{fig:marginal-examples1}
\end{figure}

\begin{figure}
    \centering
    \includegraphics[width=0.95\linewidth]{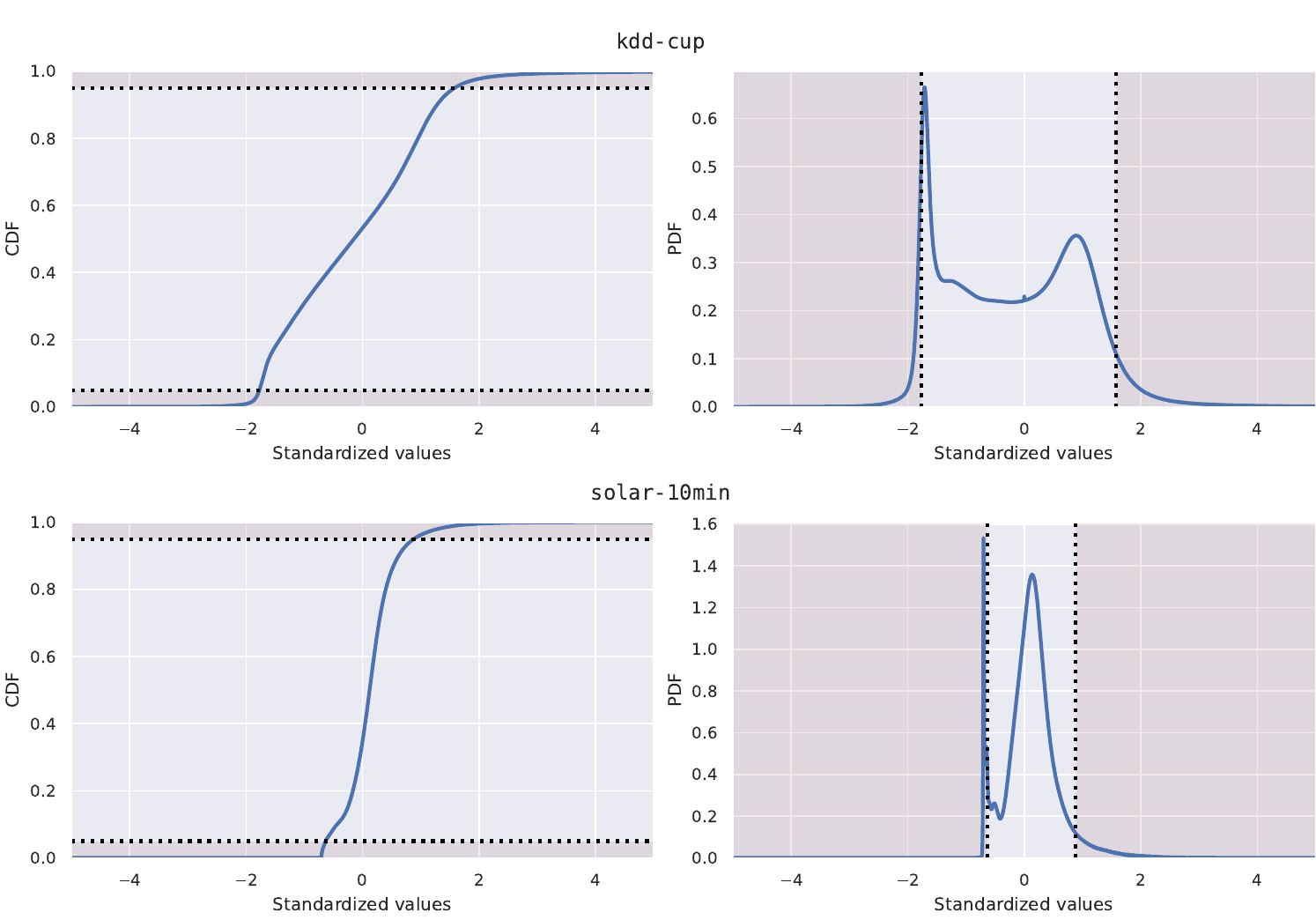}
    \caption{
        Cumulative distribution functions (CDF) and probability density functions (PDF) of hand-picked example marginals generated by \tactistt{} for \texttt{kdd-cup} and \texttt{solar-10min}.
        These examples have been selected to show various behaviours of multimodal marginals.
        The red areas represent the part of the distribution which has been removed from the sampling, as mentioned in \cref{app:inverting_flow}.
    }\label{fig:marginal-examples2}
\end{figure}

\subsection{Learning Dependencies Between Variables} \label{app:correlations}

In \cref{app:tactis-ic}, we have shown that the \tactisic{} ablation, which removes all dependencies between the variables, still leads to decent results for the metrics.
This could indicate that \tactistt{} makes predictions using only its marginal distributions without having learned anything noteworthy in its attentional copula.
To verify if this is the case, we can inspect the correlations between variables in the predictive distribution of the model (see \cref{sec:bg-copula}).
If the copula were unused, we would observe zero Pearson product-moment correlation coefficients between all variables.
Otherwise, we would observe non-zero values.
Rest assured, the following experiments show that \tactistt{} does learn correlations between variables, making use of its attentional copula.

\paragraph{Intra-series correlations}

In this first experiment, we inspect the correlations learned between various time steps within the same series (intra-series).
To achieve this, we compute the correlation coefficient between the last forecasted time step \footnote{Selected due to typically being the most variable.} and all other earlier time steps of the forecast. 
These are then averaged over the multiple series and the multiple forecasts for each trained model (one per backtesting period and trial).
\cref{fig:time-correlations} shows the distribution of these averaged correlation coefficients for each studied dataset.
From these, it is clear that \tactistt{} used its attentional copula to learn intra-series dependencies.
As expected, the correlation between time steps decreases with time.
An exception for time differences nearing 24 hours can be observed, e.g., in \texttt{traffic}.
This could be explained by stochastic events impacting multiple days but only at certain times during the day.

\begin{figure}
    \centering
    \includegraphics[width=0.95\linewidth]{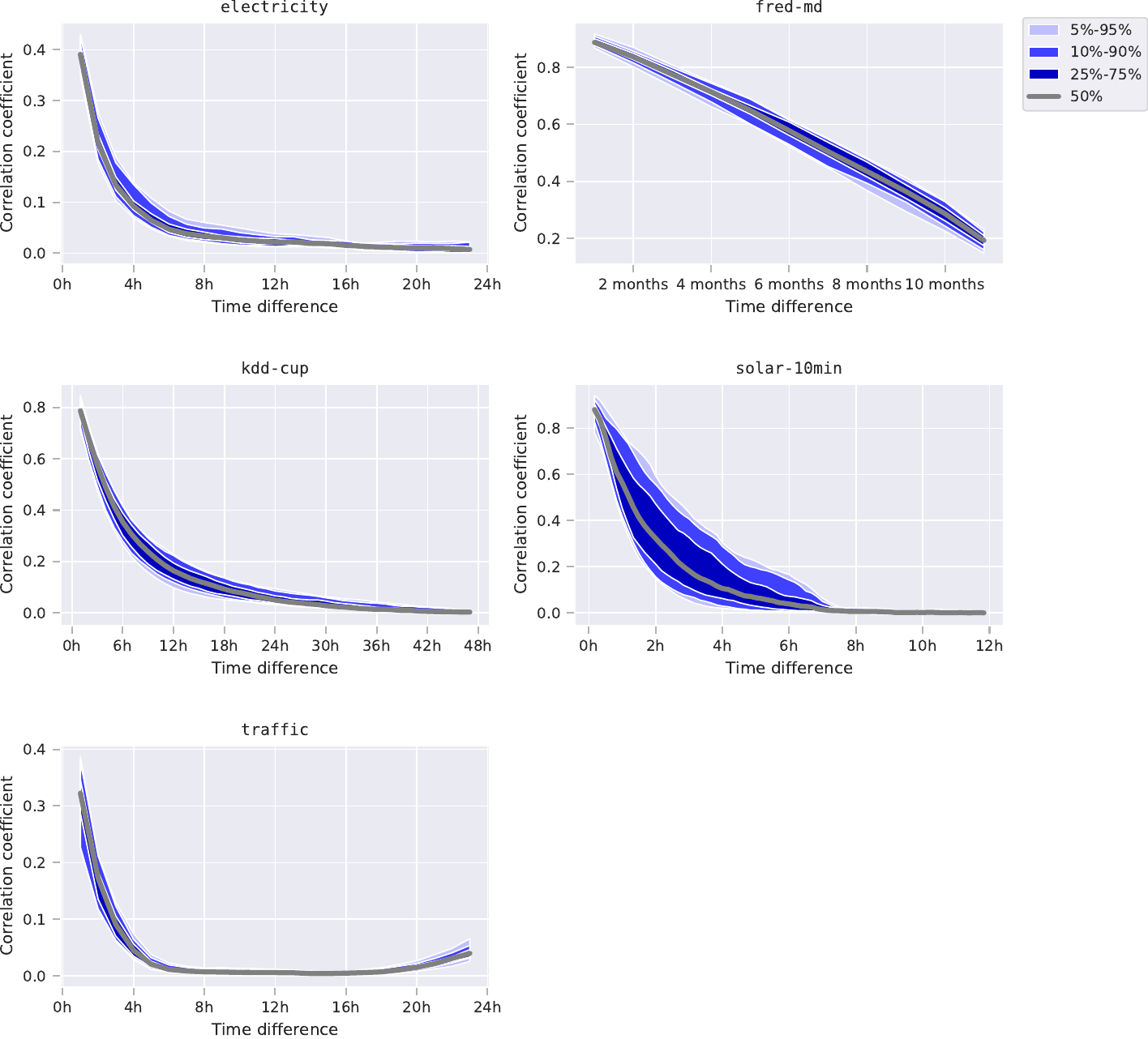}
    \caption{
        Distribution of the correlation coefficients between the last forecasted time step and prior time steps as a function of the time difference between both.
        The correlations are averaged over the multiple series and forecasts, and the distribution is taken over the multiple trained models for each dataset.
    }\label{fig:time-correlations}
\end{figure}

\paragraph{Inter-series correlations}

We now look into the ability of \tactistt{} to model dependencies between the series (inter-series) using its attentional copula.
To achieve this, we use samples taken at the last time step in the forecast to measure inter-series correlations.
\cref{fig:space-correlations-electricity,fig:space-correlations-fred,fig:space-correlations-kdd,fig:space-correlations-solar,fig:space-correlations-traffic} show the average correlation coefficients between all series, averaged over all forecasts, backtesting periods, and trials.
While the inter-series correlations tend to be of lesser magnitude than intra-series correlations, \tactistt{} did indeed learn some of the structure of the datasets.
In particular, it isolated two clusters of interdependent variables in \texttt{kdd-cup}, which makes sense given that the series were collected in two cities (Beijing and London).
Similarly, it was able to recognize that \texttt{fred-md} data contains multiple groups of series with strong dependencies inside each group.
The \texttt{traffic} result is somewhat lacklustre since we expected more spatial correlations for this dataset, which captures road occupancy levels at multiple intersections, some of which must be nearby.
This indicates that while \tactistt{} results are impressive, there is still improvement to be made to learn the dependencies in some of these datasets.

\begin{figure}
    \centering
    \includegraphics[width=0.95\linewidth]{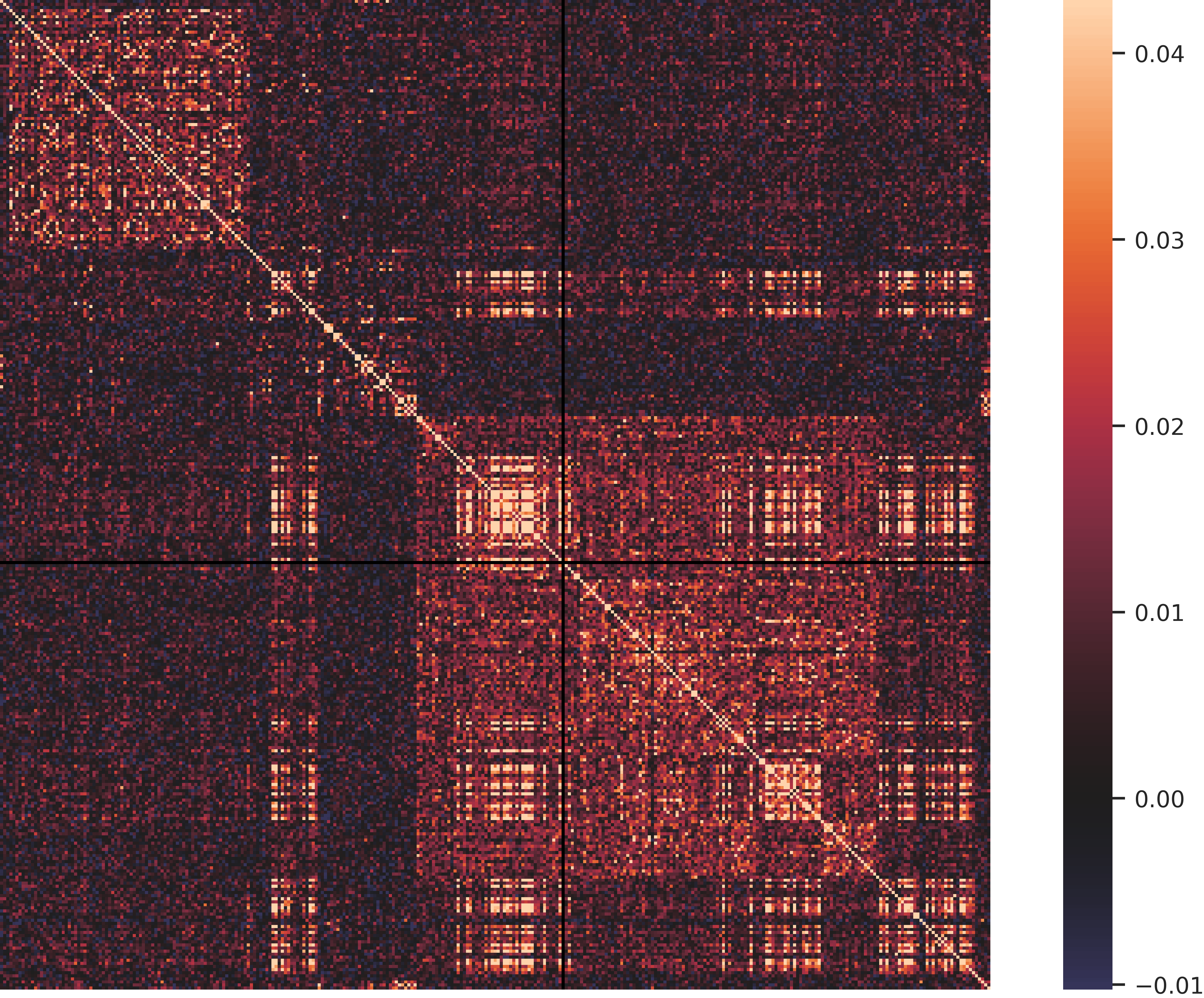}
    \caption{
        Average correlation coefficients between all series at the last forecasted time step for the \texttt{electricity} dataset.
        The bounds of the colour map are set to the 0.02 and 0.98 quantiles of the data.
        The black lines are caused by a variable with a constant forecast, thus preventing it from having correlation coefficients.
    }\label{fig:space-correlations-electricity}
\end{figure}

\begin{figure}
    \centering
    \includegraphics[width=0.95\linewidth]{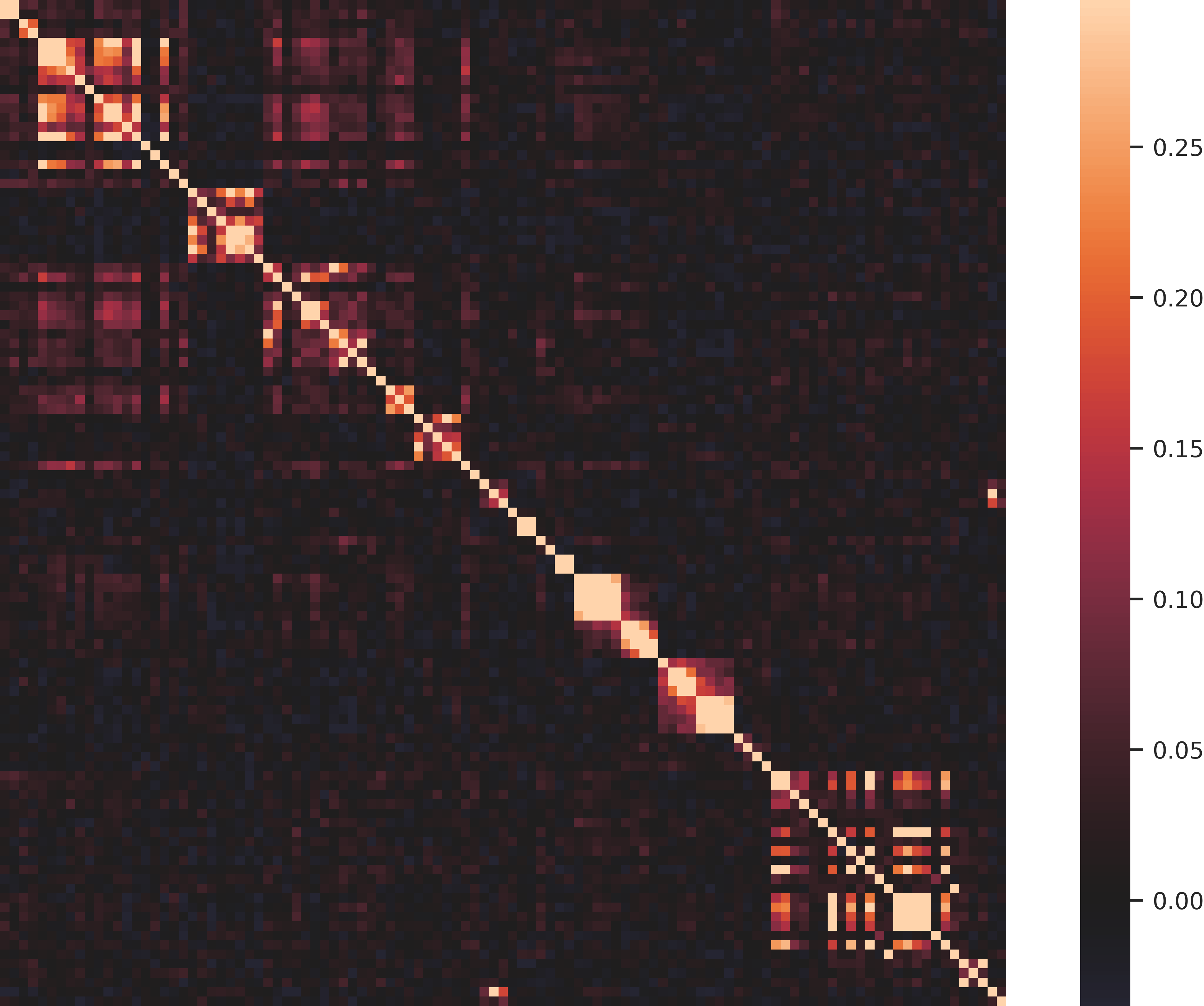}
    \caption{
        Average correlation coefficients between all series at the last forecasted time step for the \texttt{fred-md} dataset.
        The bounds of the color map are set to the 0.02 and 0.98 quantiles of the data.
    }\label{fig:space-correlations-fred}
\end{figure}

\begin{figure}
    \centering
    \includegraphics[width=0.95\linewidth]{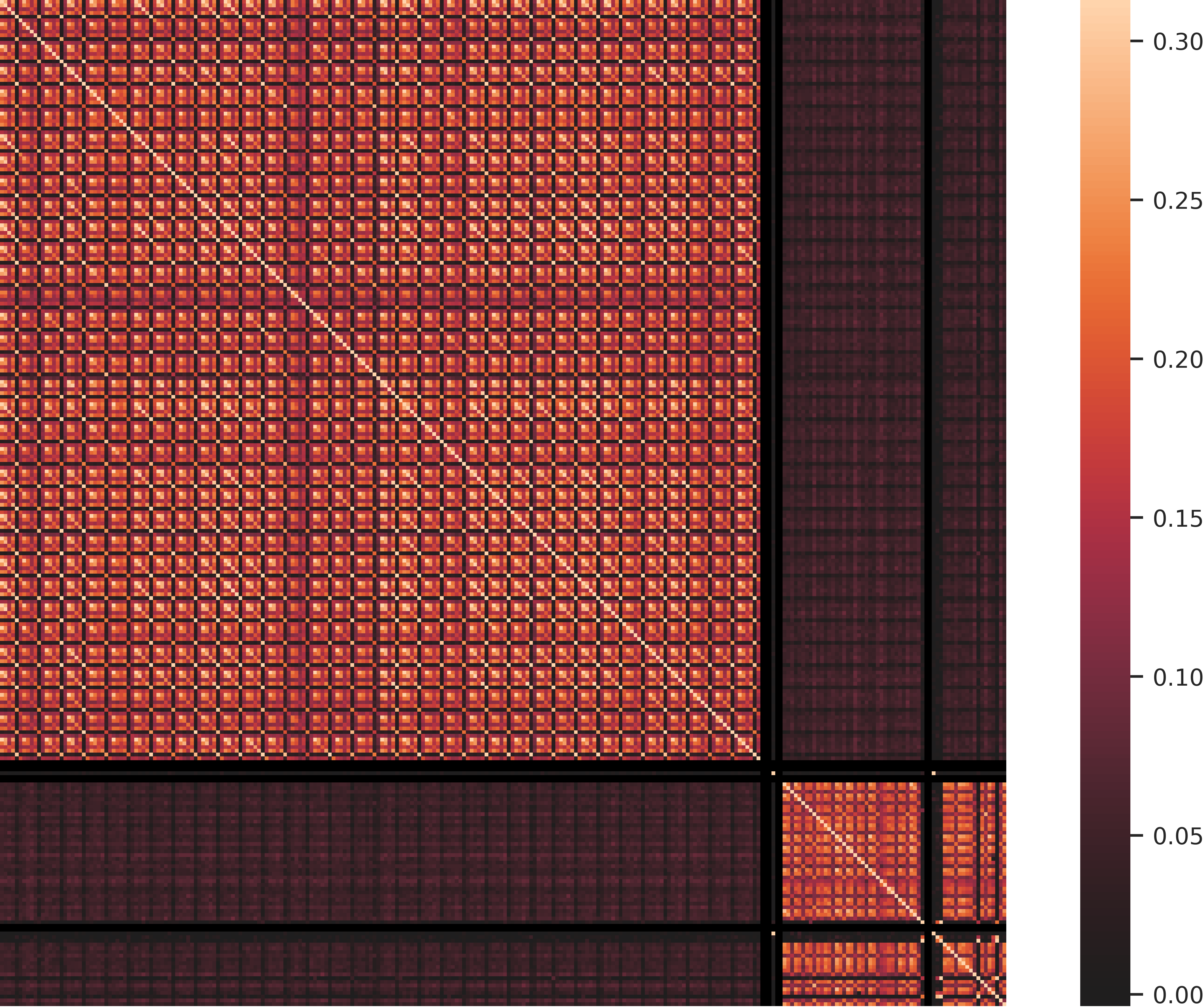}
    \caption{
        Average correlation coefficients between all series at the last forecasted time step for the \texttt{kdd-cup} dataset.
        The bounds of the colour map are set to the 0.02 and 0.98 quantiles of the data.
        The black lines are caused by variables with constant forecasts, thus preventing them from having correlation coefficients.
    }\label{fig:space-correlations-kdd}
\end{figure}

\begin{figure}
    \centering
    \includegraphics[width=0.95\linewidth]{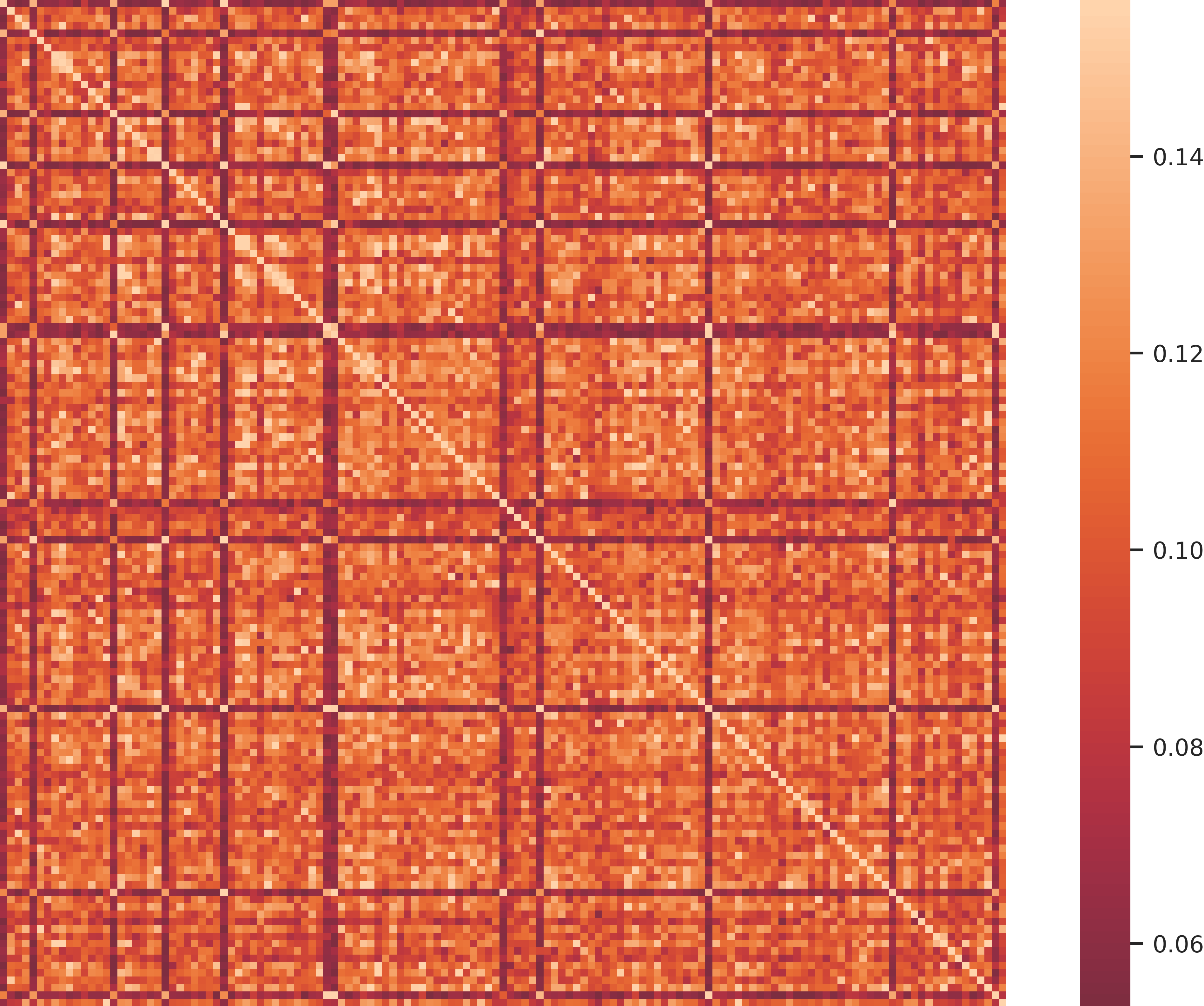}
    \caption{
        Average correlation coefficients between all series at the last forecasted time step for the \texttt{solar-10min} dataset.
        The bounds of the color map are set to the 0.02 and 0.98 quantiles of the data.
    }\label{fig:space-correlations-solar}
\end{figure}

\begin{figure}
    \centering
    \includegraphics[width=0.95\linewidth]{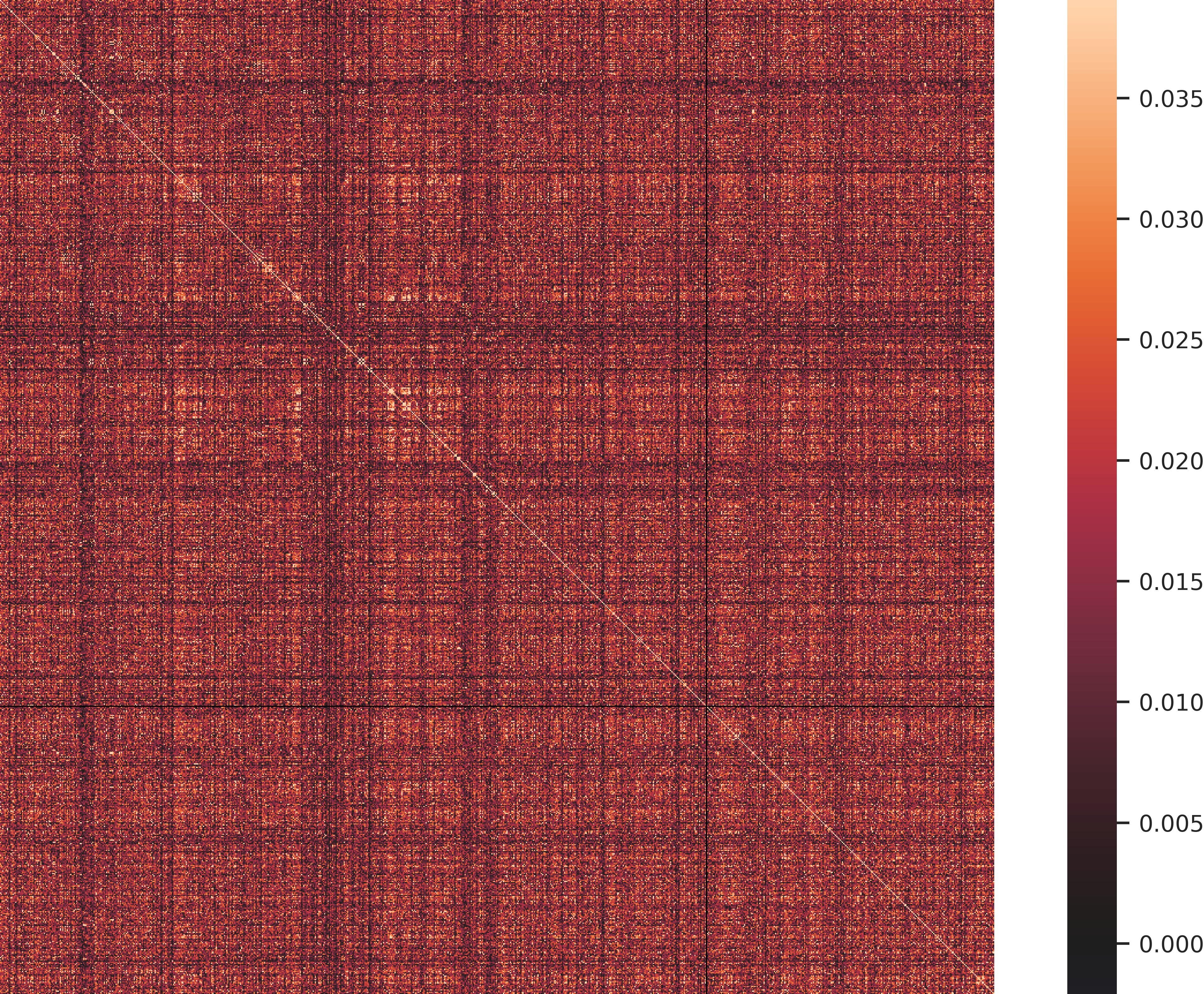}
    \caption{
        Average correlation coefficients between all series at the last forecasted time step for the \texttt{traffic} dataset.
        The bounds of the colour map are set to the 0.02 and 0.98 quantiles of the data.
        The black lines are caused by variables with constant forecasts, thus preventing them from having correlation coefficients.
    }\label{fig:space-correlations-traffic}
\end{figure}

\subsubsection{Impact of Bagging on Inter-Series Correlations} \label{app:baggingexp}

In \cref{table:bagging-maintext}, we have shown that reducing the number of series available at once during training has no negative impact on the quality of the various metrics.
However, does it affect the models' ability to learn inter-series dependencies?
In this section, we repeat the inter-series correlations experiment for the forecasts obtained by models trained using various bagging sizes.
The results for bagging sizes 1, 2, 5, and 10 are shown in \cref{fig:space-correlations-bags}.
As expected, training with a bagging size of 1 prevents the model from correctly learning inter-series dependencies.
This occurs because the model never sees data from multiple series during training; thus, the attention mechanism cannot learn to block the flow of information between two independent series.
The result is a matrix where all series are quite correlated; in fact, the correlation values are similar to those in \cref{fig:time-correlations}, suggesting that the model may consider all variables to be from the same series.
For a bagging size of 2, the correlations between some series start to fade, showing that the attention mechanism is becoming more selective about dependency patterns.
Finally, after the bagging size is increased to 5 or 10, the inter-series correlations seem to converge, suggesting that further increasing the bag size would not change the learned correlations.

\begin{figure}
    \centering
    \includegraphics[width=0.95\linewidth]{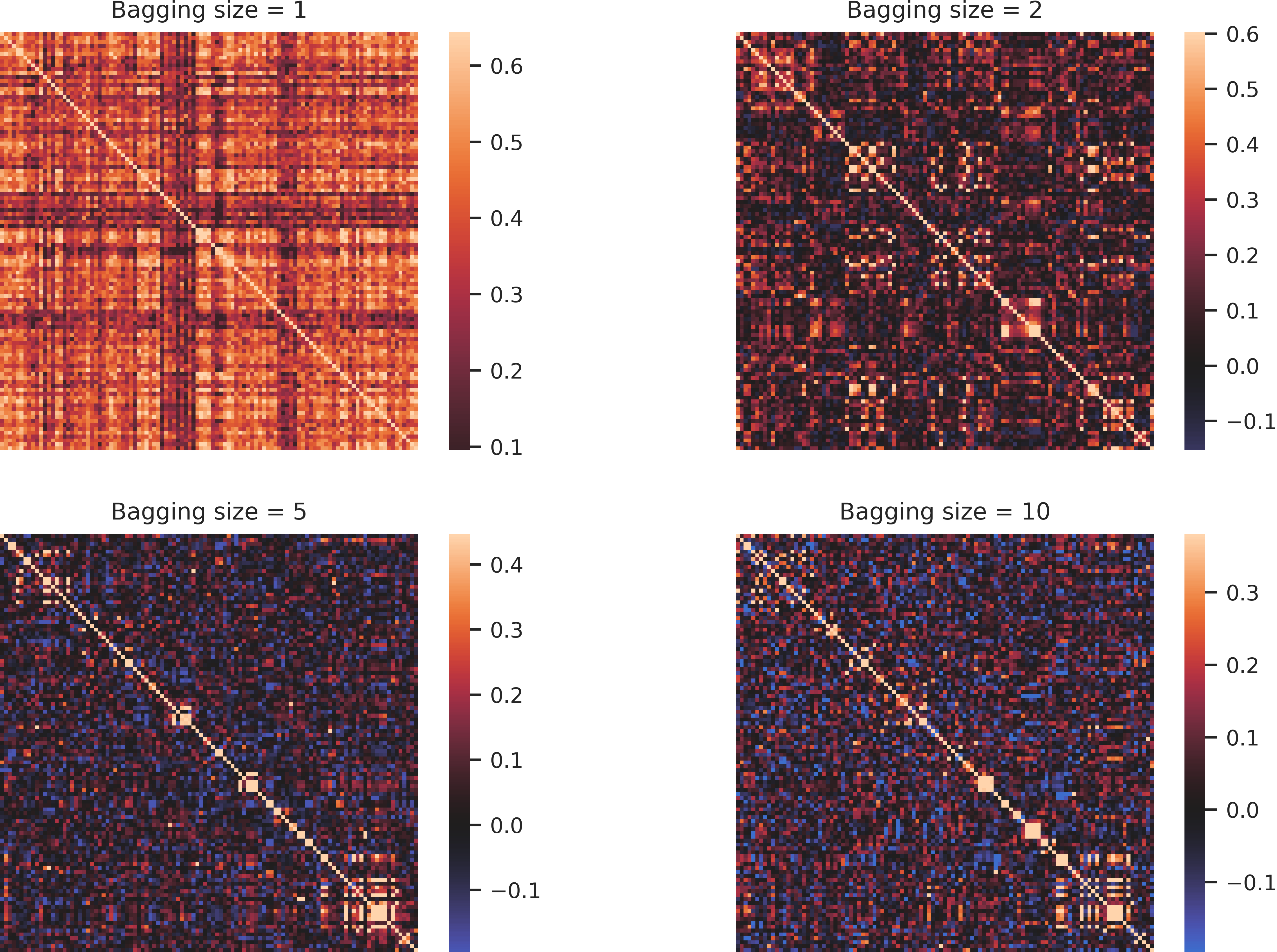}
    \caption{
        Average correlation coefficients between all series at the last forecasted time step for the \texttt{fred-md} dataset for various bagging sizes.
        The bounds of the color maps are set to the 0.02 and 0.98 quantiles of the data.
    }\label{fig:space-correlations-bags}
\end{figure}

One concern raised by this experiment is that the energy score could not reveal the highly incorrect correlations learned for a bagging size of 1.
The energy score for a bagging size of 1 is equal to $8.43 \pm 1.76$, which is very close to the one obtained for a bagging size of 20 ($8.30 \pm 1.68$) and even better than the one obtained for a bagging size of 5 ($9.36 \pm 1.88$).
One possible explanation for this lack of sensitivity is that the energy score is less sensitive at high dimensions, as explored in \citet{pinson2013discrimination}.
Further investigation is required to determine whether this is indeed the case in this experiment or if this is due to another cause.

\end{document}